\documentclass[10pt,letterpaper]{article}

\PassOptionsToPackage{nonamebreak}{natbib}

\usepackage[letterpaper,textwidth=5.5in,textheight=9in,centering]{geometry}
\usepackage{times}
\usepackage{amsmath}
\usepackage{bm}
\usepackage{natbib}
\setcitestyle{authoryear,round,citesep={;},aysep={,},yysep={;}}

\usepackage{hyperref}
\usepackage{url}
\usepackage{booktabs}
\usepackage{graphicx}
\usepackage{multirow}
\usepackage[table]{xcolor}
\usepackage{etoolbox}
\usepackage{titlesec}
\usepackage{abstract}

\usepackage{amsmath,amsfonts,bm}

\def\eqref#1{equation~\ref{#1}}
\def\1{\bm{1}}

\def\vb{{\bm{b}}}

\def\vx{{\bm{x}}}

\def\mA{{\bm{A}}}
\def\mB{{\bm{B}}}

\def\mW{{\bm{W}}}

\DeclareMathAlphabet{\mathsfit}{\encodingdefault}{\sfdefault}{m}{sl}
\SetMathAlphabet{\mathsfit}{bold}{\encodingdefault}{\sfdefault}{bx}{n}

\titleformat{\section}
  {\normalfont\bfseries\fontsize{12}{13}\selectfont}
  {\thesection}{0.65em}{}

\titleformat{\subsection}
  {\normalfont\bfseries\fontsize{10}{11}\selectfont}
  {\thesubsection}{0.65em}{}

\titlespacing*{\section}
  {0pt}{12pt plus 2pt minus 2pt}{5pt}

\titlespacing*{\subsection}
  {0pt}{9pt plus 1pt minus 1pt}{3pt}

\AtBeginEnvironment{table}{\setlength{\abovecaptionskip}{0pt}}

\usepackage{amssymb}
\usepackage{amsthm}
\theoremstyle{plain}
\newtheorem{theorem}{Theorem}[section]
\newtheorem{lemma}[theorem]{Lemma}
\newtheorem{corollary}[theorem]{Corollary}

\theoremstyle{definition}
\newtheorem{definition}[theorem]{Definition}

\begin{document}

\fontsize{10}{11}\selectfont

% ------------------------------------------------------------
% Title block
% ------------------------------------------------------------
\begin{center}

% \rule{\textwidth}{0.45pt}
\rule{\textwidth}{2pt}

\vspace{0.85em}

{\fontsize{15}{17}\selectfont\bfseries
Hidden Trajectories Don't Lie: Correcting Corruption\\[-0.05em]
Drift for Test-Time Adaptation\par}

\vspace{0.55em}

\rule{\textwidth}{0.45pt}

\vspace{0.85em}

{\fontsize{9.8}{11}\selectfont\bfseries
Amirhossein Sabour\textsuperscript{1,2,*}
\qquad
Amir Khasahmadi\textsuperscript{1}
\qquad
Mengliu Zhao\textsuperscript{1,\(\dagger\)}\par}

\vspace{0.35em}

{\fontsize{8.8}{10}\selectfont
\textsuperscript{1}Autodesk Research
\qquad
\textsuperscript{2}McMaster University\par}

\vspace{0.12em}

{\fontsize{8}{9}\selectfont\ttfamily
\{amirhossein.sabour,amir.khasahmadi,mengliu.zhao\}@autodesk.com\par}

\end{center}

\vspace{0.35em}

\begingroup
\renewcommand{\thefootnote}{\fnsymbol{footnote}}
\footnotetext[1]{This study was conducted during the author's internship at Autodesk Research.}
\footnotetext[2]{Corresponding author: \texttt{mengliu.zhao@autodesk.com}}
\endgroup

\begin{abstract}
Test-time adaptation (TTA) adapts pretrained models to distribution shifts encountered during deployment. However, most existing methods rely primarily on the weight representations of the final layer or signals derived from the model’s output probability distribution. In this work, we study \emph{hidden trajectories}: the ordered sequence of intermediate representations produced across a deep neural network. We observe that clean and corrupted trajectories become more separable as corruption severity increases. Based on this observation, we introduce \emph{Drift-TTA}, which adds zero-initialized low-rank adapters to a frozen backbone and trains them to align corrupted hidden trajectories with those from cleaner inputs. The proposed approach supports two variants: a source-dependent version trained with clean targets, and a source-free setting that learns from corrupted images at different severities. Building on these two variants, we further develop an online TTA approach that freezes the backbone and adapters while updating a scalar gate to control adaptation strength. We further cast trajectory alignment as MAP inference under a Gaussian observation model and derive a KL-based transfer bound characterizing
test-time trajectory misalignment in terms of offline residual error, residual dimensionality, and distribution shift between offline and test corruptions. We evaluate Drift-TTA on CIFAR-10-C, CIFAR-100-C, and ImageNet-C using supervised ViTs and CLIP models. 
On CIFAR-100-C at severity 5, our source-free online approach improves accuracy over no TTA by 13.3 percentage points with ViT-B/16 and 11.6 percentage points with CLIP ViT-B/16, reaching 74.9\% and 47.4\%, respectively.
Under the same setting, prior TTA baselines yield ECEs of 18.00\% for ViT and 14.70\% for CLIP, while our approach reduces them to 9.53\% and 5.77\%, respectively.
These results demonstrate that hidden trajectory alignment provides an effective signal for both offline correction and lightweight online test-time adaptation.
\end{abstract}

\section{Introduction}

Test time adaptation (TTA) adapts a pretrained model to unlabeled test data during deployment, when corruptions such as sensor noise, weather, blur, and compression artifacts shift inputs away from the clean source distribution~\citep{hendrycks2019benchmarking}.
Many recent methods guide adaptation using prediction-derived objectives, learned update mechanisms, or auxiliary self-supervised signals~\citep{liang2025comprehensive,xiao2024beyond}.
Some minimize entropy or use related confidence and energy objectives to update a small set of model parameters, including the affine parameters of batch normalization layers~\citep{wang2021tent,choi2025aea}, while others adapt representations or features directly~\citep{murphy2026neo,cho2025fata}.

These studies demonstrate the potential of TTA to recover much of the accuracy lost on standard corruption benchmarks, often with modest computation and without source labels during test time.
Entropy minimization is the primary approach in TTA, with several methods extending it through related confidence- and energy-based objectives~\citep{deng2025mgg,lei2025ttvd}. However, these objectives encourage the model to produce more confident predictions, but do not directly optimize whether those predictions are correct. As a result, an incorrect prediction can become more confident without necessarily being corrected.

Furthermore, intermediate features remain comparatively underexplored in TTA, both in how they change as corruption severity increases and in whether they can provide a complementary signal for adaptation.
To examine this, we look beyond the final output or a single representation and follow how features evolve through the network.
Chain-of-Embedding~\citep{coe2024} provides a framework for this by representing a forward pass as a \emph{trajectory} of hidden states whose geometry can differ across outcomes.
We extend this view to vision under corruption and examine whether clean and corrupted inputs trace distinguishable trajectories through a ViT or CLIP image tower.
Across all models we test, features derived from the hidden trajectory separate clean and corrupted inputs more clearly as corruption severity increases.
This observation motivates \emph{Drift-TTA}, which adds zero-initialized low-rank adapters~\citep{hu2022lora} to a frozen backbone and trains them so that corrupted inputs produce hidden trajectories closer to those of their clean counterparts.
Figure~\ref{fig:trajectory-drift-correction} summarizes this intuition.
Corruption progressively displaces the hidden trajectory as severity increases, while Drift-TTA learns a correction in the opposite direction and shifts the resulting trajectory features back toward the clean distribution.

\begin{figure}[t]
  \centering
  % \includesvg[
  %     width=\linewidth,
  %     inkscapearea=drawing
  % ]{imgs/Hook_Figure_Final}
  \includegraphics[width=\linewidth]{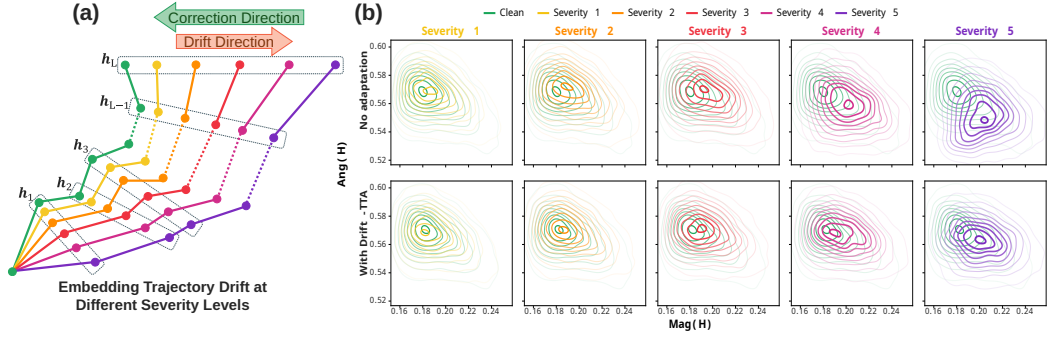}
  \caption{
  \textbf{Corruption progressively shifts hidden trajectories, while Drift-TTA moves them back toward clean.}
  \textbf{(a)} Schematic of severity-dependent trajectory drift and the correction learned by Drift-TTA.
  \textbf{(b)} Distributions of the trajectory magnitude and angle features for clean and corrupted inputs at severities~1--5, before and after applying Drift-TTA.
  Each column compares clean inputs with one corruption severity.
  Without adaptation, the corrupted distributions move farther from clean as severity increases.
  With Drift-TTA, they shift back toward the clean distribution.
  }
  \label{fig:trajectory-drift-correction}
  \end{figure}

Our contributions are summarized as follows:
\begin{itemize}

\item \textbf{We introduce hidden trajectories as a vision-domain extension of the chain-of-embeddings framework previously studied in large language models:}
 A hidden trajectory captures the ordered sequence of mean-pooled representations across backbone layers, and we show that trajectory-derived features increasingly separate clean from corrupted inputs as corruption severity grows across all tested vision models.

\item \textbf{We introduce hidden-trajectory alignment as a deployment-efficient mechanism for corruption correction:}
Drift-TTA instantiates this idea by attaching zero-initialized low-rank adapters to a frozen ViT or CLIP image encoder and training them offline to align corrupted hidden trajectories with those of cleaner inputs. The formulation supports both source-dependent adaptation using clean targets and source-free adaptation using corrupted/further-corrupted image pairs, with the latter requiring no clean source data. Once trained, the adapters remain fixed at inference, eliminating test-time gradients and per-sample optimization. The same formulation extends directly to CLIP zero-shot classification.

\item \textbf{We introduce an online variant of Drift-TTA that adapts the strength of this learned trajectory correction at test time:}
Drift-TTA-online keeps the backbone and trained adapters frozen and updates only one scalar per adapter branch to control its correction strength.
Each scalar scales the correction vector produced by its adapter and is optimized using prediction entropy on each unlabeled target batch.
The same procedure applies to both source-dependent and source-free adapters.

\item \textbf{We provide a probabilistic interpretation and theoretical bounds for hidden-trajectory alignment:}
We interpret the trajectory-alignment objective as regularized MAP estimation under a Gaussian observation model. We then derive a KL-based transfer bound and an explicit Gaussian form relating test-time residual error to offline residual error, residual dimensionality, and offline-to-test distribution shift.

\item \textbf{We establish that hidden-trajectory alignment consistently improves corruption robustness and calibration across models:}
On CIFAR-100-C at severity 5, source-dependent Drift-TTA-online improves mean accuracy by $16.6\%$ and $13.0\%$ over the unadapted ViT-B/16 and CLIP ViT-B/16, respectively, while exceeding the strongest compared TTA baselines by $2.5\%$ and $6.2\%$. It also improves calibration, reducing mean ECE relative to these baselines by $9.94\%$ for ViT and $9.69\%$ for CLIP.

\end{itemize}

\section{Background and Related Work}
\label{sec:background}

This section introduces the test time adaptation setting, reviews existing methods, and summarizes Chain-of-Embedding~\citep{coe2024}, which motivates this work.

\subsection{Test Time Adaptation}
\label{sec:related-tta}

Test time adaptation modifies a pretrained model using unlabeled test samples during test time.

The labels of the test samples are never available for adaptation, and many methods
also assume that the data used to train the pretrained model are unavailable at
test time~\citep{wang2021tent}.

Let $f_{\omega}$ denote a model with pretrained parameters $\omega$, trained on source inputs $x_s$ and labels $y_s$.
At test time, the model receives inputs $x_t$ without their labels $y_t$.
Let $\alpha$ denote the variables adapted at test time, and let $\mathcal{L}_t$ be an unlabeled test time objective.

\begin{definition}[Test time adaptation]
\label{def:testtimeadaptation}
Optimization-based test time adaptation estimates
\begin{equation}
    \alpha^{\star}
    =
    \arg\min_{\alpha}
    \mathcal{L}_t\bigl(f_{\omega,\alpha}(x_t)\bigr),
    \qquad
    \hat{y}_t
    =
    f_{\omega,\alpha^{\star}}(x_t),
    \label{testtimeloss}
\end{equation}
without using $y_t$. Adaptation is \emph{online} when each test batch $x_t$ is used as it arrives.
\end{definition}

Existing methods differ mainly in the objective $\mathcal{L}_t$ that guides adaptation, the variables $\alpha$ they adapt, and how they use the test inputs $x_t$.
A common approach is to use the model predictions themselves to guide adaptation.
Tent~\citep{wang2021tent} minimizes prediction entropy while updating the affine parameters of normalization layers.
EATA~\citep{niu2022efficient} builds on entropy minimization by selecting reliable test samples and regularizing important model parameters using Fisher information, while SAR~\citep{niu2023towards} combines entropy minimization with sample filtering and sharpness-aware optimization to improve stability.
For continual test streams, CMF~\citep{lee2024continual} further filters the information used during adaptation to reduce forgetting.
% Other methods use related signals from the model output or learn how the update itself should be performed~\citep{choi2025aea,deng2025mgg}.
Other methods use related signals from the model output, learn how the update itself should be performed, or use auxiliary tasks during adaptation~\citep{choi2025aea,deng2025mgg,alfarra2026sbat}.
Despite these differences, the model predictions remain the main signal used for adaptation.

Other methods use information from the model representations themselves.
FOA~\citep{niu2024test} combines prediction entropy with feature statistics and adapts prompts and activations using derivative-free optimization.
PEA~\citep{ma2026architecture} instead uses feature statistics across intermediate layers to progressively align hidden representations without backpropagation.
Related work also operates on latent features or their geometry~\citep{murphy2026neo,bao2025mint,lei2025ttvd,cho2025fata,chen2022contrastive,wang2023feature}.
These methods show that intermediate representations can also provide useful signals for adaptation.

For vision-language models such as CLIP~\citep{radford2021clip}, TTA has also been extended to prompts, tokens, memories, and class representations~\citep{zhang2025otvp,wang2025tca,bao2025latte}.
TPT~\citep{shu2022test} adapts the text prompt for each test image using prediction entropy.
TDA~\citep{karmanov2024efficient} and DMN-ZS~\citep{zhang2024dmn} instead accumulate information from test samples using feature caches or dynamic memories.
BCA~\citep{zhou2025bayesian} adapts class embeddings together with their priors, while DPE~\citep{zhang2024dpe} evolves visual and textual class prototypes using entropy and cross-modal alignment.
More recently, BAT~\citep{maharana2025batclip} adapts the LayerNorm parameters of both the vision and text encoders using entropy and alignment, while BITTA~\citep{sun2026bilateral} additionally uses high-entropy samples for unlearning.
Other work also considers how adaptation is applied across layers or maintained over longer test streams~\citep{maharana2025palm,vray2025reservoirtta}.

These methods are commonly evaluated on the \mbox{-C} benchmarks~\citep{hendrycks2019benchmarking}, including CIFAR-10-C, CIFAR-100-C, and ImageNet-C, which apply fifteen corruption families at five severity levels.
Table~\ref{tab:tta-comparison} compares representative methods by their adaptation signal, updated model components, and test-time setting.

\begin{table}[t]
    \centering
    \footnotesize
    \setlength{\tabcolsep}{3pt}

    \caption{\textbf{Comparison of representative TTA methods.}
    ``Adaptation signal'' denotes the information used to guide adaptation,
    while ``Modules updated'' denotes what is changed during adaptation.
    We describe the test-time setting along three axes:
    ``Access'' indicates whether adaptation uses test samples as they arrive
    (online) or after they are available together (offline).
    ``Granularity'' indicates whether the method can adapt using a single test
    sample (instance), a group of test samples (batch), or both.
    ``Persistence'' indicates whether test-time changes are reset after adaptation
    (episodic) or kept for later test samples (continual).
    ``Multi-layer'' indicates whether representations from multiple network
    layers are used during adaptation.}

    \label{tab:tta-comparison}

    \resizebox{\linewidth}{!}{%
    \begin{tabular}{@{}l l l c c c c@{}}
        \toprule
        & & &
        \multicolumn{3}{c}{Test-time setting}
        & \\
        \cmidrule(lr){4-6}

        Method
        & Adaptation signal
        & Modules updated
        & Access
        & Granularity
        & Persistence
        & Multi-layer \\
        \midrule

        \textsc{Tent}~\citep{wang2021tent}
        & Entropy
        & Norm.\ affine
        & Online
        & Batch
        & Continual
        & \\

        SAR~\citep{niu2023towards}
        & Entropy + sharpness
        & Norm.\ affine
        & Online
        & Batch
        & Continual
        & \\

        EATA~\citep{niu2022efficient}
        & Entropy + Fisher
        & Norm.\ affine
        & Online
        & Batch
        & Continual
        & \\

        FOA~\citep{niu2024test}
        & Entropy + feature statistics
        & Prompt + activations
        & Online
        & Batch
        & Continual
        & \checkmark \\

        CMF~\citep{lee2024continual}
        & Entropy + filtering
        & Model parameters
        & Online
        & Batch
        & Continual
        & \\

        PEA~\citep{ma2026architecture}
        & Feature statistics
        & Hidden features
        & Online
        & Batch
        & Continual
        & \checkmark \\

        \midrule

        TPT~\citep{shu2022test}
        & Entropy
        & Text prompt
        & Online
        & Instance
        & Episodic
        & \\

        TDA~\citep{karmanov2024efficient}
        & Entropy + pseudo-labels
        & Feature cache
        & Online
        & Instance
        & Continual
        & \\

        BCA~\citep{zhou2025bayesian}
        & Likelihood + prior
        & Class embeddings + priors
        & Online
        & Instance
        & Continual
        & \\

        DMN-ZS~\citep{zhang2024dmn}
        & Entropy + feature memory
        & Dynamic memory
        & Online
        & Instance
        & Continual
        & \\

        DPE~\citep{zhang2024dpe}
        & Entropy + alignment
        & Class prototypes
        & Online
        & Instance
        & Continual
        & \\

        BAT~\citep{maharana2025batclip}
        & Entropy + alignment
        & Vision/text LayerNorm
        & Online
        & Batch
        & Continual
        & \\

        BITTA~\citep{sun2026bilateral}
        & Entropy + unlearning
        & Vision/text LayerNorm
        & Online
        & Batch
        & Continual
        & \\

        \midrule

        % \textbf{Drift-TTA (ours)}
        Drift-TTA (ours)
        & Hidden trajectory alignment
        & LoRA adapters
        & Offline
        & Instance
        & --
        & \checkmark \\

        % \textbf{Drift-TTA-online (ours)}
        Drift-TTA-online (ours)
        & Hidden trajectory alignment
        & Scalar gates
        & Online
        & Instance / Batch
        & Episodic
        & \checkmark \\

        \bottomrule
    \end{tabular}%
    }

\end{table}

Table~\ref{tab:tta-comparison} also shows how methods differ in their use of test samples.
Online, instance-level, episodic adaptation processes each sample as it arrives and resets before the next, making it more constrained than offline, batch-level, or continual settings.
Drift-TTA-online supports this setting and can also operate on batches when multiple samples are available.

Overall, existing TTA methods use a range of adaptation signals, including model outputs, feature statistics, memories, pseudo-labels, cross-modal alignment, and intermediate representations.

\paragraph{Chain-of-Embedding.}
Chain-of-Embedding (CoE)~\citep{coe2024} studies how hidden representations evolve across layers in large language models.
It shows that geometric summaries of this sequence can separate correct from incorrect generations.
We transfer this view from language model outcomes to vision models under corrupted inputs.
Specifically, we compare clean and corrupted versions of the same input across network depth to study how corruption changes the learned representations.

\section{Drift-TTA: Low-Rank Alignment of Hidden Trajectories}
\label{sec:method}
\label{sec:trajectory}

Drift-TTA operates on the sequence of hidden representations produced across network depth.
For an input $\mathbf{x}$, let $\mathbf{z}_l$ denote the spatial representation after block $l$ of a backbone with $L$ blocks.
We define the hidden trajectory as
\begin{equation}
    \mathcal{H}(\mathbf{x})
    =
    \bigl(\mathbf{h}_1,\ldots,\mathbf{h}_L\bigr),
    \qquad
    \mathbf{h}_l
    =
    \operatorname{MeanPool}\bigl(\mathbf{z}_l\bigr)
    \in\mathbb{R}^d,
    \label{eq:trajectory}
\end{equation}
where mean pooling excludes the class or prefix token for transformer backbones.

For analysis, we include the input embedding $\mathbf{h}_0$ and define
\begin{equation}
    \mathrm{Mag}(H)
    =
    \frac{
    \frac{1}{L}\sum_{l=0}^{L-1}\|\mathbf{h}_{l+1}-\mathbf{h}_l\|_2
    }{
    \|\mathbf{h}_L-\mathbf{h}_0\|_2
    },
    \qquad
    \mathrm{Ang}(H)
    =
    \frac{
    \frac{1}{L}\sum_{l=0}^{L-1}\angle(\mathbf{h}_{l+1},\mathbf{h}_l)
    }{
    \angle(\mathbf{h}_L,\mathbf{h}_0)
    },
    \label{eq:trajectory-features}
\end{equation}
where
$\angle(\mathbf{a},\mathbf{b})
=
\arccos\!\left(
\frac{\mathbf{a}^{\top}\mathbf{b}}
{\|\mathbf{a}\|_2\|\mathbf{b}\|_2}
\right)$.
These features are used only to analyze and visualize how corruption changes intermediate representations across network depth; training directly aligns the hidden states $\mathbf{h}_l$.

Drift-TTA adds zero-initialized low-rank adapters within the frozen backbone and trains them to move corrupted hidden trajectories toward those produced by cleaner inputs.
The same trajectory-alignment objective is used for source-dependent and source-free training.
Table~\ref{tab:settings} summarizes these settings, and Figure~\ref{fig:drift-tta-overview} illustrates the method.

\begin{table}[t]
    \centering
    \small
    \caption{\textbf{Comparison of adaptation paradigms.}
    The table compares the data required beyond the pretrained model and when
    adaptation is performed.
    For Drift-TTA, source dependent and source free denote two training settings:
    source dependent uses clean images as trajectory targets for corrupted inputs,
    while source free uses corrupted and further corrupted images without clean images
    or labels.
    After training, either setting can be used directly with the learned adapters fixed
    or with Drift-TTA-online, which additionally updates only scalar gates at test time.}
    \label{tab:settings}

    \scriptsize
    \setlength{\tabcolsep}{8pt}
    \begin{tabular}{@{}l c c c c c@{}}
        \toprule
        & \multicolumn{3}{c}{Data required}
        & \multicolumn{2}{c}{When adaptation occurs} \\
        \cmidrule(lr){2-4}
        \cmidrule(lr){5-6}

        Paradigm
        & Clean
        & Clean
        & Unlabeled
        & Before
        & At test \\
        & images
        & labels
        & corrupted images
        & testing
        & time \\
        \midrule

        Fine tuning
        & \checkmark
        & \checkmark
        &
        & \checkmark
        & \\

        Unsupervised DA
        & \checkmark
        & \checkmark
        & \checkmark
        & \checkmark
        & \\

        Test time training
        & \checkmark
        & \checkmark
        & \checkmark
        &
        & \checkmark \\

        Online TTA (e.g.\ \textsc{Tent})
        &
        &
        & \checkmark
        &
        & \checkmark \\

        \midrule

        Drift-TTA, source dependent
        & \checkmark
        &
        & \checkmark
        & \checkmark
        & \\

        Drift-TTA, source free
        &
        &
        & \checkmark
        & \checkmark
        & \\

        Drift-TTA-online, source dependent
        & \checkmark
        &
        & \checkmark
        & \checkmark
        & \checkmark \\

        Drift-TTA-online, source free
        &
        &
        & \checkmark
        & \checkmark
        & \checkmark \\

        \bottomrule
    \end{tabular}
\end{table}

\begin{figure}[t]
  \centering
  \includegraphics[width=\linewidth]{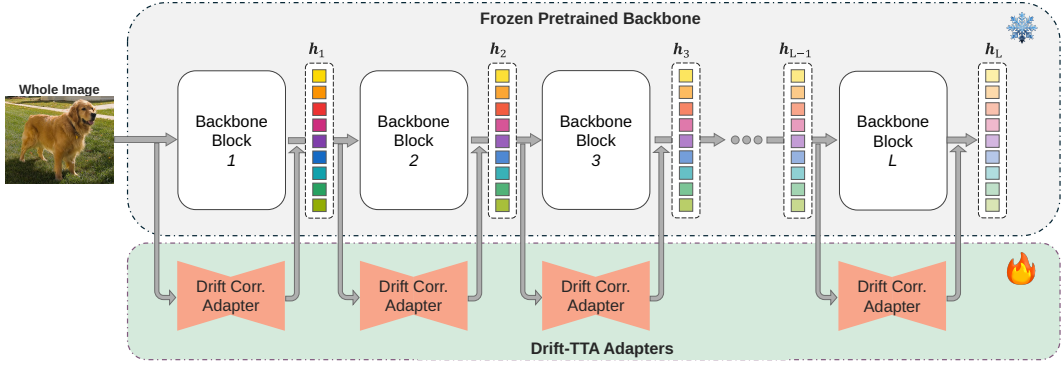}
  \caption{
\textbf{Architecture of Drift-TTA.}
A frozen pretrained backbone is augmented with zero-initialized low-rank adapters attached to selected linear maps.
The mean-pooled representations $\mathbf{h}_1,\ldots,\mathbf{h}_L$ produced across the backbone form the hidden trajectory.
The learned adapters move corrupted trajectories toward their clean counterparts while keeping the pretrained backbone frozen.
After training, the adapters can be used directly with fixed weights.
For online adaptation, Drift-TTA-online keeps the backbone and adapter weights frozen and updates only the strength of each learned correction.
}
  \label{fig:drift-tta-overview}
  \end{figure}

\subsection{Low-Rank Trajectory Adapters}
\label{ssec:adapter}
\label{ssec:placement}

Each adapter is a trainable residual branch on a linear map of the frozen backbone.
Let $\mathcal{S}\subseteq\{1,\ldots,L\}$ denote the adapted blocks and $\mathcal{L}_l$ the selected linear maps in block $l$, so that $\left(\mathcal{S},\{\mathcal{L}_l\}_{l\in\mathcal{S}}\right)$ defines the adapter placement.
For transformer backbones, $\mathcal{L}_l$ contains token-wise linear maps within the attention and MLP sublayers.
For CLIP, the same procedure is applied only to the image encoder, while the text encoder and class prototypes remain frozen.
The $q$-th map in block $l$ computes $\mathbf{y}_{l,q}=\mathbf{W}_{l,q}\mathbf{x}_{l,q}+\mathbf{b}_{l,q}$, with weights $\mathbf{W}_{l,q}\in\mathbb{R}^{d_{\mathrm{out}}\times d_{\mathrm{in}}}$, input $\mathbf{x}_{l,q}\in\mathbb{R}^{d_{\mathrm{in}}}$, and bias $\mathbf{b}_{l,q}\in\mathbb{R}^{d_{\mathrm{out}}}$.
Following LoRA~\citep{hu2022lora}, a rank-$r$ adapter adds the correction
\begin{equation}
\Delta_{l,q}\!\left(\mathbf{x}_{l,q}\right)
=
\alpha_{l,q}\mathbf{A}_{l,q}\mathbf{B}_{l,q}\mathbf{x}_{l,q},
\qquad
\mathbf{B}_{l,q}\in\mathbb{R}^{r\times d_{\mathrm{in}}},
\qquad
\mathbf{A}_{l,q}\in\mathbb{R}^{d_{\mathrm{out}}\times r},
\label{eq:adapter-branch}
\end{equation}
where $\mathbf{A}_{l,q}$ and $\mathbf{B}_{l,q}$ are trainable low-rank weights and the scalar gate $\alpha_{l,q}$ sets the strength of the correction.
The adapted output is $\widetilde{\mathbf{y}}_{l,q}=\mathbf{y}_{l,q}+\Delta_{l,q}(\mathbf{x}_{l,q})$, where the effective correction $\mathbf{A}_{l,q}\mathbf{B}_{l,q}$ has rank at most $r$ and $\mathbf{W}_{l,q}$ and $\mathbf{b}_{l,q}$ remain fixed.
We initialize $\mathbf{A}_{l,q}=\mathbf{0}$ and $\mathbf{B}_{l,q}$ with small random values, so the adapted model initially produces the same computation as the frozen backbone.
Only the adapter weights receive gradients during offline training, with $\alpha_{l,q}=1$.
Once trained, these weights remain fixed and are used either directly with unit strength or with gates adapted online (Section~\ref{ssec:online-tta}).

\subsection{Trajectory Alignment Objective}
\label{ssec:objective}

Each training example is a triplet of a student input $\mathbf{x}$, which may be augmented or corrupted, a cleaner teacher view $\mathbf{x}^{\mathrm{tgt}}$, and an anchor view $\mathbf{x}^{\mathrm{anc}}$.
The adapted backbone processes $\mathbf{x}$, while the frozen backbone processes $\mathbf{x}^{\mathrm{tgt}}$ to provide the reference trajectory and decision.
Let $\mathbf{h}_l^{\theta}(\mathbf{x})$ and $\mathbf{h}_l^{0}(\mathbf{x}^{\mathrm{tgt}})$ denote the mean-pooled hidden states of the adapted and frozen models after block $l$, where $\theta$ contains all adapter parameters.
Their decisions $\mathbf{o}^{\theta}(\mathbf{x})$ and $\mathbf{o}^{0}(\mathbf{x}^{\mathrm{tgt}})$ are classifier logits for supervised ViTs.
Since CLIP has no task-specific classification head, its decision is instead the vector of zero-shot image--text similarity scores.
We omit the learned temperature from these scores so that their scale does not dominate the trajectory terms, which leaves the predicted class unchanged.
With the aligned blocks $\mathcal{T}$ running from the first adapted block to block $L$, the alignment loss matches both the trajectory and the decision:
\begin{equation}
\mathcal{L}_{\mathrm{align}}(\theta)
=
\sum_{l\in\mathcal{T}}w_l\bigl\|\mathbf{h}_l^{\theta}(\mathbf{x})-\operatorname{sg}\bigl[\mathbf{h}_l^{0}(\mathbf{x}^{\mathrm{tgt}})\bigr]\bigr\|_2^2
+\lambda_{\mathrm{dec}}\bigl\|\mathbf{o}^{\theta}(\mathbf{x})-\operatorname{sg}\bigl[\mathbf{o}^{0}(\mathbf{x}^{\mathrm{tgt}})\bigr]\bigr\|_2^2,
\label{eq:alignment-loss}
\end{equation}
where $\operatorname{sg}[\cdot]$ stops gradients through the frozen teacher.
The effective low-rank corrections are regularized by
\begin{equation}
\mathcal{L}_{\mathrm{ctrl}}(\theta)
=
\sum_{l\in\mathcal{S}}
\sum_{q\in\mathcal{L}_l}
\left\|
\mathbf{A}_{l,q}\mathbf{B}_{l,q}
\right\|_F^2.
\label{eq:control-loss}
\end{equation}
Because the adapters remain active for every input, the quietness loss uses the anchor view to limit unnecessary changes to inputs that do not require further correction:
\begin{equation}
\mathcal{L}_{\mathrm{quiet}}(\theta)
=
\sum_{l\in\mathcal{T}}w_l\bigl\|\mathbf{h}_l^{\theta}(\mathbf{x}^{\mathrm{anc}})-\operatorname{sg}\bigl[\mathbf{h}_l^{0}(\mathbf{x}^{\mathrm{anc}})\bigr]\bigr\|_2^2
+\lambda_{\mathrm{dec}}\bigl\|\mathbf{o}^{\theta}(\mathbf{x}^{\mathrm{anc}})-\operatorname{sg}\bigl[\mathbf{o}^{0}(\mathbf{x}^{\mathrm{anc}})\bigr]\bigr\|_2^2.
\label{eq:quiet-loss}
\end{equation}
The complete training objective is
\begin{equation}
\mathcal{L}(\theta)
=
\mathcal{L}_{\mathrm{align}}(\theta)
+
\lambda_{\mathrm{ctrl}}\mathcal{L}_{\mathrm{ctrl}}(\theta)
+
\lambda_{\mathrm{quiet}}\mathcal{L}_{\mathrm{quiet}}(\theta).
\label{eq:objective}
\end{equation}
This objective is shared by two training settings, which differ only in how the student, teacher, and anchor views are constructed.

% MOVE TO EXPERIMENTAL SETUP:
% We use $w_l=1$, $\lambda_{\mathrm{dec}}=1$, $\lambda_{\mathrm{ctrl}}=10^{-3}$, and $\lambda_{\mathrm{quiet}}=0.05$ by default.

\paragraph{Source-dependent training.}
\label{ssec:targets}
Source-dependent training starts from a clean source image $\mathbf{x}^{\mathrm{clean}}$ and corrupts it with a corruption operator $\mathcal{C}$, giving the student input $\mathbf{x}=\mathcal{C}(\mathbf{x}^{\mathrm{clean}})$.
The clean image serves as both the teacher view and the anchor, $\mathbf{x}^{\mathrm{tgt}}=\mathbf{x}^{\mathrm{anc}}=\mathbf{x}^{\mathrm{clean}}$.
The adapters therefore learn to move the corrupted trajectory back toward the hidden trajectory and decision of the clean forward pass.

\paragraph{Source-free training.}
\label{ssec:duality}
When clean images are unavailable, source-free training starts from a corrupted image $\mathbf{x}^{\mathrm{cur}}$ and corrupts it further, giving the student input $\mathbf{x}=\mathbf{x}^{+}=\operatorname{Worsen}(\mathbf{x}^{\mathrm{cur}})$ with the same corruption family at a higher effective severity.
The current image then serves as both the cleaner teacher view and the anchor, $\mathbf{x}^{\mathrm{tgt}}=\mathbf{x}^{\mathrm{anc}}=\mathbf{x}^{\mathrm{cur}}$, so no clean images or class labels are required.
The adapters thus learn to move the trajectory of $\mathbf{x}^{+}$ back toward that of $\mathbf{x}^{\mathrm{cur}}$, although this target trajectory remains corrupted.
At deployment, the learned correction is applied to the current corrupted input and partially reverses the same direction of displacement.
For stochastic corruptions, the worsening operation preserves the underlying corruption realization whenever possible, since an independent draw can change the displacement direction and weaken the teacher correspondence.

\subsection{Online Test-Time Adaptation}
\label{ssec:online-tta}
\label{ssec:clip}

Adapters trained in either setting apply each learned correction with fixed strength, $\alpha_{l,q}=1$.
Drift-TTA-online instead adapts the gates $\alpha_{l,q}$ in~\eqref{eq:adapter-branch} at test time, while the backbone and adapter weights remain frozen.
Collecting these gates as $\alpha=\{\alpha_{l,q}\}$ gives the test-time variables in Definition~\ref{def:testtimeadaptation}.
Setting $\alpha_{l,q}=0$ disables a correction.
This adds one online parameter per adapter branch: $36$ gates for the supervised ViTs and $24$ for CLIP in our $12$-block models.
For each unlabeled test batch $\mathcal{B}_t$, the gates are initialized to one and updated by minimizing the prediction entropy
\begin{equation}
\mathcal{L}_{\mathrm{online}}\!\left(\alpha;\mathcal{B}_t\right)
=
-\frac{1}{|\mathcal{B}_t|}
\sum_{\mathbf{x}\in\mathcal{B}_t}
\sum_c
p_{\alpha,c}(\mathbf{x})
\log p_{\alpha,c}(\mathbf{x}),
\label{eq:online-entropy}
\end{equation}
where $\mathbf{p}_{\alpha}(\mathbf{x})=\operatorname{softmax}(\mathbf{s}_{\alpha}(\mathbf{x}))$, $c$ indexes the output classes, and the prediction scores $\mathbf{s}_{\alpha}$ are the classifier logits for supervised ViTs and the temperature-scaled image--text similarity scores for CLIP.
After $K$ gradient updates, we predict the same batch with the calibrated gates and reset them before processing the next batch, which also supports per-image adaptation when $|\mathcal{B}_t|=1$.

\subsection{Probabilistic Interpretation and Theoretical Bounds}
\label{ssec:capacity}

\paragraph{Alignment as regularized Gaussian MAP.}
For a student--teacher--anchor triplet $z=(\mathbf x,\mathbf x^{\mathrm{tgt}},\mathbf x^{\mathrm{anc}})$, we can decompose \eqref{eq:objective} into the following:
\begin{equation}
\ell_\theta(z)=\mathcal L_{\mathrm{align}}(\theta;z)
+\lambda_{\mathrm{quiet}}\mathcal L_{\mathrm{quiet}}(\theta;z) + \lambda_{\mathrm{ctrl}}\mathcal L_{\mathrm{ctrl}}(\theta).
\label{eq:sample-loss-decomposition}
\end{equation}
If we stack the frozen teacher and anchor readings into $t$, and the corresponding adapted outputs into $f_\theta(v)$, where $v=(\mathbf x,\mathbf x^{\mathrm{anc}})$, then the data-fit term can be written as
$\ell_{\mathrm{data},\theta}(z)=\|t-f_\theta(v)\|_2^2$. Let $p(t \mid v,\theta)=\mathcal{N}\left(t; f_\theta(v),\frac{1}{2}I\right)$, and place the regularizing prior $\pi(\theta)\propto
\exp[-\lambda_{\mathrm{ctrl}}\mathcal{L}_{\mathrm{ctrl}}(\theta)]$,
assuming the prior is proper on the parameter domain.
Then, by Bayes' rule:
\begin{equation}
\operatorname*{arg\,min}_{\theta}\ell_\theta(z)
=\operatorname*{arg\,max}_{\theta}p(\theta\mid t,v).
\label{eq:alignment-map}
\end{equation}

\paragraph{KL variational bound for the test-time residual.}
For a learned parameter vector $\theta$, let $r_\theta(z)=t-f_\theta(v)\in\mathbb R^{d_t}$ be its weighted residual, and let $P_\theta$ and $Q_\theta$ be the residual distributions induced by offline and test examples, respectively. We define:
\begin{equation}
R_{\mathrm{off}}(\theta)=\mathbb E_{P_\theta}\|r\|^2,
\qquad R_{\mathrm{te}}(\theta)=\mathbb E_{Q_\theta}\|r\|^2,
\qquad \kappa_\theta=\mathrm{KL}(Q_\theta\Vert P_\theta).
\label{eq:population-risks}
\end{equation}

We further set $\eta>0$ with finite $Z_\eta=\mathbb E_{P_\theta}e^{\eta\|r\|^2}$, let the tilted distribution $dP_{\theta,\eta}=e^{\eta\|r\|^2}dP_\theta/Z_\eta$, Using $\kappa_\theta-\eta R_{\mathrm{te}}(\theta)+\log Z_\eta=\mathrm{KL}(Q_\theta\|P_{\theta,\eta})\ge 0$, we obtain the following KL change-of-measure bound:
\begin{equation}
R_{\mathrm{te}}(\theta)
\le
R_{\mathrm{off}}(\theta)
+
\inf_{\eta>0}
\frac{
\kappa_\theta
+
\log
\mathbb{E}_{P_\theta}
e^{\eta(\|r\|_2^2-R_{\mathrm{off}}(\theta))}
}{
\eta
}.
\label{eq:offline-test-bound}
\end{equation}

% where
% \begin{equation}
% \kappa_\theta-\eta R_{\mathrm{te}}(\theta)+\log Z_\eta=\mathrm{KL}(Q_\theta\Vert P_{\theta,\eta})\ge0
% \label{eq:shift-kl-identity}
% \end{equation}

\paragraph{Convert into an explicit Gaussian bound.}
If we further assume the actual offline residual law is Gaussian $P_\theta=\mathcal N(0,\sigma^2 I)$, with $\sigma^2>0$, substituting this into~\eqref{eq:offline-test-bound} and choosing $\eta=\sqrt{\kappa_\theta}/[\sigma^2(\sqrt{d_t}+2\sqrt{\kappa_\theta})]$ gives
\begin{equation}
\boxed{R_{\mathrm{te}}(\theta)\le R_{\mathrm{off}}(\theta)
+2\sigma^2\bigl(\sqrt{d_t\kappa_\theta}+\kappa_\theta\bigr).}
\label{eq:gaussian-transfer-bound}
\end{equation}
An intuitive explanation for \eqref{eq:gaussian-transfer-bound} is that the
test-time residual is bounded by the offline residual,
the residual dimensionality, and the KL divergence between
the offline and test residual distributions. Note that these bounds apply to fixed adapters under additional residual-distribution assumptions; they do not guarantee clean-trajectory recovery, online adaptation gains, or improved accuracy. Detailed proof is given in Appendix~\ref{app:theory}.

\section{Experiments}
\label{sec:experiments}

Drift-TTA is applied to both unimodal classifiers and CLIP-based classification models. In this section, we report comprehensive validation for our approach. The experimental setting is described in Section~\ref{ssec:experimental_setup}. Main results are presented in Section~\ref{ssec:main-results}.

\subsection{Experimental Setup}
\label{ssec:experimental_setup}

\textbf{Datasets and models.}
We evaluate Drift-TTA on CIFAR-10-C, CIFAR-100-C, and ImageNet-C~\citep{hendrycks2019benchmarking}, each containing 15 corruption types at five severity levels. For unimodal classification, ViT-Tiny is used on CIFAR-10-C, while ViT-B/16 is used on CIFAR-100-C and ImageNet-C. Each model is pretrained on the corresponding clean dataset. For CLIP-based classification, we use a pretrained CLIP ViT-B/16 model. Both source-dependent and source-free variants of Drift-TTA are evaluated.

\textbf{Training Protocol.}
Drift-TTA adds a lightweight LoRA adapter (rank 4, on all transformer blocks) to a frozen backbone. For each corruption family we train one adapter on the other 14 families (all five severities) and evaluate on the held-out family; Table \ref{tab:main_results_cifar100} reports this leave-one-family-out accuracy. The adapter weights are trained offline and remain frozen at test time. Drift-TTA-online updates only their scalar gates using prediction entropy on the current unlabeled target batch. Source-dependent is trained for 4000 steps, while source-free adapters are trained for 2400 steps. A held-out validation split is used only to monitor convergence. For supervised ViT runs, the final checkpoint is reported. For CLIP, we select among epochs {1,3,5,7,9} using unsupervised validation loss; no target labels are used for model selection..

\textbf{Evaluation.}
Each adapted model is compared with the corresponding model without adaptation on the same corrupted images. Source-dependent experiments use $2{,}000$ CIFAR and $5{,}000$ ImageNet images, while source-free experiments use $10{,}000$ CIFAR and $50{,}000$ ImageNet images. CLIP experiments use $10{,}000$ CIFAR and $5{,}000$ ImageNet images. We primarily report the change in accuracy relative to no adaptation, denoted by $\Delta$pp. Recovery and changes in clean accuracy are also reported when applicable. All experiments use a single seed, with an expected run-to-run variation of approximately $0.3\%$.

\subsection{Main Results}
\label{ssec:main-results}

\begin{table}[!t]
    \centering
    \definecolor{bestgreen}{RGB}{0,120,55}
    \definecolor{secondblue}{RGB}{30,90,170}
    \newcommand{\best}[1]{\textcolor{bestgreen}{\textbf{#1}}}
    \newcommand{\second}[1]{\textcolor{secondblue}{\underline{#1}}}

    \caption{
    \textbf{Test-time adaptation accuracy (\%) on CIFAR-100-C at severity 5.}
    Results are reported across all fifteen corruption types for supervised ViT-B/16
    and CLIP ViT-B/16. Drift-TTA-online combines a frozen trajectory-repair adapter
    with online entropy-based gate calibration under source-dependent (SD) and
    source-free (SF) training. The best result in each column is shown in
    \textcolor{bestgreen}{\textbf{green}} and the second best in
    \textcolor{secondblue}{\underline{blue}}.}
    \label{tab:main_results_cifar100}

    % \vspace{-0.4em}

    \begingroup
    \fontsize{5.8}{6.2}\selectfont
    \renewcommand{\arraystretch}{0.80}
    \setlength{\tabcolsep}{1.0pt}

    \newcommand{\rot}[1]{\rotatebox{65}{#1}}
    \newcommand{\NA}{--}
    \newcommand{\emptyvals}{%
        \NA & \NA & \NA & \NA & \NA &
        \NA & \NA & \NA & \NA & \NA &
        \NA & \NA & \NA & \NA & \NA & \NA%
    }

    \begin{tabular*}{\textwidth}{
        @{\extracolsep{\fill}}
        ccl*{16}{c}
        @{}
    }

    \toprule
    &
    & \textbf{Method}
    & \rot{Gaussian} & \rot{Shot} & \rot{Impulse} & \rot{Defocus} & \rot{Glass}
    & \rot{Motion} & \rot{Zoom} & \rot{Snow} & \rot{Frost} & \rot{Fog}
    & \rot{Brightness} & \rot{Contrast} & \rot{Elastic} & \rot{Pixelate} & \rot{JPEG}
    & \textbf{Mean} \\
    \midrule

    \multirow{20}{*}{\rotatebox{90}{\textbf{CIFAR-100-C}}}

    % ------------------------- ViT -------------------------
    & \multirow{10}{*}{\rotatebox{90}{\textbf{ViT}}}
    & ViT-B/16
    & 42.0 & 42.4 & 54.1 & 73.1 & 47.5 & 72.8 & 75.0 & 80.5 & 79.7 & 57.1 & 86.0 & 44.0 & 64.5 & 47.7 & 57.5 & 61.6 \\
    & & TENT
    & 42.6 & 43.4 & 54.8 & 73.4 & 47.2 & 73.0 & 75.2 & 80.3 & 79.5 & 54.3 & 86.0 & 39.8 & 64.4 & 46.8 & 57.2 & 61.2 \\
    & & SAR
    & 42.9 & 44.5 & 54.0 & 73.3 & 47.6 & 73.0 & 74.9 & 79.5 & 79.1 & 53.6 & 85.4 & 46.3 & 63.9 & 46.9 & 55.4 & 61.3 \\
    & & EATA
    & 42.5 & 43.8 & 54.5 & 73.2 & 47.9 & 73.0 & 75.0 & 80.1 & 79.7 & 55.6 & 85.7 & 42.0 & 65.1 & 48.9 & 57.2 & 61.6 \\
    & & FOA
    & 43.5 & 49.0 & 55.3 & 74.3 & 55.3 & 75.8 & 78.5 & 81.0 & 81.9 & 76.0 & \second{87.2} & 76.2 & 70.1 & 55.6 & 60.7 & 68.0 \\
    & & CMF
    & 56.5 & 65.3 & 66.0 & 77.0 & 61.1 & 76.4 & 79.5 & 81.0 & 81.7 & 77.8 & 86.1 & \second{82.9} & 69.6 & 68.7 & \second{65.7} & 73.0 \\
    & & PEA
    & 62.1 & 64.8 & \second{71.5} & 81.2 & 65.9 & 80.0 & 82.6 & 82.7 & 83.6 & \best{82.1} & \best{87.4} & \best{86.4} & 70.7 & \second{70.8} & 64.4 & \second{75.7} \\
    % & & NEO
    % & \emptyvals \\
    & & Drift-TTA-online (SF)
    & \second{73.6} & \second{72.2} & 45.3 & \second{83.4} & \second{70.2} & \best{81.6} & \second{84.9} & \second{82.9} & \second{84.0} & \second{77.9} & 87.1 & 81.4 & \second{71.9} & 63.7 & 63.8 & 74.9 \\
    & & Drift-TTA-online (SD)
    & \best{77.3} & \best{78.5} & \best{75.0} & \best{84.5} & \best{71.9} & \second{81.6} & \best{85.1} & \best{83.7} & \best{84.8} & 75.2 & 85.6 & 75.1 & \best{75.6} & \best{71.9} & \best{66.9} & \best{78.2} \\

    \cmidrule(lr){2-19}

    % ------------------------- CLIP -------------------------
    & \multirow{10}{*}{\rotatebox{90}{\textbf{CLIP}}}
    & ViT-B/16
    & 19.6 & 21.4 & 25.3 & 42.5 & 20.1 & 43.2 & 48.0 & 48.4 & 49.7 & 41.7 & 57.0 & 34.4 & 29.2 & 23.9 & 32.5 & 35.8 \\
    & & TPT
    & 17.2 & 19.0 & 25.4 & 42.7 & 20.0 & 42.7 & 48.3 & 49.2 & 50.2 & 42.5 & 57.8 & 38.0 & 30.2 & 25.2 & 32.5 & 36.1 \\
    & & TDA
    & 23.4 & 24.3 & 32.6 & 42.3 & 20.2 & 42.0 & 46.9 & 48.5 & 48.3 & 41.4 & 57.0 & 32.0 & 33.0 & \second{33.4} & 33.4 & 37.2 \\
    & & BCA
    & 12.5 & 13.5 & 25.9 & 39.9 & 19.2 & 40.5 & 45.3 & 47.9 & 46.8 & 40.7 & 55.8 & 27.1 & 31.6 & 31.6 & 31.3 & 34.0 \\
    & & DMN-ZS
    & 19.8 & 22.0 & 29.8 & 42.5 & 18.3 & 43.4 & 48.3 & 49.1 & 51.1 & 43.9 & 56.7 & 35.8 & 31.8 & 26.5 & 31.3 & 36.7 \\
    & & DPE
    & 22.7 & 24.2 & 30.0 & 45.5 & 22.2 & 44.8 & 50.7 & 50.0 & 50.8 & 43.6 & 59.6 & 36.6 & 31.5 & 27.9 & 34.2 & 38.3 \\
    & & BAT
    & 21.4 & 25.2 & 29.9 & 50.4 & 25.7 & 48.2 & 54.9 & 52.1 & 50.7 & 48.2 & \second{63.5} & 43.3 & 34.8 & 31.8 & 37.2 & 41.2 \\
    & & BITTA
    & 26.7 & 28.7 & 35.2 & 50.7 & 26.1 & 49.0 & 55.2 & 52.6 & 51.4 & 48.7 & \best{63.7} & 46.6 & 35.3 & 30.8 & 37.8 & 42.6 \\
    & & Drift-TTA-online (SF)
    & \second{38.7} & \second{38.3} & \second{36.9} & \second{53.5} & \best{39.3} & \second{53.2} & \best{59.1} & \second{55.5} & \second{58.2} & \best{50.2} & 62.2 & \best{52.3} & \second{41.6} & \best{33.7} & \second{38.5} & \second{47.4} \\
    & & Drift-TTA-online (SD)
    & \best{45.3} & \best{47.8} & \best{42.9} & \best{54.3} & \second{39.1} & \best{53.2} & \second{58.0} & \best{57.3} & \best{59.2} & \second{49.0} & 61.7 & \second{48.2} & \best{43.7} & 30.4 & \best{41.6} & \best{48.8} \\

    \bottomrule
    \end{tabular*}
    \endgroup
\end{table}

\paragraph{Calibration.}
Accuracy alone does not show whether a model's confidence remains calibrated after adaptation.
Table~\ref{tab:ece_calibration_mean_summary} reports mean expected calibration error (ECE) for the unadapted models, PEA and BITTA, and the fixed and online variants of Drift-TTA.
For both source dependent and source free training, the better Drift-TTA variant has lower mean ECE than the corresponding prior TTA baseline across all model--dataset settings.
Appendix~\ref{app:ece} provides the metric definition, evaluation protocol, and complete per-corruption results.

\begin{table}[!t]
    \centering
    \caption{\textbf{Mean expected calibration error (ECE, \%, $\downarrow$) on the corruption benchmarks at severity 5.}
    PEA and BITTA are the ViT and CLIP baselines, respectively.
    ``Adapter only'' uses the frozen Drift-TTA adapter; ``+ Online'' adds entropy-based gate calibration.
    For each setting, $\Delta_{\mathrm{SD}}$ and $\Delta_{\mathrm{SF}}$ select the lower ECE of these two variants and report its reduction from the corresponding baseline in percentage.
    Green $\downarrow$ values indicate lower ECE.
    Teal, orange, purple, and gray denote no adaptation, source dependent (SD), source free (SF), and the prior TTA baselines.}
    \label{tab:ece_calibration_mean_summary}
    \begingroup
    \scriptsize
    \setlength{\tabcolsep}{4pt}
    \definecolor{ECEBase}{RGB}{223,240,238}
    \definecolor{ECEPrior}{RGB}{238,238,238}
    \definecolor{ECESD}{RGB}{255,241,224}
    \definecolor{ECESF}{RGB}{239,230,247}
    \definecolor{ECESDDark}{RGB}{210,105,0}
    \definecolor{ECESFDark}{RGB}{123,44,191}
    \definecolor{ECEGain}{RGB}{0,110,60}
    \newcommand{\gain}[1]{\textcolor{ECEGain}{\textbf{$\downarrow$#1}}}
    \begin{tabular}{@{}lcccccc@{}}
    \toprule
    & \multicolumn{3}{c}{\textbf{ViT}}
      & \multicolumn{3}{c}{\textbf{CLIP}} \\
    \cmidrule(lr){2-4}
    \cmidrule(lr){5-7}
    \textbf{Method / deployment}
      & \textbf{CIFAR-10-C} & \textbf{CIFAR-100-C} & \textbf{ImageNet-C}
      & \textbf{CIFAR-10-C} & \textbf{CIFAR-100-C} & \textbf{ImageNet-C} \\
    \midrule
    \cellcolor{ECEBase}No-TTA
      & 13.51 & 10.10 & 11.54 & 4.87 & 6.30 & 4.12 \\
    \cellcolor{ECEPrior}PEA (ViT) / BITTA (CLIP)
      & 11.44 & 18.00 & 33.44 & 13.58 & 14.70 & 18.04 \\
    \midrule
    \cellcolor{ECESD}Adapter only (SD)
      & 11.72 & 8.47 & 17.82 & 3.33 & 10.33 & 4.51 \\
    \cellcolor{ECESD}$+$ Online gate (SD)
      & 8.21 & 8.06 & 13.03 & 3.53 & 5.01 & 12.89 \\
    \cellcolor{ECESD}\textcolor{ECESDDark}{\textbf{$\Delta_{\mathrm{SD}}$}}
      & \gain{3.23} & \gain{9.94} & \gain{20.41}
      & \gain{10.25} & \gain{9.69} & \gain{13.53} \\
    \midrule
    \cellcolor{ECESF}Adapter only (SF)
      & 14.50 & 9.96 & 17.43 & 3.86 & 9.96 & 3.14 \\
    \cellcolor{ECESF}$+$ Online gate (SF)
      & 10.71 & 9.53 & 14.72 & 2.28 & 5.77 & 8.01 \\
    \cellcolor{ECESF}\textcolor{ECESFDark}{\textbf{$\Delta_{\mathrm{SF}}$}}
      & \gain{0.73} & \gain{8.47} & \gain{18.72}
      & \gain{11.30} & \gain{8.93} & \gain{14.90} \\
    \bottomrule
    \end{tabular}
    \endgroup
\end{table}

\paragraph{Accuracy.} 
Table 3 compares Drift-TTA with representative TTA baselines on CIFAR-100-C at severity 5. For ViT-B/16, source-dependent Drift-TTA-online reaches $78.2\%$ mean accuracy, compared with $61.6\%$ without adaptation and $75.7\%$ for the strongest compared baseline, PEA. For CLIP ViT-B/16, it reaches $48.8\%$, compared with $35.8\%$ without adaptation and $42.6\%$ for BITTA. Source-free Drift-TTA-online reaches $74.9\%$ and $47.4\%$, respectively. Full CIFAR-10-C and ImageNet-C results are provided in Appendix ~\ref{app:additional}.

\paragraph{Representation structure.}
Figure~\ref{fig:umap-clip-frost} provides a qualitative view of how Drift-TTA changes the CLIP image representation under Frost corruption.
Without adaptation, samples from several classes overlap in the projected space.
The fixed source dependent and source free adapters produce more distinct class groups, and online gate calibration further improves this separation.
The clearest structure is obtained with the source dependent adapter and online TTA.

\begin{figure}[!t]
    \centering
    % \includesvg[width=\linewidth,inkscapearea=drawing]{imgs/clip_frost}
    \includegraphics[width=\linewidth]{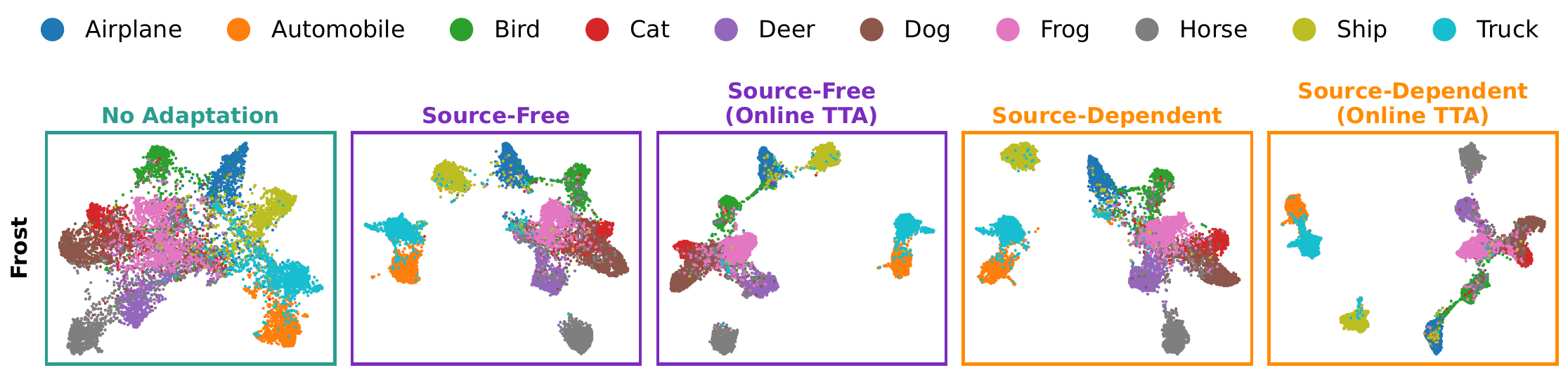}
    \caption{\textbf{UMAP of CLIP ViT-B/16 image embeddings on Frost-corrupted CIFAR-10-C at severity 5.}
    Columns show no adaptation, source-free Drift-TTA without and with online TTA, and source-dependent Drift-TTA without and with online TTA.
    Points are colored by their ground-truth class.
    Each Drift-TTA adapter is trained without Frost under the leave-one-family-out protocol, and each panel is projected independently using UMAP with Euclidean distance.}
    \label{fig:umap-clip-frost}
\end{figure}

\section{Conclusion}

In this work, we studied hidden trajectories of neural networks for test-time adaptation under distribution shift. We observed that clean and corrupted trajectories become more separable as corruption severity increases, motivating \emph{Drift-TTA}. Our approach uses low-rank adapters to correct trajectory drift while keeping the pretrained backbone frozen. Drift-TTA supports both source-dependent and source-free training, and can be used either as a fixed correction or with lightweight online adaptation using unlabeled test data. We also provided a probabilistic interpretation of trajectory alignment and theoretical bounds for offline generalization and source-free correction. Experiments on CIFAR-10-C, CIFAR-100-C, and ImageNet-C with supervised ViTs and CLIP demonstrate improved robustness to corruption, while reducing expected calibration error (ECE) relative to prior TTA baselines. Overall, these results highlight hidden trajectory alignment as an effective signal for test-time adaptation.

\section*{Acknowledgments}
We sincerely thank our colleagues Hesam Salehipour and Mehran Ebrahimi for their thoughtful feedback, valuable discussions, and insightful suggestions, which helped improve this work.

\section*{AI Use Statement}
We used generative AI tools to improve the clarity and readability
of the manuscript, assist in writing mathematical proofs, and
summarize existing literature. All AI-assisted content was reviewed
and verified by the authors. The authors take full responsibility
for the final content, including the mathematical arguments,
claims, and references.

\bibliography{ref}
\bibliographystyle{plainnat}

\appendix

\section{Method Details}
\label{app:implementation}

\subsection{Adapter Implementation}
\label{app:adapter-impl}
Each adapter is realized as a forward hook on a target linear layer, so no model code is modified. The hook computes the low-rank correction $\mA_l\mB_l\vx$ and adds it to the layer's output, which already contains $\mW_l\vx+\vb_l$. A single flag enables or disables all hooks at once, which makes the paired ``adapters-on versus adapters-off'' evaluation trivial. The up-projection $\mA_l$ is zero-initialized and the down-projection $\mB_l$ is drawn from $\mathcal{N}(0,10^{-3})$, so the branch is an exact no-op at initialization; only $\mA_l,\mB_l$ receive gradients. Each trained set of adapter factors is stored as a single checkpoint per variant. We reiterate that $\mA_l\mB_l$ is a \emph{learned} side branch and not a factorization of $\mW_l$: no singular-value decomposition or low-rank approximation of the pretrained weights is ever computed. In the supervised ViT the adapted layers are the two MLP projections and the attention output projection of each block (\texttt{mlp.fc1}, \texttt{mlp.fc2}, \texttt{attn.proj}), excluding the fused \texttt{qkv} projection; in the CLIP image tower they are the two MLP projections of each residual attention block (\texttt{mlp.c\_fc}, \texttt{mlp.c\_proj}).

\subsection{Training Details}
\label{app:training}
Adapters are trained with Adam at learning rate $10^{-3}$ and batch size $64$ ($32$ on ImageNet). The source budget is $2{,}000$ images for CIFAR and $3{,}000$ for ImageNet (about three per class). The objective uses uniform trajectory weights $w_l{=}1$ over the aligned blocks, a decision weight $\lambda_{\mathrm{dec}}{=}1.0$, a control weight $\lambda_{\mathrm{ctrl}}{=}10^{-3}$, and a quietness weight $\lambda_{\mathrm{quiet}}{=}0.05$. Supervised runs use a fixed number of epochs (five for CIFAR, two for ImageNet); the CLIP runs train to $E_{\max}=9$ and select the checkpoint on the grid $\{1,3,5,7,9\}$ by label-free validation loss.

% ---------------------------------------------------------------------
\section{Probabilistic View and Generalization Account}
\label{app:theory}

\subsection{Weighted Gaussian Likelihood and MAP}
\label{app:gaussian-map}
Let $\mathcal T$ be the aligned blocks. Write $\mathbf h_l^0$ and $\mathbf o^0$ for the frozen hidden states and decisions, and $\mathbf h_l^\theta$ and $\mathbf o^\theta$ for their adapted counterparts. The nonnegative weights $w_l$, $\lambda_{\mathrm{dec}}$, $\lambda_{\mathrm{quiet}}$, and $\lambda_{\mathrm{ctrl}}$ are fixed. Define $\operatorname{col}$ to concatenate vector blocks and set
\begin{equation}
\begin{aligned}
t=\operatorname{col}\bigl(&
\{\sqrt{w_l}\mathbf h_l^0(\mathbf x^{\mathrm{tgt}})\}_{l\in\mathcal T},
\sqrt{\lambda_{\mathrm{dec}}}\mathbf o^0(\mathbf x^{\mathrm{tgt}}),\\
&\{\sqrt{\lambda_{\mathrm{quiet}}w_l}\mathbf h_l^0(\mathbf x^{\mathrm{anc}})\}_{l\in\mathcal T},
\sqrt{\lambda_{\mathrm{quiet}}\lambda_{\mathrm{dec}}}\mathbf o^0(\mathbf x^{\mathrm{anc}})\bigr),\\
f_\theta(v)=\operatorname{col}\bigl(&
\{\sqrt{w_l}\mathbf h_l^\theta(\mathbf x)\}_{l\in\mathcal T},
\sqrt{\lambda_{\mathrm{dec}}}\mathbf o^\theta(\mathbf x),\\
&\{\sqrt{\lambda_{\mathrm{quiet}}w_l}\mathbf h_l^\theta(\mathbf x^{\mathrm{anc}})\}_{l\in\mathcal T},
\sqrt{\lambda_{\mathrm{quiet}}\lambda_{\mathrm{dec}}}\mathbf o^\theta(\mathbf x^{\mathrm{anc}})\bigr).
\end{aligned}
\label{eq:app-weighted-stacking}
\end{equation}
Omit zero-weight blocks and let $d_t\ge1$ be the resulting dimension. Teacher and anchor readings are fixed observations during optimization. Let $\Omega(\theta)=\sum_{l\in\mathcal S}\sum_{q\in\mathcal L_l}\|\mathbf A_{l,q}\mathbf B_{l,q}\|_F^2$ be the actual control penalty in~\eqref{eq:control-loss}.

\begin{lemma}[Gaussian likelihood and regularized MAP]
\label{lem:gaussian-map}
The construction~\eqref{eq:app-weighted-stacking} satisfies $\|t-f_\theta(v)\|_2^2=\ell_{\mathrm{data},\theta}(z)$. Suppose $\Theta\subset\mathbb R^{d_\theta}$ is compact with positive Lebesgue volume and all outputs and $\Omega$ are continuous on $\Theta$, where $d_\theta$ is the number of parameter coordinates. For fixed sample size $N\ge1$, the prior density
\begin{equation}
\pi_N(\theta)=\frac{\mathbf1_\Theta(\theta)
e^{-N\lambda_{\mathrm{ctrl}}\Omega(\theta)}}{A_N},
\qquad
A_N=\int_\Theta e^{-N\lambda_{\mathrm{ctrl}}\Omega(u)}\,du
\label{eq:app-control-prior}
\end{equation}
is proper. Under the product observation likelihood for $D_N=(z_1,\ldots,z_N)$, its MAP solutions minimize
$\widehat{J}_N(\theta)=N^{-1}\sum_i \ell_{\mathrm{data},\theta}(z_i)+\lambda_{\mathrm{ctrl}}\Omega(\theta)$.
\end{lemma}

\begin{proof}
Squaring the concatenated residual and summing the block norms gives exactly the terms of $\mathcal L_{\mathrm{align}}+\lambda_{\mathrm{quiet}}\mathcal L_{\mathrm{quiet}}$. The chosen Gaussian density is
\[
p(t\mid v,\theta)=\pi^{-d_t/2}\exp[-\|t-f_\theta(v)\|^2].
\]
Continuity and compactness imply $0\le\Omega(\theta)\le M_\Omega<\infty$. Consequently
$0<\operatorname{vol}(\Theta)e^{-N\lambda_{\mathrm{ctrl}}M_\Omega}\le A_N\le\operatorname{vol}(\Theta)<\infty$.
Let $T_N=(t_1,\ldots,t_N)$ and $V_N=(v_1,\ldots,v_N)$. Bayes' rule gives, on $\Theta$,
\begin{equation}
\begin{aligned}
-\log p(\theta\mid T_N,V_N)
&=\sum_{i=1}^N[-\log p(t_i\mid v_i,\theta)]-\log\pi_N(\theta)
+\log p(T_N\mid V_N)\\
&=N\widehat J_N(\theta)+C_N,
\end{aligned}
\label{eq:app-map-proof}
\end{equation}
where $C_N=Nd_t\log\pi/2+\log A_N+\log p(T_N\mid V_N)$ is independent of $\theta$. The evidence is finite and positive because the likelihood is bounded and strictly positive and the prior is proper. Minimizing this expression proves the assertion; continuity on $\Theta$ ensures a minimizer exists. Taking $N=1$ proves~\eqref{eq:alignment-map}.
\end{proof}

The product likelihood is an analytical observation model; this algebra does not establish that the actual readings were independently sampled from its Gaussians. With a uniform prior, MAP minimizes data fit alone. With~\eqref{eq:app-control-prior}, the prior depends on $\theta$ and cannot be dropped from optimization. Its exponent uses the actual product penalty $\|\mathbf A\mathbf B\|_F^2$, not $\|\theta\|^2$. A Gaussian likelihood in $t$ also does not imply a Gaussian posterior in the nonlinear adapter parameters.

\subsection{Variational Change of Distribution}
\label{app:shift-variational}
For probability measures $Q\ll P$, define $\mathrm{KL}(Q\Vert P)=\int\log(dQ/dP)\,dQ$. The following lemma is stated for a general loss so that it applies both to alignment and to entropy.

\begin{lemma}[KL variational transfer bound]
\label{lem:variational-transfer}
Let $L\ge0$ be measurable, $\mathbb E_PL$ and $\mathbb E_QL$ finite, and $\kappa=\mathrm{KL}(Q\Vert P)<\infty$. Suppose $Z_\eta=\mathbb E_Pe^{\eta L}<\infty$ for at least one $\eta>0$. For each such $\eta$, define $dP_\eta=e^{\eta L}dP/Z_\eta$. Then
\begin{equation}
\begin{aligned}
\mathrm{KL}(Q\Vert P_\eta)&=\kappa-\eta\mathbb E_QL+\log Z_\eta,\\
\mathbb E_QL&\le\mathbb E_PL+
\frac{\kappa+\log\mathbb E_Pe^{\eta(L-\mathbb E_PL)}}{\eta}.
\end{aligned}
\label{eq:app-variational-transfer}
\end{equation}
The inequality becomes an equality precisely when $Q=P_\eta$ for this $\eta$.
\end{lemma}

\begin{proof}
Since $1\le Z_\eta<\infty$, $P_\eta$ is a probability measure equivalent to $P$. The Radon--Nikodym chain rule gives, $Q$-almost surely,
\[
\log\frac{dQ}{dP_\eta}
=\log\frac{dQ}{dP}-\eta L+\log Z_\eta.
\]
All terms are integrable under the assumptions, so integration proves the identity. Nonnegativity of KL implies $\eta\mathbb E_QL\le\kappa+\log Z_\eta$. Subtract $\eta\mathbb E_PL$ from the right-hand exponential-moment expression and divide by $\eta>0$ to obtain the bound. KL is zero exactly when the two measures agree, giving the equality condition.
\end{proof}

Apply the lemma with $L(r)=\|r\|^2$, $P=P_\theta$, and $Q=Q_\theta$ will derive the KL divergence formula $\kappa_\theta-\eta R_{\mathrm{te}}(\theta)+\log Z_\eta=\mathrm{KL}(Q_\theta\Vert P_{\theta,\eta})\ge0$. Taking the infimum over admissible $\eta$ proves~\eqref{eq:offline-test-bound}. These are distributions of the same residual function under two triplet laws, so both sides use the same teacher/anchor construction and fixed $\theta$. A test triplet distribution must specify its reference views even when those references are unavailable to the deployed algorithm.

The expression $\eta\mathbb E_QL-\mathrm{KL}(Q\Vert P)\le\log Z_\eta$ has the same variational structure as an ELBO. Here $Z_\eta$ normalizes an exponential loss moment, rather than a predictive likelihood. The transfer result follows from this identity, independently of the MAP argument. It requires distributional information through $\kappa$; fitting the offline objective alone does not determine that quantity.

\subsection{Gaussian Residual Bound}
\label{app:gaussian-transfer}
The actual offline residual law is an additional modeling assumption. The covariance $\sigma^2I$ below describes this law and need not equal the observation covariance $I/2$ used in Lemma~\ref{lem:gaussian-map}.

\begin{lemma}[Gaussian quadratic exponential moment]
\label{lem:gaussian-moment}
Let $r\sim\mathcal N(\mu,\Sigma)$ in $\mathbb R^{d_t}$, where $\Sigma$ is symmetric positive semidefinite. Define
\[
m=\|\mu\|^2+\operatorname{tr}\Sigma,
\qquad a=\operatorname{tr}(\Sigma^2)+2\mu^\top\Sigma\mu,
\qquad b=\|\Sigma\|_{\mathrm{op}}.
\]
Then $m=\mathbb E\|r\|^2$. If $b>0$ and $0<\eta<(2b)^{-1}$,
\begin{equation}
\log\mathbb E e^{\eta(\|r\|^2-m)}
\le\frac{a\eta^2}{1-2b\eta}.
\label{eq:app-gaussian-moment}
\end{equation}
In particular, if $\mu=0$ and $\Sigma=\sigma^2I$,
\begin{equation}
\log\mathbb E e^{\eta(\|r\|^2-d_t\sigma^2)}
=-\frac{d_t}{2}\log(1-2\sigma^2\eta)-d_t\sigma^2\eta
\le\frac{d_t\sigma^4\eta^2}{1-2\sigma^2\eta}.
\label{eq:app-isotropic-moment}
\end{equation}
\end{lemma}

\begin{proof}
Diagonalize $\Sigma=U\operatorname{diag}(\lambda_1,\ldots,\lambda_{d_t})U^\top$, where $U$ is orthogonal and $\lambda_j\ge0$. Writing $\nu=U^\top\mu$, the squared norm has the distribution $\sum_j(\nu_j+\sqrt{\lambda_j}g_j)^2$ for independent standard normal variables $g_j$. Its expectation is $m$. Completing the square in the one-dimensional Gaussian integral gives
\[
\mathbb E e^{\eta(\nu_j+\sqrt{\lambda_j}g_j)^2}
=(1-2\eta\lambda_j)^{-1/2}
\exp\!\left(\frac{\eta\nu_j^2}{1-2\eta\lambda_j}\right).
\]
Indeed, combining the exponent with $-g_j^2/2$ gives a quadratic with coefficient $-(1-2\eta\lambda_j)/2$; integrating its completed square gives the prefactor and displayed exponent. The formula also holds when $\lambda_j=0$. Independence and centering therefore give
\begin{equation}
\log\mathbb E e^{\eta(\|r\|^2-m)}
=\sum_j\left[
-\frac12\log(1-2\eta\lambda_j)-\eta\lambda_j
+\frac{2\eta^2\nu_j^2\lambda_j}{1-2\eta\lambda_j}\right].
\label{eq:app-gaussian-mgf-exact}
\end{equation}
For $0\le u<1$, the convergent power series satisfies
$-\log(1-u)-u=\sum_{j=2}^\infty u^j/j\le u^2/[2(1-u)]$.
Applying this with $u=2\eta\lambda_j$ bounds the first two terms by $\eta^2\lambda_j^2/(1-2\eta\lambda_j)$. Since $\lambda_j\le b$, summation bounds the whole expression by
$\eta^2\sum_j(\lambda_j^2+2\nu_j^2\lambda_j)/(1-2b\eta)=a\eta^2/(1-2b\eta)$.
Substitution of $\lambda_j=\sigma^2$ and $\nu_j=0$ proves~\eqref{eq:app-isotropic-moment}.
\end{proof}

\begin{corollary}[Gaussian transfer]
\label{cor:gaussian-transfer}
Suppose $P_\theta=\mathcal N(\mu,\Sigma)$ and $Q_\theta$ has finite second moment and finite $\kappa_\theta=\mathrm{KL}(Q_\theta\Vert P_\theta)$. With $a,b$ from Lemma~\ref{lem:gaussian-moment},
\begin{equation}
R_{\mathrm{te}}(\theta)\le R_{\mathrm{off}}(\theta)
+2\sqrt{a\kappa_\theta}+2b\kappa_\theta.
\label{eq:app-general-gaussian-transfer}
\end{equation}
The isotropic centered case is exactly~\eqref{eq:gaussian-transfer-bound}.
\end{corollary}

\begin{proof}
For $b>0$, one has $a>0$. Lemmas~\ref{lem:variational-transfer} and~\ref{lem:gaussian-moment} imply
\[
R_{\mathrm{te}}-R_{\mathrm{off}}
\le\frac{\kappa_\theta}{\eta}+\frac{a\eta}{1-2b\eta}.
\]
If $\kappa_\theta>0$, choose
$\eta=\sqrt{\kappa_\theta}/(\sqrt a+2b\sqrt{\kappa_\theta})$, which lies in $(0,(2b)^{-1})$.
Then $\kappa_\theta/\eta=\sqrt{a\kappa_\theta}+2b\kappa_\theta$ and $a\eta/(1-2b\eta)=\sqrt{a\kappa_\theta}$, proving the bound. If $\kappa_\theta=0$, then $Q_\theta=P_\theta$ and the risks agree; equivalently take $\eta\downarrow0$. If $b=0$, $P_\theta$ is a point mass, and finite KL again forces $Q_\theta=P_\theta$. Finally, for $\mu=0$, $\Sigma=\sigma^2I$, substitute $a=d_t\sigma^4$, $b=\sigma^2$, and $R_{\mathrm{off}}=d_t\sigma^2$.
\end{proof}

Writing $J_s(\theta)=R_s(\theta)+\lambda_{\mathrm{ctrl}}\Omega(\theta)$ for $s\in\{\mathrm{off},\mathrm{te}\}$ gives $J_{\mathrm{te}}-J_{\mathrm{off}}=R_{\mathrm{te}}-R_{\mathrm{off}}$. Thus the same transfer bound holds for the regularized objective. No estimate of the control penalty in terms of adapter rank is needed.

\section{Additional Results}
\label{app:additional}

The main text presents the primary CIFAR-100-C comparison in Table~\ref{tab:main_results_cifar100}. Here, Tables~\ref{tab:main_results_cifar10} and~\ref{tab:main_results_imagenet} provide the full per-corruption results on CIFAR-10-C and ImageNet-C, respectively, for supervised ViT and CLIP-based classification. These tables compare the source-free and source-dependent variants of Drift-TTA with the unadapted models and prior TTA baselines.

\begin{table}[!t]
    \centering
    \definecolor{bestgreen}{RGB}{0,120,55}
    \definecolor{secondblue}{RGB}{30,90,170}
    \newcommand{\best}[1]{\textcolor{bestgreen}{\textbf{#1}}}
    \newcommand{\second}[1]{\textcolor{secondblue}{\underline{#1}}}

    \caption{
    \textbf{Per-corruption test-time adaptation accuracy (\%) on CIFAR-10-C}
    for unimodal ViT and CLIP-based classification across the fifteen corruption
    types. The best result in each column is shown in
    \textcolor{bestgreen}{\textbf{green}} and the second best in
    \textcolor{secondblue}{\underline{blue}}.
    }
    \label{tab:main_results_cifar10}

    % \vspace{-0.4em}

    \begingroup
    \fontsize{5.8}{6.2}\selectfont
    \renewcommand{\arraystretch}{0.80}
    \setlength{\tabcolsep}{1.0pt}

    \newcommand{\rot}[1]{\rotatebox{65}{#1}}
    \newcommand{\NA}{--}
    \newcommand{\emptyvals}{%
        \NA & \NA & \NA & \NA & \NA &
        \NA & \NA & \NA & \NA & \NA &
        \NA & \NA & \NA & \NA & \NA & \NA%
    }

    \begin{tabular*}{\textwidth}{
        @{\extracolsep{\fill}}
        ccl*{16}{c}
        @{}
    }

    \toprule
    &
    & \textbf{Method}
    & \rot{Gaussian} & \rot{Shot} & \rot{Impulse} & \rot{Defocus} & \rot{Glass}
    & \rot{Motion} & \rot{Zoom} & \rot{Snow} & \rot{Frost} & \rot{Fog}
    & \rot{Brightness} & \rot{Contrast} & \rot{Elastic} & \rot{Pixelate} & \rot{JPEG}
    & \textbf{Mean} \\
    \midrule

    \multirow{19}{*}{\rotatebox{90}{\textbf{CIFAR-10-C}}}

    % ------------------------- ViT -------------------------
    & \multirow{9}{*}{\rotatebox{90}{\textbf{ViT}}}
    & ViT-Tiny
    & 43.8 & 49.8 & 69.3 & 82.1 & 71.0 & 82.3 & 84.6 & 91.8 & 89.5 & 87.5
    & \best{95.5} & 83.4 & 83.2 & 51.8 & 81.4 & 76.5 \\
    & & TENT
    & 43.8 & 50.1 & 69.6 & 81.6 & 71.2 & 82.6 & 84.6 & 91.7 & 89.7 & 88.0
    & 95.2 & 80.5 & 83.5 & 49.0 & 81.6 & 76.2 \\
    & & SAR
    & 45.4 & 54.7 & 70.0 & 81.9 & 69.8 & 81.7 & 83.9 & 89.0 & 87.5 & 83.6
    & 93.3 & 83.2 & 81.7 & 51.5 & 79.5 & 75.8 \\
    & & EATA
    & 43.8 & 49.8 & 69.3 & 82.1 & 71.0 & 82.3 & 84.6 & 91.8 & 89.5 & 87.5
    & \best{95.5} & 83.4 & 83.2 & 51.7 & 81.4 & 76.5 \\
    & & FOA
    & 62.7 & 71.6 & \second{77.5} & 88.2 & 76.4 & 87.7 & 91.1 & 91.6 & 90.9 & \second{89.1}
    & 95.2 & 85.3 & 85.2 & 76.5 & 81.3 & 83.3 \\
    & & CMF
    & 66.2 & 71.6 & 71.1 & 87.8 & 75.6 & 87.5 & 90.5 & 89.6 & 90.5 & 88.1
    & 94.8 & 86.3 & 83.3 & \best{83.3} & 83.2 & 83.3 \\
    & & PEA
    & 61.9 & 66.0 & \best{77.8} & 88.5 & 77.4 & 88.3 & 90.7 & 91.9 & 91.5 & \best{89.9}
    & \best{95.5} & 86.9 & 85.6 & \second{81.7} & 82.0 & 83.7 \\
    & & Drift-TTA (SF)
    & \second{82.0} & \second{80.8} & 73.7 & \second{92.8} & \best{85.9} & \best{90.3}
    & \second{93.9} & \second{92.7} & \second{92.9} & 88.8
    & \second{95.4} & \best{90.1} & \second{88.0} & 63.9 & \second{83.6} & \second{86.3} \\
    & & Drift-TTA (SD)
    & \best{83.8} & \best{85.2} & 76.2 & \best{93.4} & \second{82.7} & \second{90.0}
    & \best{94.3} & \best{92.7} & \best{93.0} & \second{89.1}
    & 95.3 & \second{87.6} & \best{88.6} & 72.2 & \best{84.1} & \best{87.2} \\

    \cmidrule(lr){2-19}

    % ------------------------- CLIP -------------------------
    & \multirow{10}{*}{\rotatebox{90}{\textbf{CLIP}}}
    & ViT-B/16
    & 38.0 & 41.7 & 54.4 & 71.7 & 40.9 & 67.9 & 73.6 & 73.9 & 77.4 & 70.3 & 84.5 & 62.3 & 53.8 & 47.6 & 59.4 & 61.2 \\
    & & TPT
    & 39.9 & 44.9 & 58.8 & 72.3 & 43.5 & 70.0 & 74.8 & 75.5 & 78.5 & 72.7 & 85.1 & 70.8 & 57.2 & 52.3 & 61.3 & 63.8 \\
    & & TDA
    & 42.9 & 46.0 & 63.3 & 72.3 & 46.9 & 69.8 & 74.7 & 77.5 & 79.4 & 71.6 & 85.9 & 62.3 & 60.5 & 63.7 & 62.9 & 65.3 \\
    & & BCA
    & 31.5 & 34.5 & 54.2 & 66.5 & 35.6 & 64.9 & 70.6 & 73.9 & 76.3 & 68.5 & 84.0 & 55.4 & 51.5 & 55.8 & 56.0 & 58.6 \\
    & & DMN-ZS
    & 45.0 & 49.5 & 59.8 & 73.6 & 48.8 & 70.9 & 74.8 & 77.1 & 77.5 & 70.1 & 86.2 & 67.6 & 62.8 & 55.3 & 62.1 & 65.4 \\
    & & DPE
    & 34.8 & 38.7 & 57.8 & 73.1 & 44.3 & 68.6 & 75.9 & 74.3 & 78.9 & 71.2 & 86.9 & 64.4 & 57.5 & 52.4 & 58.5 & 62.5 \\
    & & BAT
    & 60.5 & 65.5 & 63.4 & 80.1 & 52.3 & 80.4 & 82.0 & 83.1 & 83.6 & 81.3 & 89.6 & 82.4 & 67.5 & 60.5 & 68.1 & 73.3 \\
    & & BITTA
    & 62.6 & 67.4 & 65.0 & 80.6 & 56.3 & 80.7 & 82.3 & 83.5 & 84.3 & 81.9 & 89.9 & 83.3 & 69.1 & \second{65.5} & 69.3 & 74.8 \\
    & & Drift-TTA (SF)
    & \second{69.2} & \second{69.3} & \second{67.8} & \best{81.9} & \second{71.8} & \second{80.9} & \second{84.1} & \second{83.7} & \second{86.4} & \best{82.4} & \second{90.0} & \best{84.4} & \second{74.2} & 65.1 & \second{70.2} & \second{77.4} \\
    & & Drift-TTA (SD)
    & \best{73.8} & \best{78.3} & \best{75.4} & \second{81.7} & \best{74.6} & \best{81.1} & \best{84.7} & \best{84.4} & \best{87.0} & \second{82.3} & \best{90.5} & \second{84.4} & \best{75.4} & \best{68.6} & \best{71.1} & \best{79.5} \\

    \bottomrule
    \end{tabular*}
    \endgroup
\end{table}

\begin{table}[!t]
    \centering
    \definecolor{bestgreen}{RGB}{0,120,55}
    \definecolor{secondblue}{RGB}{30,90,170}
    \newcommand{\best}[1]{\textcolor{bestgreen}{\textbf{#1}}}
    \newcommand{\second}[1]{\textcolor{secondblue}{\underline{#1}}}

    \caption{
    \textbf{Per-corruption test-time adaptation accuracy (\%) on ImageNet-C}
    for unimodal ViT and CLIP-based classification across the fifteen corruption
    types. The best result in each column is shown in
    \textcolor{bestgreen}{\textbf{green}} and the second best in
    \textcolor{secondblue}{\underline{blue}}.
    }
    \label{tab:main_results_imagenet}

    % \vspace{-0.4em}

    \begingroup
    \fontsize{5.8}{6.2}\selectfont
    \renewcommand{\arraystretch}{0.80}
    \setlength{\tabcolsep}{1.0pt}

    \newcommand{\rot}[1]{\rotatebox{65}{#1}}
    \newcommand{\NA}{--}
    \newcommand{\emptyvals}{%
        \NA & \NA & \NA & \NA & \NA &
        \NA & \NA & \NA & \NA & \NA &
        \NA & \NA & \NA & \NA & \NA & \NA%
    }

    \begin{tabular*}{\textwidth}{
        @{\extracolsep{\fill}}
        ccl*{16}{c}
        @{}
    }

    \toprule
    &
    & \textbf{Method}
    & \rot{Gaussian} & \rot{Shot} & \rot{Impulse} & \rot{Defocus} & \rot{Glass}
    & \rot{Motion} & \rot{Zoom} & \rot{Snow} & \rot{Frost} & \rot{Fog}
    & \rot{Brightness} & \rot{Contrast} & \rot{Elastic} & \rot{Pixelate} & \rot{JPEG}
    & \textbf{Mean} \\
    \midrule

    \multirow{19}{*}{\rotatebox{90}{\textbf{ImageNet-C}}}

    % ------------------------- ViT -------------------------
    & \multirow{9}{*}{\rotatebox{90}{\textbf{ViT}}}
    & ViT-B/16
    & 56.7 & 56.8 & 57.5 & 46.9 & 35.6 & 53.1 & 44.8 & 62.2 & 62.5 & 65.7
    & 77.6 & 32.6 & 46.0 & 66.9 & 67.6 & 55.5 \\
    & & TENT
    & 57.1 & 58.1 & 59.2 & 44.7 & 43.2 & 56.6 & 50.6 & 62.8 & 60.5 & 65.2
    & 78.0 & 59.7 & 49.8 & 68.2 & 68.6 & 58.8 \\
    & & SAR
    & 59.9 & 62.2 & 62.8 & 54.1 & 54.1 & 59.1 & 54.5 & 63.5 & 65.6 & 65.3
    & 78.2 & 64.4 & 58.3 & 69.2 & 69.7 & 62.7 \\
    & & EATA
    & 57.3 & 59.1 & 59.9 & 53.6 & 49.4 & 58.2 & 51.8 & 63.0 & 62.9 & 65.7
    & 77.8 & 62.0 & 55.7 & 65.8 & 68.7 & 60.7 \\
    & & FOA
    & 61.5 & \best{63.5} & \best{64.3} & \second{56.9} & \second{55.1} & 61.0
    & \best{60.9} & 68.4 & \best{70.9} & \best{73.6}
    & \best{80.9} & \best{66.0} & 61.8 & \second{73.5} & \best{73.6} & \best{66.1} \\
    & & CMF
    & 60.0 & 61.2 & 60.9 & 56.6 & \best{56.8} & \second{62.4}
    & \second{60.8} & \best{69.3} & 67.9 & 72.7
    & 78.8 & \second{65.2} & \second{69.4} & \best{73.8} & \second{72.0} & \second{65.9} \\
    & & PEA
    & 57.7 & 58.4 & 58.9 & 53.2 & 50.4 & 59.8
    & 60.4 & \second{69.1} & \second{68.7} & \second{72.9}
    & \second{80.2} & 63.5 & \best{69.9} & 72.0 & \second{72.0} & 64.5 \\
    & & Drift-TTA (SF)
    & \best{62.0} & 61.4 & 61.9 & 56.5 & 40.5 & 58.7
    & 50.0 & 64.4 & 63.2 & 64.5
    & 77.1 & 55.7 & 49.4 & 68.7 & 69.0 & 60.2 \\
    & & Drift-TTA (SD)
    & \second{61.9} & \second{62.8} & \second{62.9} & \best{58.3} & 42.6 & \best{64.7}
    & 50.7 & 64.3 & 61.9 & 63.6
    & 77.4 & 53.0 & 47.9 & 70.0 & 68.1 & 60.7 \\

    \cmidrule(lr){2-19}

    % ------------------------- CLIP -------------------------
    & \multirow{10}{*}{\rotatebox{90}{\textbf{CLIP}}}
    & ViT-B/16
    & 13.0 & 14.4 & 13.7 & 24.1 & 15.8 & 25.0 & 26.6 & 32.6 & 30.7 & 37.0 & 55.0 & 17.6 & 13.6 & 33.0 & 33.5 & 25.7 \\
    & & TPT
    & 8.8 & 9.4 & 9.6 & 24.5 & 16.2 & 25.1 & 24.0 & 33.8 & 32.3 & 37.8 & 55.6 & 19.1 & 14.3 & 35.9 & 34.8 & 25.4 \\
    & & TDA
    & 9.6 & 11.5 & 11.2 & 24.8 & 14.6 & 23.9 & 25.0 & 35.1 & 32.8 & 38.7 & \second{56.5} & 16.3 & 14.2 & 38.8 & 33.8 & 25.8 \\
    & & BCA
    & 6.7 & 7.7 & 7.2 & 26.0 & 15.0 & 24.9 & 25.9 & 33.9 & \second{32.9} & 37.9 & 56.1 & 17.3 & 14.3 & 39.6 & 35.8 & 25.4 \\
    & & DMN-ZS
    & 11.8 & 12.2 & 12.2 & 23.2 & 15.0 & 24.0 & 23.1 & 31.5 & 29.7 & 36.5 & 53.1 & 16.9 & 12.1 & 31.6 & 33.1 & 24.4 \\
    & & DPE
    & 10.3 & 13.5 & 17.3 & 26.4 & 18.0 & 27.5 & 24.9 & 35.2 & \best{33.2} & 39.6 & \best{56.5} & 20.2 & 16.0 & 36.8 & 35.4 & 27.4 \\
    & & BAT
    & 20.2 & 23.5 & 19.9 & \second{27.2} & 21.0 & 31.2 & 28.7 & 36.3 & 31.5 & 41.0 & 56.4 & \second{26.2} & \best{23.9} & \second{39.9} & \second{39.2} & 31.1 \\
    & & BITTA
    & 22.5 & 24.1 & 22.4 & \best{28.0} & \best{23.9} & 31.9 & \best{29.3} & \second{36.9} & 32.2 & \best{41.9} & 55.9 & \best{26.4} & \second{22.7} & \best{40.4} & \best{39.7} & \second{31.9} \\
    & & Drift-TTA (SF)
    & \best{27.2} & \best{28.4} & \best{26.7} & 26.6 & \second{23.2} & \second{32.9} & \second{29.0} & 36.4 & 32.5 & 39.7 & 54.4 & 25.8 & 22.5 & 38.5 & 35.2 & \best{31.9} \\
    & & Drift-TTA (SD)
    & \second{26.1} & \second{27.2} & \second{26.4} & 26.4 & 20.2 & \best{33.1} & 28.0 & \best{37.6} & 32.9 & \second{41.1} & 53.8 & 25.8 & 20.5 & 39.0 & 34.3 & 31.5 \\

    \bottomrule
    \end{tabular*}
    \endgroup
\end{table}

\section{Single-Model Generalization to Unseen Corruption Families}
\label{app:unseen}

The per-corruption tables above follow the leave-one-family-out (LOFO) protocol, in which a separate Drift-TTA adapter is trained for each held-out corruption; a full fifteen-family table therefore aggregates fifteen independently trained adapters. Prior test-time adaptation methods instead deploy a \emph{single} model that must cope with any corruption, so this asymmetry could be raised against a direct comparison. To remove it, we train one Drift-TTA adapter on the fifteen standard corruption families and evaluate it, with no further training, on the four validation corruptions distributed with the -C benchmarks that are disjoint from the training families: \emph{speckle noise}, \emph{Gaussian blur}, \emph{spatter}, and \emph{saturate}. This is a strict, single-model, apples-to-apples setting: exactly as for the baselines, one model adapts to all four unseen corruptions, in both its offline form (the frozen adapter) and its online form (the adapter with entropy-based gate calibration), under source-dependent (SD) and source-free (SF) training. Tables~\ref{tab:unseen_imagenet}--\ref{tab:unseen_cifar100} report accuracy at severity~5 on ImageNet-C, CIFAR-10-C, and CIFAR-100-C, against the strongest single-model baseline for each backbone (PEA for ViT, BITTA for CLIP). Across all three benchmarks and both backbones, Drift-TTA improves over the unadapted model and over the baseline, showing that the trajectory-repair prior transfers to corruption families never seen during training.

\begin{table}[!t]
    \centering
    \definecolor{bestgreen}{RGB}{0,120,55}
    \definecolor{secondblue}{RGB}{30,90,170}
    \newcommand{\best}[1]{\textcolor{bestgreen}{\textbf{#1}}}
    \newcommand{\second}[1]{\textcolor{secondblue}{\underline{#1}}}

    \caption{
    \textbf{Test-time adaptation accuracy (\%) on ImageNet-C for the four held-out (unseen) corruption families at severity 5.}
    A single Drift-TTA adapter is trained on the fifteen standard corruption families and evaluated, without further training, on four disjoint validation corruptions. ``Drift-TTA'' is the frozen adapter (offline); ``Drift-TTA-online'' adds online entropy-based gate calibration; SD and SF denote source-dependent and source-free training. The best result in each column (within each backbone) is shown in \textcolor{bestgreen}{\textbf{green}} and the second best in \textcolor{secondblue}{\underline{blue}}.}
    \label{tab:unseen_imagenet}

    % \vspace{-0.2em}
    \begingroup
    \scriptsize
    \setlength{\tabcolsep}{4pt}

    \begin{tabular}{@{}ccl ccccc@{}}
    \toprule
    & & \textbf{Method}
    & \textbf{Speckle} & \textbf{Gaussian Blur} & \textbf{Spatter} & \textbf{Saturate} & \textbf{Mean} \\
    \midrule

    \multirow{12}{*}{\rotatebox{90}{\textbf{ImageNet-C}}}
    % ------------------------- ViT -------------------------
    & \multirow{6}{*}{\rotatebox{90}{\textbf{ViT}}}
    & ViT-B/16               & 64.3 & 44.9 & 69.1 & 73.4 & 62.9 \\
    & & PEA                  & 68.4 & 50.4 & \best{73.9} & \best{77.9} & \second{67.6} \\
    & & Drift-TTA (SF)       & 69.4 & 53.0 & 69.7 & 74.3 & 66.6 \\
    & & Drift-TTA (SD)       & \second{71.1} & \second{54.9} & 69.9 & 73.8 & 67.4 \\
    & & Drift-TTA-online (SF)& 69.8 & 54.5 & 69.9 & \second{74.5} & 67.2 \\
    & & Drift-TTA-online (SD)& \best{71.9} & \best{56.9} & \second{70.5} & 74.2 & \best{68.4} \\
    \cmidrule(lr){2-8}
    % ------------------------- CLIP -------------------------
    & \multirow{6}{*}{\rotatebox{90}{\textbf{CLIP}}}
    & ViT-B/16               & 25.8 & 21.3 & 38.3 & 49.5 & 33.7 \\
    & & BITTA                & 36.4 & 25.8 & 38.6 & \second{51.1} & 38.0 \\
    & & Drift-TTA (SF)       & 37.5 & 25.4 & 41.2 & \best{51.2} & 38.8 \\
    & & Drift-TTA (SD)       & \second{38.8} & \second{27.4} & \best{45.7} & 49.7 & \best{40.4} \\
    & & Drift-TTA-online (SF)& \best{39.1} & \best{27.9} & 41.5 & 50.7 & 39.8 \\
    & & Drift-TTA-online (SD)& 38.1 & 27.0 & \second{45.0} & 50.1 & \second{40.1} \\
    \bottomrule
    \end{tabular}
    \endgroup
\end{table}

\begin{table}[!t]
    \centering
    \definecolor{bestgreen}{RGB}{0,120,55}
    \definecolor{secondblue}{RGB}{30,90,170}
    \newcommand{\best}[1]{\textcolor{bestgreen}{\textbf{#1}}}
    \newcommand{\second}[1]{\textcolor{secondblue}{\underline{#1}}}

    \caption{
    \textbf{Test-time adaptation accuracy (\%) on CIFAR-10-C for the four held-out (unseen) corruption families at severity 5.}
    A single Drift-TTA adapter is trained on the fifteen standard corruption families and evaluated, without further training, on four disjoint validation corruptions. ``Drift-TTA'' is the frozen adapter (offline); ``Drift-TTA-online'' adds online entropy-based gate calibration; SD and SF denote source-dependent and source-free training. The best result in each column (within each backbone) is shown in \textcolor{bestgreen}{\textbf{green}} and the second best in \textcolor{secondblue}{\underline{blue}}.}
    \label{tab:unseen_cifar10}

    % \vspace{-0.2em}
    \begingroup
    \scriptsize
    \setlength{\tabcolsep}{4pt}

    \begin{tabular}{@{}ccl ccccc@{}}
    \toprule
    & & \textbf{Method}
    & \textbf{Speckle} & \textbf{Gaussian Blur} & \textbf{Spatter} & \textbf{Saturate} & \textbf{Mean} \\
    \midrule

    \multirow{12}{*}{\rotatebox{90}{\textbf{CIFAR-10-C}}}
    % ------------------------- ViT -------------------------
    & \multirow{6}{*}{\rotatebox{90}{\textbf{ViT}}}
    & ViT-Tiny               & 54.8 & 77.5 & 94.7 & 94.0 & 80.3 \\
    & & PEA                  & 70.3 & 87.0 & \second{95.1} & \best{94.7} & 86.8 \\
    & & Drift-TTA (SF)       & 76.6 & 92.8 & 95.0 & \second{94.6} & 89.8 \\
    & & Drift-TTA (SD)       & \best{88.4} & \best{94.1} & \best{95.3} & 94.1 & \best{93.0} \\
    & & Drift-TTA-online (SF)& 81.6 & 93.5 & 94.6 & 94.2 & 91.0 \\
    & & Drift-TTA-online (SD)& \second{88.3} & \second{93.8} & 95.1 & 94.1 & \second{92.8} \\
    \cmidrule(lr){2-8}
    % ------------------------- CLIP -------------------------
    & \multirow{6}{*}{\rotatebox{90}{\textbf{CLIP}}}
    & ViT-B/16               & 45.3 & 70.4 & 82.8 & 83.4 & 70.5 \\
    & & BITTA                & 68.1 & 81.7 & 87.7 & 89.2 & 81.7 \\
    & & Drift-TTA (SF)       & 66.2 & 80.6 & 86.2 & 87.8 & 80.2 \\
    & & Drift-TTA (SD)       & \second{79.7} & 81.5 & 87.1 & 87.7 & \second{84.0} \\
    & & Drift-TTA-online (SF)& 74.4 & \second{83.5} & \second{88.2} & \second{89.6} & 83.9 \\
    & & Drift-TTA-online (SD)& \best{80.4} & \best{84.1} & \best{89.4} & \best{90.2} & \best{86.0} \\
    \bottomrule
    \end{tabular}
    \endgroup
\end{table}

\begin{table}[!t]
    \centering
    \definecolor{bestgreen}{RGB}{0,120,55}
    \definecolor{secondblue}{RGB}{30,90,170}
    \newcommand{\best}[1]{\textcolor{bestgreen}{\textbf{#1}}}
    \newcommand{\second}[1]{\textcolor{secondblue}{\underline{#1}}}

    \caption{
    \textbf{Test-time adaptation accuracy (\%) on CIFAR-100-C for the four held-out (unseen) corruption families at severity 5.}
    A single Drift-TTA adapter is trained on the fifteen standard corruption families and evaluated, without further training, on four disjoint validation corruptions. ``Drift-TTA'' is the frozen adapter (offline); ``Drift-TTA-online'' adds online entropy-based gate calibration; SD and SF denote source-dependent and source-free training. The best result in each column (within each backbone) is shown in \textcolor{bestgreen}{\textbf{green}} and the second best in \textcolor{secondblue}{\underline{blue}}.}
    \label{tab:unseen_cifar100}

    % \vspace{-0.2em}
    \begingroup
    \scriptsize
    \setlength{\tabcolsep}{4pt}

    \begin{tabular}{@{}ccl ccccc@{}}
    \toprule
    & & \textbf{Method}
    & \textbf{Speckle} & \textbf{Gaussian Blur} & \textbf{Spatter} & \textbf{Saturate} & \textbf{Mean} \\
    \midrule

    \multirow{12}{*}{\rotatebox{90}{\textbf{CIFAR-100-C}}}
    % ------------------------- ViT -------------------------
    & \multirow{6}{*}{\rotatebox{90}{\textbf{ViT}}}
    & ViT-B/16               & 38.8 & 70.6 & 85.3 & 80.7 & 68.9 \\
    & & PEA                  & 64.3 & 80.9 & \best{87.3} & \best{85.5} & 79.5 \\
    & & Drift-TTA (SF)       & 71.1 & 83.9 & \second{87.1} & 84.2 & 81.6 \\
    & & Drift-TTA (SD)       & \best{80.7} & \second{84.4} & 86.4 & 81.8 & \second{83.4} \\
    & & Drift-TTA-online (SF)& 74.6 & 84.1 & 87.1 & \second{84.5} & 82.6 \\
    & & Drift-TTA-online (SD)& \second{80.3} & \best{84.6} & 86.9 & 81.9 & \best{83.4} \\
    \cmidrule(lr){2-8}
    % ------------------------- CLIP -------------------------
    & \multirow{6}{*}{\rotatebox{90}{\textbf{CLIP}}}
    & ViT-B/16               & 21.9 & 41.2 & 58.2 & 54.4 & 43.9 \\
    & & BITTA                & 30.2 & 50.9 & 60.1 & 59.3 & 50.1 \\
    & & Drift-TTA (SF)       & 35.6 & 52.1 & 62.6 & 59.4 & 52.4 \\
    & & Drift-TTA (SD)       & \best{50.4} & \second{55.3} & \second{63.2} & 59.9 & \second{57.2} \\
    & & Drift-TTA-online (SF)& 39.9 & 54.8 & 62.9 & \best{60.4} & 54.5 \\
    & & Drift-TTA-online (SD)& \second{50.0} & \best{57.0} & \best{64.1} & \second{60.1} & \best{57.8} \\
    \bottomrule
    \end{tabular}
    \endgroup
\end{table}

\section{Calibration (Expected Calibration Error)}
\label{app:ece}

Beyond accuracy, we report expected calibration error (ECE, \%, $\downarrow$) at severity~5 for
all three benchmarks and both backbones. For each setting we compare the unadapted model
(No-TTA) and the strongest prior test-time-adaptation baseline (PEA for the unimodal ViT, BITTA
for CLIP) against \emph{Drift-TTA} (the frozen repair adapter) and \emph{Drift-TTA-online} (the
same adapter with its per-branch gates tuned online by entropy minimization), under both
source-dependent (SD) and source-free (SF) training. ECE uses 15 equal-width bins over the
top-label confidence, on the same corrupted evaluation sets and configurations used for the
accuracy tables. The best SD and SF variants reduce mean ECE relative to PEA and BITTA across
all model--dataset settings. Table~\ref{tab:ece_calibration_mean_summary} summarizes these
results in the main text. Tables~\ref{tab:ece_cifar10_expanded}--\ref{tab:ece_imagenet_expanded}
give the full per-corruption breakdown.

\paragraph{ECE computation.} We use the standard binned estimator of \citet{guo2017calibration}.
For each test image the top-label confidence is $\hat{p}=\max_k \operatorname{softmax}(z)_k$ and the
prediction is correct when $\hat{y}=y$. Confidences are partitioned into $M{=}15$ equal-width bins
$\{B_m\}_{m=1}^{M}$ over $[0,1]$, and
$\mathrm{ECE}=\sum_{m=1}^{M}\frac{|B_m|}{N}\,\bigl|\operatorname{acc}(B_m)-\operatorname{conf}(B_m)\bigr|$,
reported as a percentage. For the unimodal ViT rows the logits $z$ are the classifier outputs; for
the CLIP rows $z$ are the temperature-scaled image--text cosine scores
($\mathrm{logit\_scale}\cdot\langle f_I,f_T\rangle$), i.e.\ the model's own predictive distribution.
Confidences are collected from each method's \emph{final} prediction (post-adaptation for the
adapted rows) on exactly the corrupted test sets used for the accuracy tables, all at severity~5, and binned with a single
shared implementation so every row uses an identical metric.

\paragraph{Baselines.} The PEA (ViT) and BITTA (CLIP) rows are obtained by running the authors'
official implementations under their reported settings and computing ECE with the same estimator;
no calibration-specific tuning is applied. PEA uses source-statistic alignment with a continual
stream at batch size~$64$; BITTA uses its bilateral information-aware objective with one gradient
step per batch at batch sizes $200/200/64$ (CIFAR-10-C/CIFAR-100-C/ImageNet-C). As a correctness
check, the recomputed top-1 accuracy of every baseline reproduces the corresponding entry in the
accuracy tables to the first decimal, confirming that the calibration numbers are measured on the
same models and predictions.

\begin{table}[!t]
    \centering
    \caption{\textbf{Expected calibration error (ECE, \%, $\downarrow$) on CIFAR-10-C at severity 5.}
    All 15 corruption families are reported for both backbones. PEA is the ViT prior-TTA baseline
    and BITTA is the CLIP prior-TTA baseline. Drift-TTA is shown with its adapter alone and with
    additional online gate tuning, under source-dependent (SD) and source-free (SF) training. In each
    $\Delta$ row, the lower ECE of the two variants is selected independently per column and subtracted
    from PEA/BITTA. Green $\downarrow$ values denote how many percentage points the selected ECE is
    below the baseline; red $\uparrow$ values denote how many points it is above. Method-cell tints follow the visualization palette:
    teal for no adaptation, orange for SD, purple for SF, and gray for PEA/BITTA.}
    \label{tab:ece_cifar10_expanded}
    \begingroup
    \fontsize{5.8}{6.2}\selectfont
    \renewcommand{\arraystretch}{0.80}
    \setlength{\tabcolsep}{1.0pt}
    \definecolor{ECEBase}{RGB}{223,240,238}
    \definecolor{ECEPrior}{RGB}{238,238,238}
    \definecolor{ECESD}{RGB}{255,241,224}
    \definecolor{ECESF}{RGB}{239,230,247}
    \definecolor{ECESDDark}{RGB}{210,105,0}
    \definecolor{ECESFDark}{RGB}{123,44,191}
    \definecolor{ECEGain}{RGB}{0,110,60}
    \definecolor{ECELoss}{RGB}{160,55,45}
    \newcommand{\rot}[1]{\rotatebox{65}{#1}}
    \newcommand{\gain}[1]{\textcolor{ECEGain}{\textbf{$\downarrow$#1}}}
    \newcommand{\loss}[1]{\textcolor{ECELoss}{$\uparrow$#1}}
    \begin{tabular*}{\textwidth}{@{\extracolsep{\fill}}cl*{16}{c}@{}}
    \toprule
    & \textbf{Method / deployment}
    & \rot{Gaussian} & \rot{Shot} & \rot{Impulse} & \rot{Defocus} & \rot{Glass}
    & \rot{Motion} & \rot{Zoom} & \rot{Snow} & \rot{Frost} & \rot{Fog}
    & \rot{Brightness} & \rot{Contrast} & \rot{Elastic} & \rot{Pixelate}
    & \rot{JPEG} & \textbf{Mean} \\
    \midrule
    \multirow{8}{*}{\rotatebox{90}{\textbf{ViT}}}
    & \cellcolor{ECEBase}No-TTA
    & 11.33 & 8.24 & 3.45 & 12.72 & 9.17 & 12.84 & 14.46 & 15.96 & 17.11 & 22.07 & 14.33 & 15.92 & 17.84 & 13.33 & 13.86 & 13.51 \\
    & \cellcolor{ECEPrior}PEA
    & 3.59 & 3.75 & 5.54 & 14.32 & 7.73 & 14.02 & 15.82 & 13.42 & 14.81 & 15.89 & 13.71 & 14.36 & 13.69 & 10.94 & 10.08 & 11.44 \\
    \cmidrule(lr){2-18}
    & \cellcolor{ECESD}Drift-TTA: adapter only (SD)
    & 10.11 & 9.97 & 4.92 & 13.85 & 9.55 & 12.50 & 14.64 & 12.68 & 14.71 & 16.82 & 12.35 & 13.61 & 13.69 & 6.17 & 10.30 & 11.72 \\
    & \cellcolor{ECESD}Drift-TTA: adapter $+$ online gate (SD)
    & 6.26 & 6.60 & 6.16 & 8.74 & 6.00 & 7.92 & 9.66 & 8.72 & 10.68 & 11.34 & 8.78 & 8.89 & 8.73 & 8.13 & 6.51 & 8.21 \\
    & \cellcolor{ECESD}\textcolor{ECESDDark}{\textbf{$\Delta$ best SD vs.\ PEA}}
    & \loss{2.67} & \loss{2.85} & \gain{0.62} & \gain{5.58} & \gain{1.73} & \gain{6.10} & \gain{6.16} & \gain{4.70} & \gain{4.13} & \gain{4.55} & \gain{4.93} & \gain{5.47} & \gain{4.96} & \gain{4.77} & \gain{3.57} & \gain{3.23} \\
    \cmidrule(lr){2-18}
    & \cellcolor{ECESF}Drift-TTA: adapter only (SF)
    & 12.02 & 9.88 & 8.88 & 16.73 & 16.31 & 14.94 & 16.49 & 14.85 & 18.06 & 19.98 & 13.53 & 17.49 & 16.73 & 7.79 & 13.74 & 14.50 \\
    & \cellcolor{ECESF}Drift-TTA: adapter $+$ online gate (SF)
    & 9.13 & 7.70 & 4.28 & 12.53 & 11.37 & 10.53 & 12.39 & 10.45 & 13.34 & 14.90 & 10.11 & 13.01 & 12.41 & 10.00 & 8.54 & 10.71 \\
    & \cellcolor{ECESF}\textcolor{ECESFDark}{\textbf{$\Delta$ best SF vs.\ PEA}}
    & \loss{5.54} & \loss{3.95} & \gain{1.26} & \gain{1.79} & \loss{3.64} & \gain{3.49} & \gain{3.43} & \gain{2.97} & \gain{1.47} & \gain{0.99} & \gain{3.60} & \gain{1.35} & \gain{1.28} & \gain{3.15} & \gain{1.54} & \gain{0.73} \\
    \midrule
    \multirow{8}{*}{\rotatebox{90}{\textbf{CLIP}}}
    & \cellcolor{ECEBase}No-TTA
    & 12.58 & 9.72 & 4.37 & 3.19 & 3.75 & 1.93 & 2.39 & 5.01 & 6.00 & 4.43 & 5.52 & 4.59 & 2.07 & 4.74 & 2.74 & 4.87 \\
    & \cellcolor{ECEPrior}BITTA
    & 23.00 & 18.35 & 21.61 & 10.32 & 25.52 & 8.99 & 9.41 & 7.70 & 7.27 & 8.81 & 4.33 & 6.91 & 16.62 & 17.47 & 17.42 & 13.58 \\
    \cmidrule(lr){2-18}
    & \cellcolor{ECESD}Drift-TTA: adapter only (SD)
    & 2.35 & 2.50 & 3.12 & 2.80 & 1.69 & 2.09 & 3.28 & 4.16 & 5.65 & 4.89 & 5.29 & 5.24 & 2.65 & 1.69 & 2.60 & 3.33 \\
    & \cellcolor{ECESD}Drift-TTA: adapter $+$ online gate (SD)
    & 10.12 & 5.33 & 8.25 & 3.08 & 3.61 & 2.78 & 0.90 & 1.02 & 1.31 & 0.74 & 1.88 & 1.58 & 2.34 & 6.39 & 3.61 & 3.53 \\
    & \cellcolor{ECESD}\textcolor{ECESDDark}{\textbf{$\Delta$ best SD vs.\ BITTA}}
    & \gain{20.65} & \gain{15.85} & \gain{18.49} & \gain{7.52} & \gain{23.83} & \gain{6.90} & \gain{8.51} & \gain{6.68} & \gain{5.96} & \gain{8.07} & \gain{2.45} & \gain{5.33} & \gain{14.28} & \gain{15.78} & \gain{14.82} & \gain{10.25} \\
    \cmidrule(lr){2-18}
    & \cellcolor{ECESF}Drift-TTA: adapter only (SF)
    & 3.08 & 3.25 & 1.87 & 3.19 & 4.19 & 2.23 & 4.30 & 5.28 & 5.57 & 5.83 & 5.13 & 5.35 & 1.90 & 3.81 & 2.85 & 3.86 \\
    & \cellcolor{ECESF}Drift-TTA: adapter $+$ online gate (SF)
    & 4.22 & 1.73 & 5.28 & 0.84 & 1.57 & 0.98 & 1.17 & 2.29 & 2.14 & 2.40 & 2.31 & 2.12 & 1.76 & 3.43 & 1.90 & 2.28 \\
    & \cellcolor{ECESF}\textcolor{ECESFDark}{\textbf{$\Delta$ best SF vs.\ BITTA}}
    & \gain{19.92} & \gain{16.62} & \gain{19.74} & \gain{9.48} & \gain{23.95} & \gain{8.01} & \gain{8.24} & \gain{5.41} & \gain{5.13} & \gain{6.41} & \gain{2.02} & \gain{4.79} & \gain{14.86} & \gain{14.04} & \gain{15.52} & \gain{11.30} \\
    \bottomrule
    \end{tabular*}
    \endgroup
\end{table}

\begin{table}[!t]
    \centering
    \caption{\textbf{Expected calibration error (ECE, \%, $\downarrow$) on CIFAR-100-C at severity 5.}
    All 15 corruption families are reported for both backbones. PEA is the ViT prior-TTA baseline
    and BITTA is the CLIP prior-TTA baseline. Drift-TTA is shown with its adapter alone and with
    additional online gate tuning, under source-dependent (SD) and source-free (SF) training. In each
    $\Delta$ row, the lower ECE of the two variants is selected independently per column and subtracted
    from PEA/BITTA. Green $\downarrow$ values denote how many percentage points the selected ECE is
    below the baseline; red $\uparrow$ values denote how many points it is above. Method-cell tints follow the visualization palette:
    teal for no adaptation, orange for SD, purple for SF, and gray for PEA/BITTA.}
    \label{tab:ece_cifar100_expanded}
    \begingroup
    \fontsize{5.8}{6.2}\selectfont
    \renewcommand{\arraystretch}{0.80}
    \setlength{\tabcolsep}{1.0pt}
    \definecolor{ECEBase}{RGB}{223,240,238}
    \definecolor{ECEPrior}{RGB}{238,238,238}
    \definecolor{ECESD}{RGB}{255,241,224}
    \definecolor{ECESF}{RGB}{239,230,247}
    \definecolor{ECESDDark}{RGB}{210,105,0}
    \definecolor{ECESFDark}{RGB}{123,44,191}
    \definecolor{ECEGain}{RGB}{0,110,60}
    \definecolor{ECELoss}{RGB}{160,55,45}
    \newcommand{\rot}[1]{\rotatebox{65}{#1}}
    \newcommand{\gain}[1]{\textcolor{ECEGain}{\textbf{$\downarrow$#1}}}
    \newcommand{\loss}[1]{\textcolor{ECELoss}{$\uparrow$#1}}
    \begin{tabular*}{\textwidth}{@{\extracolsep{\fill}}cl*{16}{c}@{}}
    \toprule
    & \textbf{Method / deployment}
    & \rot{Gaussian} & \rot{Shot} & \rot{Impulse} & \rot{Defocus} & \rot{Glass}
    & \rot{Motion} & \rot{Zoom} & \rot{Snow} & \rot{Frost} & \rot{Fog}
    & \rot{Brightness} & \rot{Contrast} & \rot{Elastic} & \rot{Pixelate}
    & \rot{JPEG} & \textbf{Mean} \\
    \midrule
    \multirow{8}{*}{\rotatebox{90}{\textbf{ViT}}}
    & \cellcolor{ECEBase}No-TTA
    & 16.84 & 18.22 & 9.46 & 7.62 & 19.09 & 6.03 & 7.31 & 5.53 & 7.23 & 5.37 & 6.30 & 10.77 & 4.40 & 16.24 & 11.11 & 10.10 \\
    & \cellcolor{ECEPrior}PEA
    & 17.69 & 17.49 & 22.72 & 14.10 & 18.13 & 16.91 & 13.93 & 16.51 & 16.74 & 23.06 & 13.98 & 20.14 & 22.56 & 19.43 & 16.64 & 18.00 \\
    \cmidrule(lr){2-18}
    & \cellcolor{ECESD}Drift-TTA: adapter only (SD)
    & 6.54 & 7.49 & 6.61 & 8.24 & 5.23 & 7.75 & 8.69 & 9.86 & 10.32 & 12.94 & 7.99 & 16.79 & 7.91 & 5.93 & 4.75 & 8.47 \\
    & \cellcolor{ECESD}Drift-TTA: adapter $+$ online gate (SD)
    & 9.08 & 8.41 & 8.10 & 7.84 & 9.38 & 7.26 & 7.81 & 7.19 & 7.90 & 6.37 & 7.09 & 5.28 & 7.50 & 10.41 & 11.25 & 8.06 \\
    & \cellcolor{ECESD}\textcolor{ECESDDark}{\textbf{$\Delta$ best SD vs.\ PEA}}
    & \gain{11.15} & \gain{10.00} & \gain{16.11} & \gain{6.26} & \gain{12.90} & \gain{9.65} & \gain{6.12} & \gain{9.32} & \gain{8.84} & \gain{16.69} & \gain{6.89} & \gain{14.86} & \gain{15.06} & \gain{13.50} & \gain{11.89} & \gain{9.94} \\
    \cmidrule(lr){2-18}
    & \cellcolor{ECESF}Drift-TTA: adapter only (SF)
    & 6.91 & 4.98 & 12.48 & 8.40 & 5.92 & 8.47 & 9.40 & 9.55 & 13.70 & 18.48 & 9.56 & 22.20 & 7.57 & 6.70 & 5.08 & 9.96 \\
    & \cellcolor{ECESF}Drift-TTA: adapter $+$ online gate (SF)
    & 7.22 & 6.46 & 24.90 & 7.48 & 7.28 & 6.48 & 7.15 & 7.22 & 8.56 & 12.23 & 7.08 & 14.90 & 6.99 & 9.99 & 9.04 & 9.53 \\
    & \cellcolor{ECESF}\textcolor{ECESFDark}{\textbf{$\Delta$ best SF vs.\ PEA}}
    & \gain{10.78} & \gain{12.51} & \gain{10.24} & \gain{6.62} & \gain{12.21} & \gain{10.43} & \gain{6.78} & \gain{9.29} & \gain{8.18} & \gain{10.83} & \gain{6.90} & \gain{5.24} & \gain{15.57} & \gain{12.73} & \gain{11.56} & \gain{8.47} \\
    \midrule
    \multirow{8}{*}{\rotatebox{90}{\textbf{CLIP}}}
    & \cellcolor{ECEBase}No-TTA
    & 3.51 & 3.53 & 4.24 & 7.49 & 1.00 & 9.04 & 8.50 & 11.97 & 9.83 & 5.68 & 11.34 & 6.74 & 4.23 & 2.66 & 4.79 & 6.30 \\
    & \cellcolor{ECEPrior}BITTA
    & 21.45 & 21.64 & 15.67 & 11.77 & 19.94 & 11.95 & 10.75 & 10.81 & 12.00 & 12.91 & 6.79 & 11.55 & 17.92 & 19.52 & 15.88 & 14.70 \\
    \cmidrule(lr){2-18}
    & \cellcolor{ECESD}Drift-TTA: adapter only (SD)
    & 10.52 & 13.07 & 7.89 & 12.06 & 9.46 & 12.41 & 12.56 & 13.20 & 12.55 & 9.87 & 11.80 & 9.36 & 10.43 & 2.55 & 7.27 & 10.33 \\
    & \cellcolor{ECESD}Drift-TTA: adapter $+$ online gate (SD)
    & 3.93 & 4.92 & 2.79 & 5.51 & 4.54 & 3.92 & 5.91 & 6.03 & 6.44 & 3.73 & 5.54 & 4.48 & 4.11 & 9.80 & 3.48 & 5.01 \\
    & \cellcolor{ECESD}\textcolor{ECESDDark}{\textbf{$\Delta$ best SD vs.\ BITTA}}
    & \gain{17.52} & \gain{16.72} & \gain{12.88} & \gain{6.26} & \gain{15.40} & \gain{8.03} & \gain{4.84} & \gain{4.78} & \gain{5.56} & \gain{9.18} & \gain{1.25} & \gain{7.07} & \gain{13.81} & \gain{16.97} & \gain{12.40} & \gain{9.69} \\
    \cmidrule(lr){2-18}
    & \cellcolor{ECESF}Drift-TTA: adapter only (SF)
    & 8.54 & 9.05 & 6.71 & 11.41 & 9.60 & 12.33 & 12.96 & 13.67 & 13.44 & 9.10 & 12.93 & 8.53 & 10.39 & 4.03 & 6.71 & 9.96 \\
    & \cellcolor{ECESF}Drift-TTA: adapter $+$ online gate (SF)
    & 4.54 & 5.60 & 3.65 & 6.36 & 7.91 & 7.11 & 6.93 & 8.03 & 7.46 & 5.17 & 7.40 & 4.95 & 6.06 & 2.50 & 2.89 & 5.77 \\
    & \cellcolor{ECESF}\textcolor{ECESFDark}{\textbf{$\Delta$ best SF vs.\ BITTA}}
    & \gain{16.91} & \gain{16.04} & \gain{12.02} & \gain{5.41} & \gain{12.03} & \gain{4.84} & \gain{3.82} & \gain{2.78} & \gain{4.54} & \gain{7.74} & \loss{0.61} & \gain{6.60} & \gain{11.86} & \gain{17.02} & \gain{12.99} & \gain{8.93} \\
    \bottomrule
    \end{tabular*}
    \endgroup
\end{table}

\begin{table}[!t]
    \centering
    \caption{\textbf{Expected calibration error (ECE, \%, $\downarrow$) on ImageNet-C at severity 5.}
    All 15 corruption families are reported for both backbones. PEA is the ViT prior-TTA baseline
    and BITTA is the CLIP prior-TTA baseline. Drift-TTA is shown with its adapter alone and with
    additional online gate tuning, under source-dependent (SD) and source-free (SF) training. In each
    $\Delta$ row, the lower ECE of the two variants is selected independently per column and subtracted
    from PEA/BITTA. Green $\downarrow$ values denote how many percentage points the selected ECE is
    below the baseline; red $\uparrow$ values denote how many points it is above. Method-cell tints follow the visualization palette:
    teal for no adaptation, orange for SD, purple for SF, and gray for PEA/BITTA.}
    \label{tab:ece_imagenet_expanded}
    \begingroup
    \fontsize{5.8}{6.2}\selectfont
    \renewcommand{\arraystretch}{0.80}
    \setlength{\tabcolsep}{1.0pt}
    \definecolor{ECEBase}{RGB}{223,240,238}
    \definecolor{ECEPrior}{RGB}{238,238,238}
    \definecolor{ECESD}{RGB}{255,241,224}
    \definecolor{ECESF}{RGB}{239,230,247}
    \definecolor{ECESDDark}{RGB}{210,105,0}
    \definecolor{ECESFDark}{RGB}{123,44,191}
    \definecolor{ECEGain}{RGB}{0,110,60}
    \definecolor{ECELoss}{RGB}{160,55,45}
    \newcommand{\rot}[1]{\rotatebox{65}{#1}}
    \newcommand{\gain}[1]{\textcolor{ECEGain}{\textbf{$\downarrow$#1}}}
    \newcommand{\loss}[1]{\textcolor{ECELoss}{$\uparrow$#1}}
    \begin{tabular*}{\textwidth}{@{\extracolsep{\fill}}cl*{16}{c}@{}}
    \toprule
    & \textbf{Method / deployment}
    & \rot{Gaussian} & \rot{Shot} & \rot{Impulse} & \rot{Defocus} & \rot{Glass}
    & \rot{Motion} & \rot{Zoom} & \rot{Snow} & \rot{Frost} & \rot{Fog}
    & \rot{Brightness} & \rot{Contrast} & \rot{Elastic} & \rot{Pixelate}
    & \rot{JPEG} & \textbf{Mean} \\
    \midrule
    \multirow{8}{*}{\rotatebox{90}{\textbf{ViT}}}
    & \cellcolor{ECEBase}No-TTA
    & 12.85 & 7.46 & 11.68 & 9.99 & 4.99 & 2.68 & 4.68 & 4.40 & 20.98 & 48.11 & 8.98 & 20.51 & 5.97 & 4.73 & 5.15 & 11.54 \\
    & \cellcolor{ECEPrior}PEA
    & 36.07 & 34.25 & 34.59 & 31.67 & 36.02 & 33.91 & 39.18 & 33.59 & 38.87 & 40.01 & 18.93 & 37.44 & 37.93 & 23.61 & 25.48 & 33.44 \\
    \cmidrule(lr){2-18}
    & \cellcolor{ECESD}Drift-TTA: adapter only (SD)
    & 21.28 & 20.22 & 21.28 & 18.71 & 9.40 & 19.38 & 11.17 & 15.51 & 18.80 & 33.80 & 14.82 & 25.63 & 7.65 & 16.92 & 12.75 & 17.82 \\
    & \cellcolor{ECESD}Drift-TTA: adapter $+$ online gate (SD)
    & 16.22 & 14.85 & 16.31 & 13.41 & 4.16 & 14.25 & 5.60 & 10.99 & 14.06 & 29.33 & 11.86 & 19.08 & 2.57 & 12.98 & 9.74 & 13.03 \\
    & \cellcolor{ECESD}\textcolor{ECESDDark}{\textbf{$\Delta$ best SD vs.\ PEA}}
    & \gain{19.85} & \gain{19.40} & \gain{18.28} & \gain{18.26} & \gain{31.86} & \gain{19.66} & \gain{33.58} & \gain{22.60} & \gain{24.81} & \gain{10.68} & \gain{7.07} & \gain{18.36} & \gain{35.36} & \gain{10.63} & \gain{15.74} & \gain{20.41} \\
    \cmidrule(lr){2-18}
    & \cellcolor{ECESF}Drift-TTA: adapter only (SF)
    & 19.31 & 15.02 & 16.28 & 15.23 & 6.41 & 13.72 & 8.28 & 15.60 & 24.63 & 43.98 & 15.29 & 34.32 & 6.41 & 13.53 & 13.40 & 17.43 \\
    & \cellcolor{ECESF}Drift-TTA: adapter $+$ online gate (SF)
    & 16.17 & 12.20 & 13.28 & 12.84 & 3.66 & 11.31 & 5.50 & 12.90 & 21.68 & 40.40 & 12.97 & 32.20 & 3.54 & 11.15 & 10.99 & 14.72 \\
    & \cellcolor{ECESF}\textcolor{ECESFDark}{\textbf{$\Delta$ best SF vs.\ PEA}}
    & \gain{19.90} & \gain{22.05} & \gain{21.31} & \gain{18.83} & \gain{32.36} & \gain{22.60} & \gain{33.68} & \gain{20.69} & \gain{17.19} & \loss{0.39} & \gain{5.96} & \gain{5.24} & \gain{34.39} & \gain{12.46} & \gain{14.49} & \gain{18.72} \\
    \midrule
    \multirow{8}{*}{\rotatebox{90}{\textbf{CLIP}}}
    & \cellcolor{ECEBase}No-TTA
    & 7.77 & 7.30 & 7.83 & 0.96 & 4.71 & 1.32 & 2.21 & 1.76 & 4.68 & 1.74 & 3.15 & 2.77 & 8.81 & 2.56 & 4.22 & 4.12 \\
    & \cellcolor{ECEPrior}BITTA
    & 18.78 & 20.19 & 20.47 & 19.62 & 21.08 & 16.61 & 19.12 & 16.07 & 20.13 & 15.59 & 10.25 & 16.13 & 23.39 & 15.56 & 17.56 & 18.04 \\
    \cmidrule(lr){2-18}
    & \cellcolor{ECESD}Drift-TTA: adapter only (SD)
    & 5.59 & 4.62 & 5.03 & 2.90 & 9.43 & 2.04 & 6.23 & 1.89 & 5.87 & 2.08 & 2.32 & 1.76 & 10.95 & 2.15 & 4.87 & 4.51 \\
    & \cellcolor{ECESD}Drift-TTA: adapter $+$ online gate (SD)
    & 15.64 & 14.28 & 14.08 & 12.44 & 18.78 & 11.51 & 14.68 & 9.58 & 14.38 & 8.89 & 6.44 & 9.74 & 19.86 & 11.44 & 11.56 & 12.89 \\
    & \cellcolor{ECESD}\textcolor{ECESDDark}{\textbf{$\Delta$ best SD vs.\ BITTA}}
    & \gain{13.19} & \gain{15.57} & \gain{15.44} & \gain{16.72} & \gain{11.65} & \gain{14.57} & \gain{12.89} & \gain{14.18} & \gain{14.26} & \gain{13.51} & \gain{7.93} & \gain{14.37} & \gain{12.44} & \gain{13.41} & \gain{12.69} & \gain{13.53} \\
    \cmidrule(lr){2-18}
    & \cellcolor{ECESF}Drift-TTA: adapter only (SF)
    & 2.05 & 1.90 & 3.43 & 1.31 & 4.53 & 3.13 & 1.91 & 1.82 & 5.27 & 1.64 & 2.37 & 3.26 & 8.52 & 3.47 & 2.55 & 3.14 \\
    & \cellcolor{ECESF}Drift-TTA: adapter $+$ online gate (SF)
    & 6.93 & 7.49 & 8.65 & 6.69 & 12.00 & 5.72 & 9.64 & 6.62 & 12.79 & 7.10 & 5.32 & 4.01 & 14.44 & 4.68 & 8.13 & 8.01 \\
    & \cellcolor{ECESF}\textcolor{ECESFDark}{\textbf{$\Delta$ best SF vs.\ BITTA}}
    & \gain{16.73} & \gain{18.29} & \gain{17.04} & \gain{18.31} & \gain{16.55} & \gain{13.48} & \gain{17.21} & \gain{14.25} & \gain{14.86} & \gain{13.95} & \gain{7.88} & \gain{12.87} & \gain{14.87} & \gain{12.09} & \gain{15.01} & \gain{14.90} \\
    \bottomrule
    \end{tabular*}
    \endgroup
\end{table}

\clearpage
\section{Hidden-Trajectory Drift under Corruption}
\label{app:trajectory-drift}

This section explains how each panel of Figure~\ref{fig:trajectory-drift} is constructed and how it should be interpreted. We first define the trajectory quantities shown in the figure, then examine whether the association between trajectory drift and prediction error remains after controlling for corruption severity. Figure~\ref{fig:trajectory-drift-imagenet} provides the corresponding analysis on ImageNet-C.

\paragraph{Experimental setup.}
We evaluate a frozen ViT-B/16 without test-time adaptation. CIFAR-100-C contains $N{=}10{,}000$ images per corruption condition, while ImageNet-C contains $N{=}50{,}000$. Each corrupted image is paired with its clean counterpart, and both are processed by the same frozen model. Consequently, trajectory drift and prediction error are measured on the same images. Panels~\textbf{(a)} and~\textbf{(e)} use one representative corruption family---Gaussian noise for CIFAR-100-C and fog for ImageNet-C---while panels~\textbf{(b)}--\textbf{(d)} summarize all fifteen corruption families.

\begin{figure}[t]
  \centering
  % \includesvg[
  %     width=\linewidth,
  %     inkscapearea=drawing
  % ]{imgs/fig1_hook_c100_gaussian}
  \includegraphics[width=\linewidth]{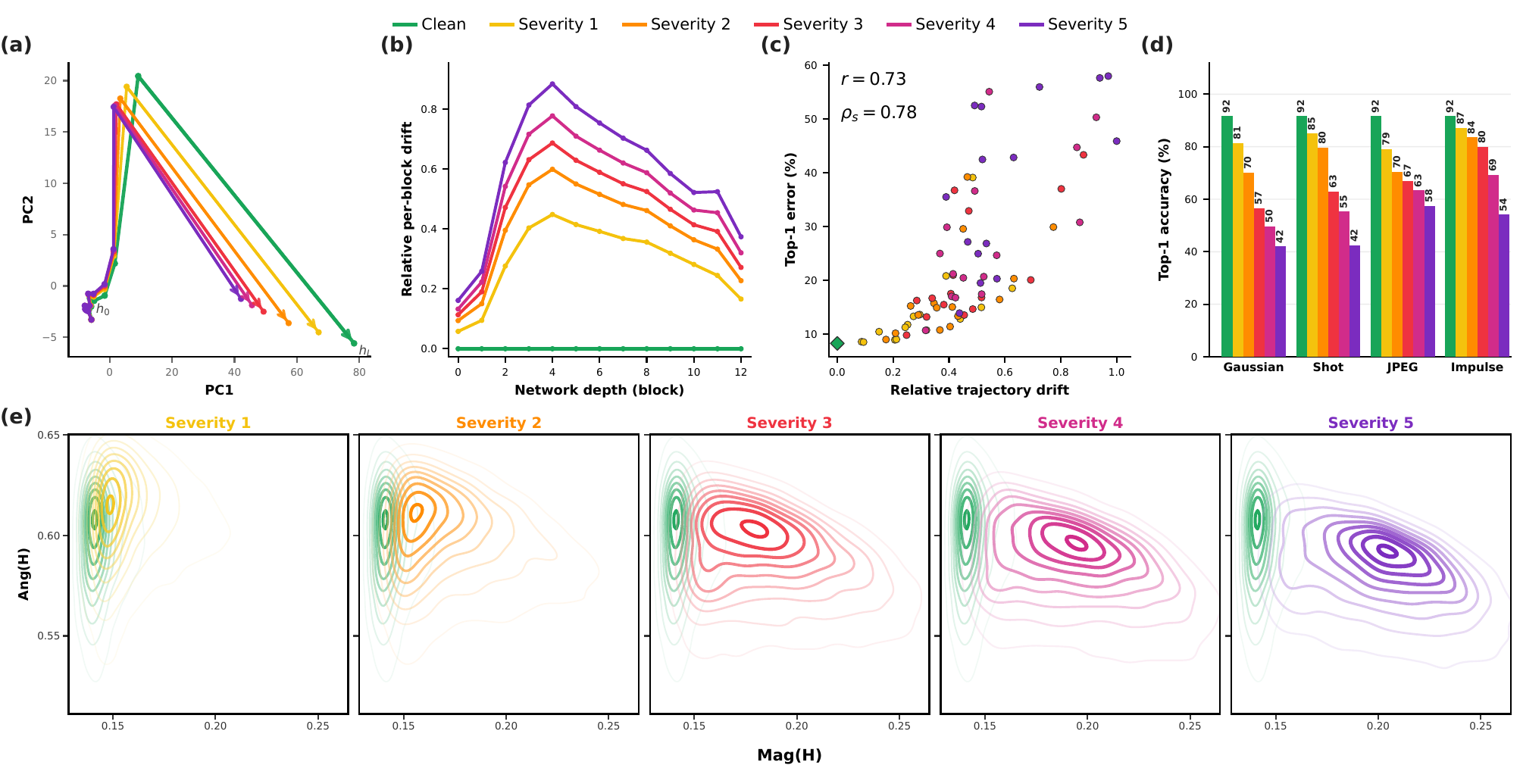}
  \caption{
  \textbf{Hidden-trajectory drift on CIFAR-100-C.}
  A frozen ViT-B/16 is evaluated on CIFAR-100-C without adaptation.
  \textbf{(a)} Mean hidden trajectories for clean inputs and Gaussian noise at severities~1--5, projected onto the first two principal components.
  \textbf{(b)} Relative clean-to-corrupted drift across network depth, averaged over all 15 corruption families.
  \textbf{(c)} Relative trajectory drift versus top-1 error across the 75 corruption-family and severity conditions (Pearson $r{=}0.73$; Spearman $\rho_s{=}0.78$).
  \textbf{(d)} Top-1 accuracy across severity for four corruption families.
  \textbf{(e)} Magnitude and angle feature distributions for clean inputs and Gaussian noise at severities~1--5.
  }
  \label{fig:trajectory-drift}
  \end{figure}

\paragraph{Panel (a): Mean hidden trajectories.}
For corruption family $c$ and severity $s$, let $\mathbf{x}_i^{c,s}$ denote the corrupted version of clean image $\mathbf{x}_i$. At trajectory position $l\in\{0,\ldots,L\}$, we average the hidden representation over all images:
\begin{equation}
\overline{\mathbf{h}}_{l}^{\,c,s}
=
\frac{1}{N}
\sum_{i=1}^{N}
\mathbf{h}_{l}\!\left(\mathbf{x}_{i}^{c,s}\right).
\label{eq:condition-mean-trajectory}
\end{equation}
This gives one mean trajectory for clean inputs and one for each corruption severity. The six trajectories are projected jointly onto the first two principal components. Each line therefore shows how the mean representation changes from the input embedding $\mathbf{h}_0$ to the final block representation $\mathbf{h}_L$. The projection is used only for visualization; all drift measurements are computed in the original representation space.

\paragraph{Panel (b): Relative drift across network depth.}
We first compute the mean norm of the clean representation at each trajectory position:
\begin{equation}
\nu_l
=
\frac{1}{N}
\sum_{i=1}^{N}
\left\lVert
\mathbf{h}_l(\mathbf{x}_i)
\right\rVert_2.
\label{eq:clean-trajectory-scale}
\end{equation}
The relative drift for corruption family $c$ and severity $s$ is then
\begin{equation}
\delta_l^{\,c,s}
=
\frac{
    \frac{1}{N}
    \sum_{i=1}^{N}
    \left\lVert
    \mathbf{h}_l(\mathbf{x}_i^{c,s})
    -
    \mathbf{h}_l(\mathbf{x}_i)
    \right\rVert_2
}{
    \nu_l
}.
\label{eq:relative-block-drift}
\end{equation}
The numerator measures how far the corrupted representation moves from its clean counterpart. Dividing by $\nu_l$ accounts for differences in representation scale across network depth. Panel~\textbf{(b)} averages this quantity over all fifteen corruption families and plots one curve per severity. The clean curve is zero because every clean representation is compared with itself.

\paragraph{Panel (c): Trajectory drift and prediction error.}
For each corruption-family and severity condition, we average the relative drift across the complete trajectory:
\begin{equation}
D^{c,s}
=
\frac{1}{L+1}
\sum_{l=0}^{L}
\delta_l^{\,c,s}.
\label{eq:relative-trajectory-drift}
\end{equation}
Panel~\textbf{(c)} contains one point for each of the $15\times5=75$ conditions. Its horizontal coordinate is the trajectory drift $D^{c,s}$, divided by the largest value among the $75$ conditions so that the displayed values lie between zero and one:
\begin{equation}
\widetilde{D}^{c,s}
=
\frac{
D^{c,s}
}{
\max_{c',s'}D^{c',s'}
}.
\label{eq:display-trajectory-drift}
\end{equation}
This final scaling changes only the displayed axis and does not affect either correlation. The vertical coordinate is the top-1 error over the same images. Pearson $r$ measures the linear association between drift and error, while Spearman $\rho_s$ measures whether conditions with larger drift generally have larger error, independent of a linear relationship. The clean point is shown at zero drift as a reference and is excluded from both calculations.

\paragraph{Panel (d): Accuracy across severity.}
Panel~\textbf{(d)} provides a direct view of the corresponding accuracy loss. We first retain corruption families whose accuracy decreases monotonically from clean through severity~5. Among these families, the panel shows the four with the largest mean accuracy loss across severities~1--5. Each group contains the clean accuracy followed by the accuracy at the five corruption severities.

\paragraph{Panel (e): Magnitude and angle distributions.}
Panel~\textbf{(e)} provides an image-level view of trajectory geometry. Each image is represented by the magnitude and angle features defined in Section~\ref{sec:trajectory}. The magnitude feature measures the average distance traveled between adjacent hidden states, normalized by the distance between $\mathbf{h}_0$ and $\mathbf{h}_L$. The angle feature similarly measures the average change in orientation between adjacent states, normalized by the angle between the trajectory endpoints. Each column compares the clean distribution with one corruption severity. Increasing separation between the contours indicates that individual corrupted images follow different hidden trajectories from their clean counterparts.

\paragraph{Controlling for corruption severity.}
Across the $75$ CIFAR-100-C conditions, relative trajectory drift and top-1 error have Pearson correlation $r{=}0.73$ and Spearman correlation $\rho_s{=}0.78$. Both quantities generally increase with corruption severity, so we additionally test whether the association remains when severity is held fixed.

Within each severity level, Pearson correlation remains positive and ranges from $0.59$ to $0.71$ across the fifteen corruption families. After subtracting the corresponding severity-level mean from both drift and error, their residual Pearson correlation is $r{=}0.63$. The corresponding partial Spearman correlation is $\rho_s{=}0.67$.

The five severity levels of a corruption family are related and should not be treated as independent observations. We therefore resample complete corruption-family profiles, retaining all five severities whenever a family is selected. This family-cluster bootstrap gives a $95\%$ confidence interval of $[0.39,0.83]$ for the severity-controlled Pearson correlation. A complementary permutation test that reassigns complete five-severity drift profiles across corruption families gives $p{<}0.01$. Thus, the observed association is not explained only by the shared increase of drift and error with corruption severity.

\paragraph{ImageNet-C.}
Figure~\ref{fig:trajectory-drift-imagenet} repeats the analysis on ImageNet-C, using fog as the representative corruption family in panels~\textbf{(a)} and~\textbf{(e)}. Across the $75$ conditions, relative trajectory drift has Pearson correlation $r{=}0.78$ and Spearman correlation $\rho_s{=}0.87$ with top-1 error. Within-severity Pearson correlations range from $0.52$ to $0.69$. After removing severity-level means, the residual Pearson correlation is $r{=}0.53$, while the partial Spearman correlation is $\rho_s{=}0.69$. The family-cluster confidence interval for the severity-controlled Pearson correlation is $[0.19,0.77]$, and the corresponding family-profile permutation test gives $p{<}0.01$.

\begin{figure}[p]
\centering
\includegraphics[width=\linewidth]{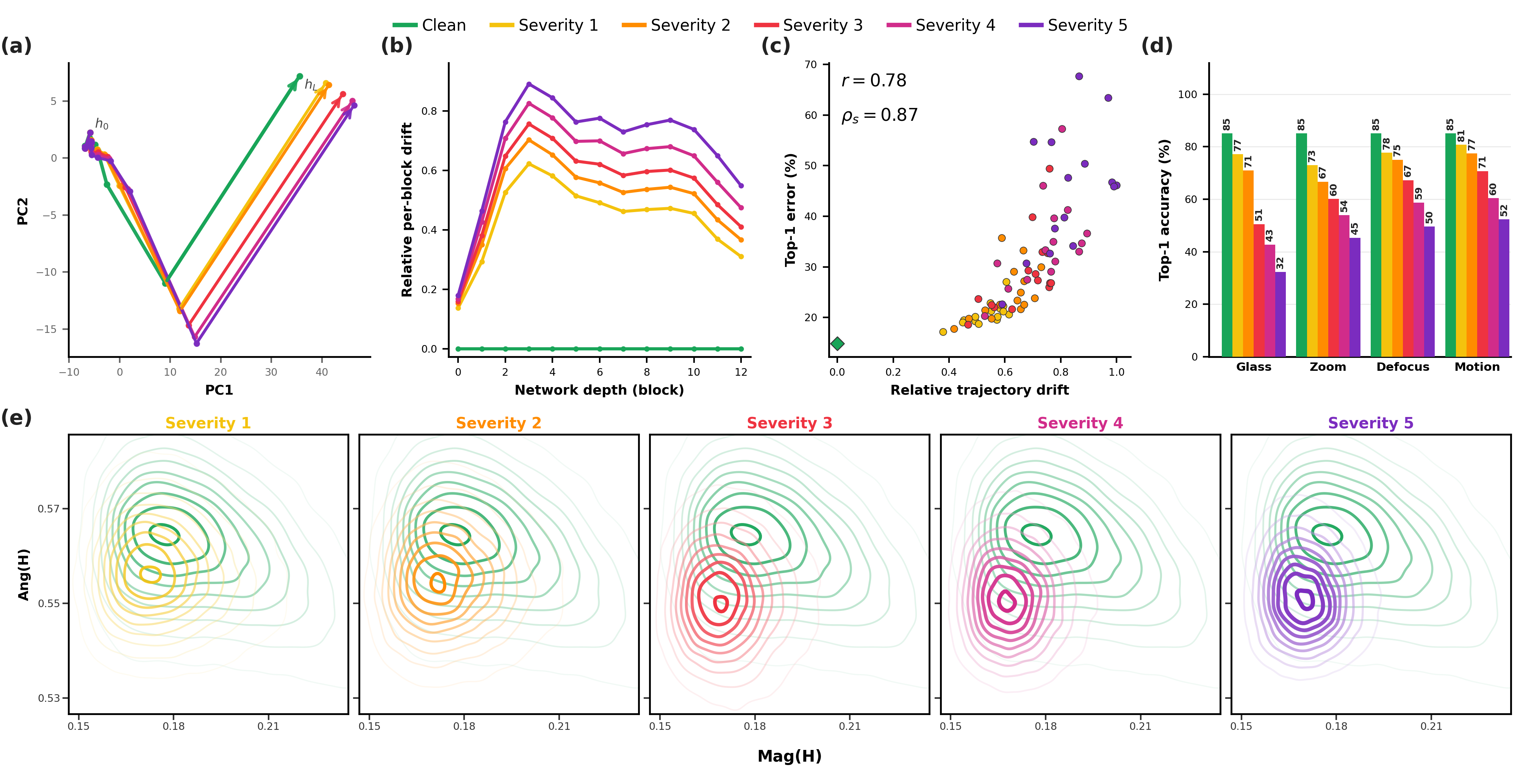}
\caption{
\textbf{Hidden-trajectory drift on ImageNet-C.}
A frozen ViT-B/16 is evaluated on the full ImageNet-C validation set without adaptation.
Panels~\textbf{(a)} and~\textbf{(e)} use fog as a representative corruption family, while panels~\textbf{(b)}--\textbf{(d)} summarize all fifteen corruption families following Figure~\ref{fig:trajectory-drift}.
Across the $75$ corruption-family and severity conditions, relative trajectory drift is associated with top-1 error (Pearson $r{=}0.78$; Spearman $\rho_s{=}0.87$).
The clean point in \textbf{(c)} is shown as a reference and is excluded from both correlations.
}
\label{fig:trajectory-drift-imagenet}
\end{figure}

% =====================================================================
\clearpage
\section{Class-Structure Visualizations under Adaptation}
\label{app:umap}
To visualize how adaptation changes the class structure of the representations used for classification, we project these features to two dimensions using UMAP and t-SNE and color each point by its true class. For each backbone, we show five model states: No Adaptation, source-free Drift-TTA without and with online gate adaptation, and source-dependent Drift-TTA without and with online gate adaptation. All panels use the complete CIFAR-10-C test set of $10{,}000$ images at severity~5 under the leave-one-family-out protocol. We show six corruption families spanning noise, blur, weather, and digital corruptions.

Each panel is projected independently, so absolute positions and axis directions should not be compared across panels. The intended comparison is the compactness and separation of the class distributions within each panel. UMAP and t-SNE provide complementary views and show the same qualitative progression: relative to No Adaptation, the fixed Drift-TTA adapters generally produce more compact and better-separated class distributions, while online gate adaptation further improves this structure. Figures~\ref{fig:umap-clip} and~\ref{fig:tsne-clip} show the CLIP image tower, while Figures~\ref{fig:umap-vit} and~\ref{fig:tsne-vit} show the ViT.

\begin{figure}[p]
\centering
\includegraphics[width=0.97\linewidth]{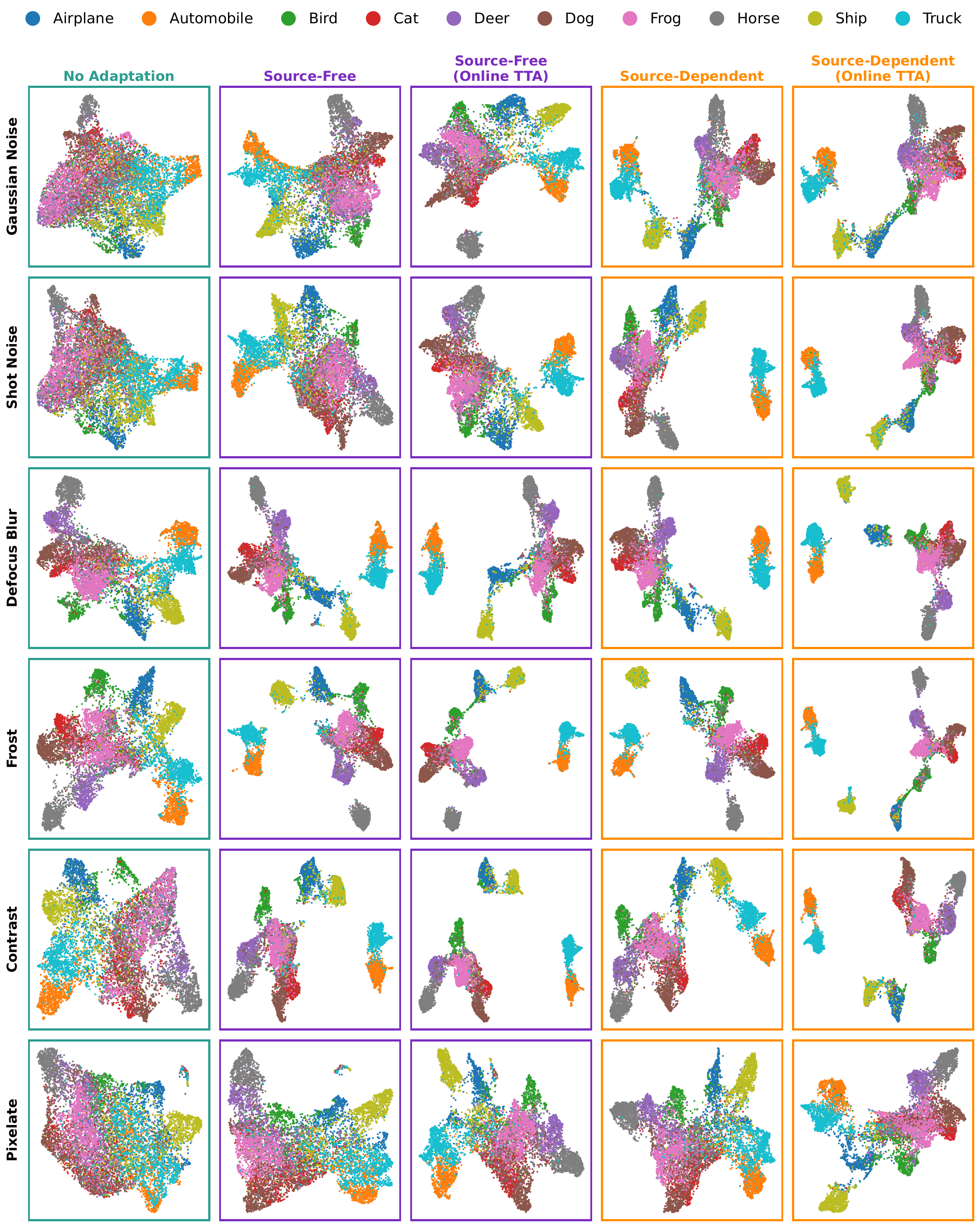}
\caption{\textbf{UMAP of CLIP ViT-B/16 image embeddings on CIFAR-10-C (severity~5), colored by class.} Rows are six corruption families; columns are, left to right, No Adaptation, source-free Drift-TTA without and with the online gate, and source-dependent Drift-TTA without and with the online gate. Column groups are color-coded (teal: no adaptation; purple: source-free; orange: source-dependent). Points are the $L_2$-normalized image embedding used for zero-shot classification, projected with UMAP (Euclidean metric); each panel is embedded independently. Leave-one-family-out: each column uses only the adapter trained without the family shown in that row.}
\label{fig:umap-clip}
\end{figure}

\begin{figure}[p]
\centering
% \includesvg[width=0.97\linewidth,inkscapearea=drawing]{imgs/tsne_grid_clip}
\includegraphics[width=0.97\linewidth]{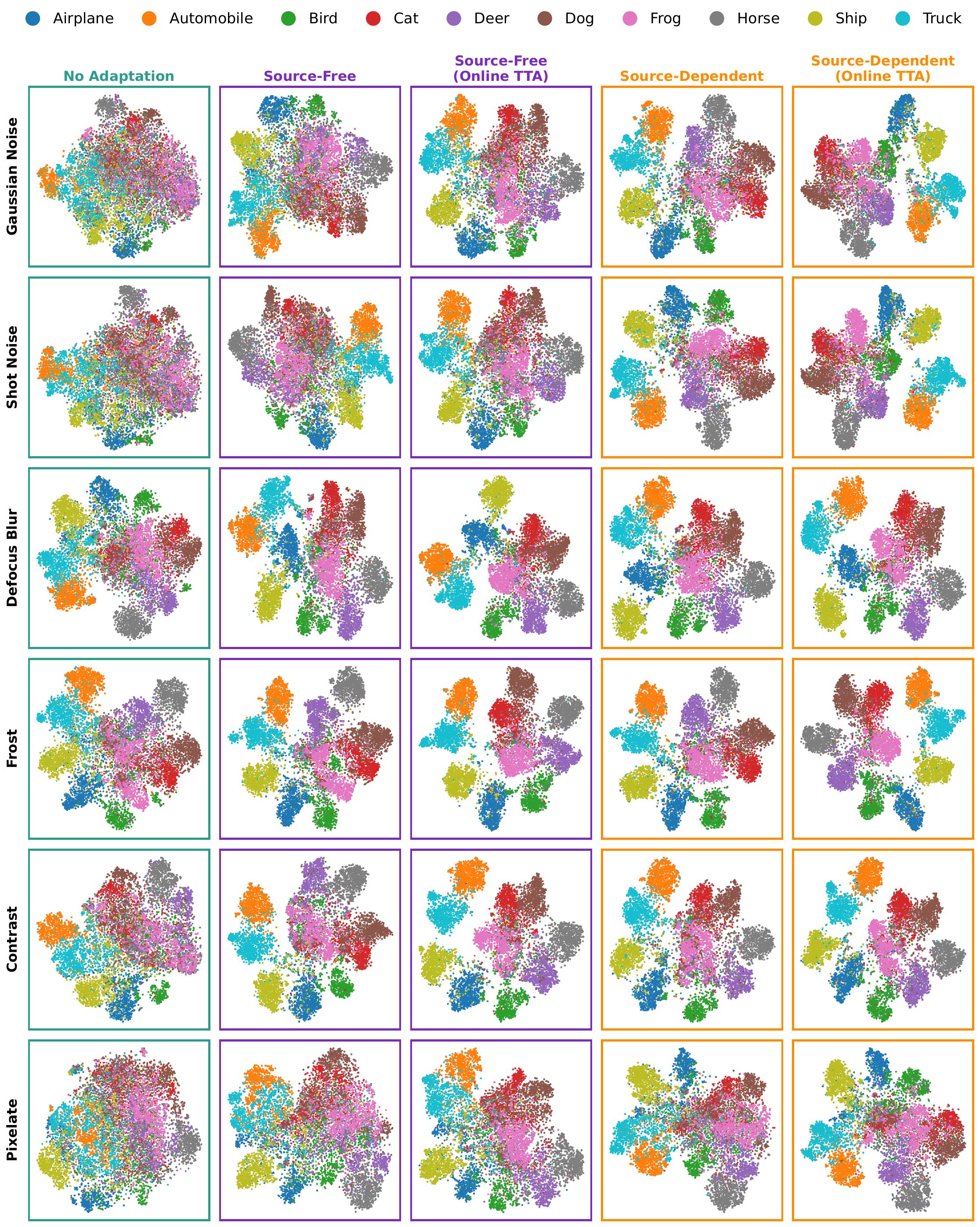}
\caption{\textbf{t-SNE of CLIP ViT-B/16 image embeddings on CIFAR-10-C (severity~5), colored by class.} Rows are six corruption families; columns are, from left to right, No Adaptation, source-free Drift-TTA without and with online gate adaptation, and source-dependent Drift-TTA without and with online gate adaptation. Points are the $L_2$-normalized image embeddings used for zero-shot classification, projected independently with t-SNE using Euclidean distance and perplexity~30. Layout, color coding, and the leave-one-family-out protocol match Figure~\ref{fig:umap-clip}.}
\label{fig:tsne-clip}
\end{figure}

\begin{figure}[p]
\centering
% \includesvg[width=0.97\linewidth,inkscapearea=drawing]{imgs/umap_grid_vit}
\includegraphics[width=0.97\linewidth]{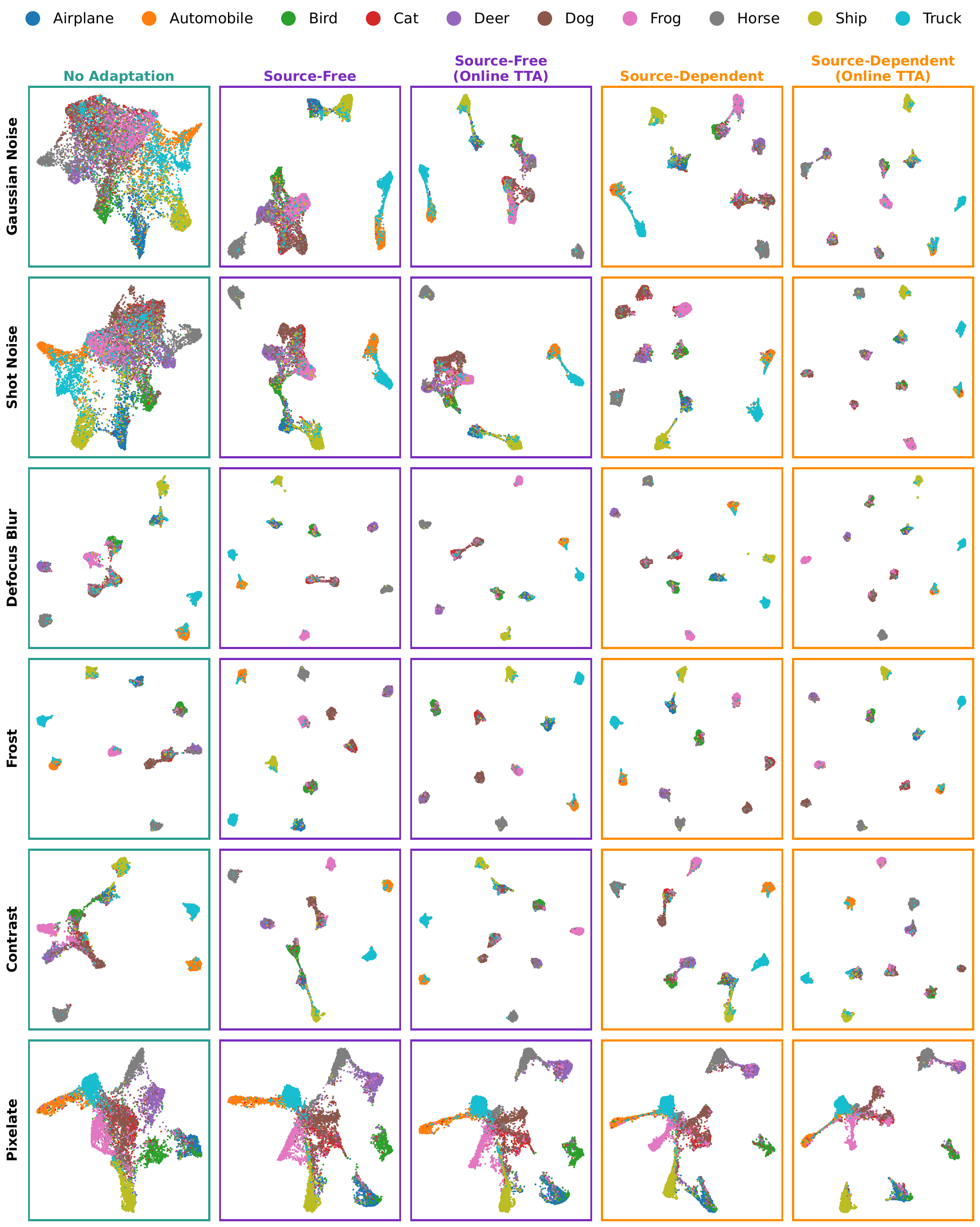}
\caption{\textbf{UMAP of ViT image features on CIFAR-10-C (severity~5), colored by class.} Layout, color coding, and protocol match Figure~\ref{fig:umap-clip}; here the points are the pre-classifier pooled embedding of the ViT. As with CLIP, both Drift-TTA priors tighten the class structure and the online gate sharpens it further.}
\label{fig:umap-vit}
\end{figure}

\begin{figure}[p]
\centering
% \includesvg[width=0.97\linewidth,inkscapearea=drawing]{imgs/tsne_grid_vit}
\includegraphics[width=0.97\linewidth]{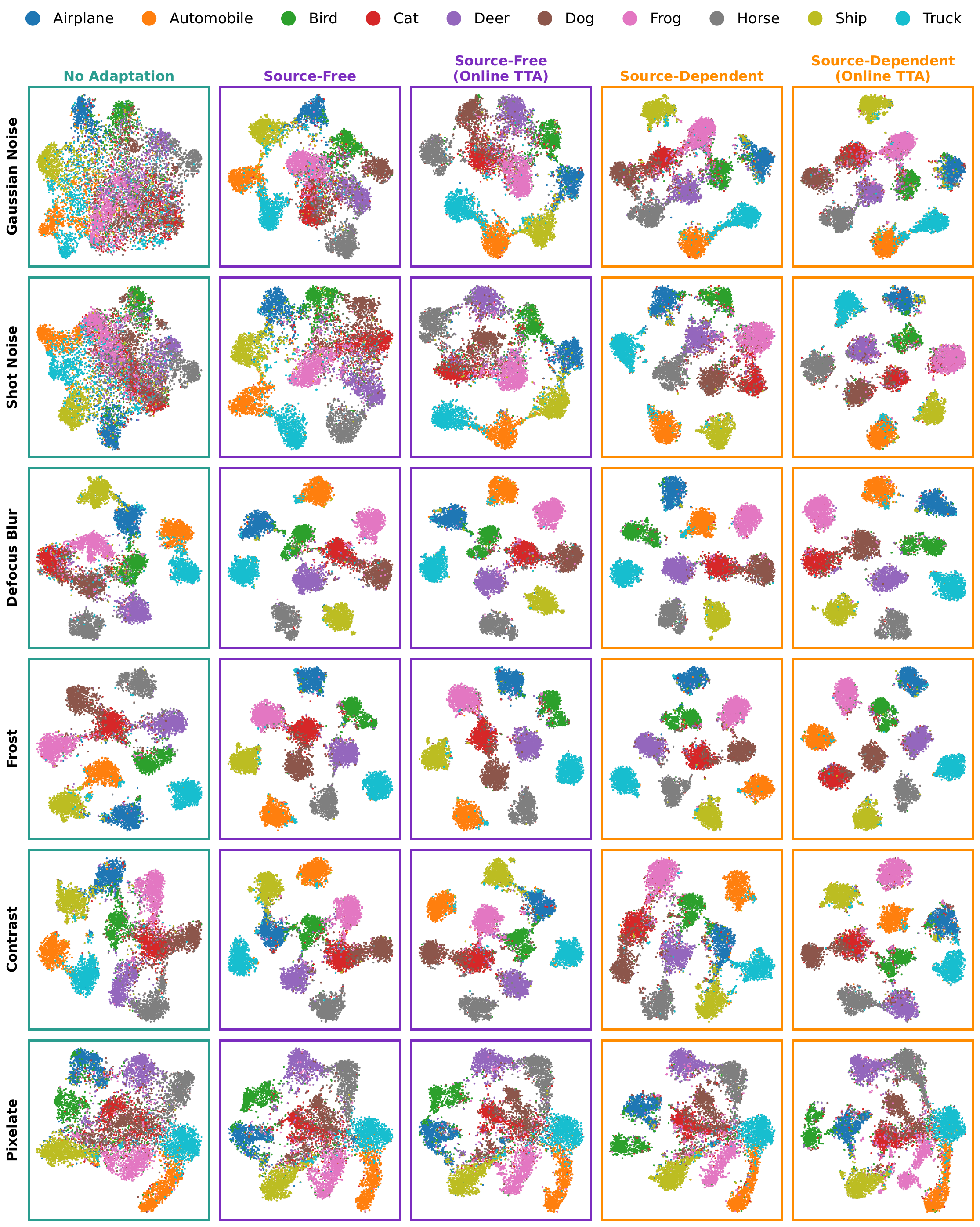}
\caption{\textbf{t-SNE of ViT image features on CIFAR-10-C (severity~5), colored by class.} Layout, color coding, projection settings, and protocol match Figure~\ref{fig:tsne-clip}; here, the points are the $L_2$-normalized pre-classifier pooled embeddings of the ViT. The class structure follows the same qualitative progression as the corresponding UMAP visualization in Figure~\ref{fig:umap-vit}.}
\label{fig:tsne-vit}
\end{figure}

% =====================================================================
\clearpage
\section{Per-Family Trajectory-Feature Distributions}
\label{app:traj-dist}
% This section extends the trajectory-feature view of Figure~\ref{fig:drift-tta-overview}(c) from a single case to individual corruption families. 

This section extends the trajectory-feature view of Figure~\ref{fig:trajectory-drift-correction}(b) from a single case to individual corruption families. Each image reduces every input to two scalar summaries of its hidden trajectory, the magnitude feature $\mathrm{Mag}(H)$ and the angle feature $\mathrm{Ang}(H)$, and draws kernel-density contours of the clean reference (green) against corruption severities~1--5 (columns). The three rows are, top to bottom, the unadapted model (No Adaptation), the frozen offline Drift-TTA adapter, and the same adapter with the online entropy gate enabled. All densities use CIFAR-10-C, CIFAR-100-C, or ImageNet-C at the stated severities under the leave-one-family-out protocol, so each panel reports its family using only the adapter that never saw that family.

The intended reading is spatial. In the examples shown, the corrupted density moves away from the clean contours as severity increases, while the offline adapter shifts it back toward clean and the online gate further strengthens this correction. We show one illustrative corruption family for each of the twelve backbone$\times$prior$\times$dataset settings, covering supervised ViT and CLIP, source-free and source-dependent training, and all three datasets.

\begin{figure}[p]
\centering
% \includesvg[width=0.97\linewidth,inkscapearea=drawing]{imgs/fig2_vit_sf_c10_glass_blur_3row_panel}
\includegraphics[width=0.97\linewidth]{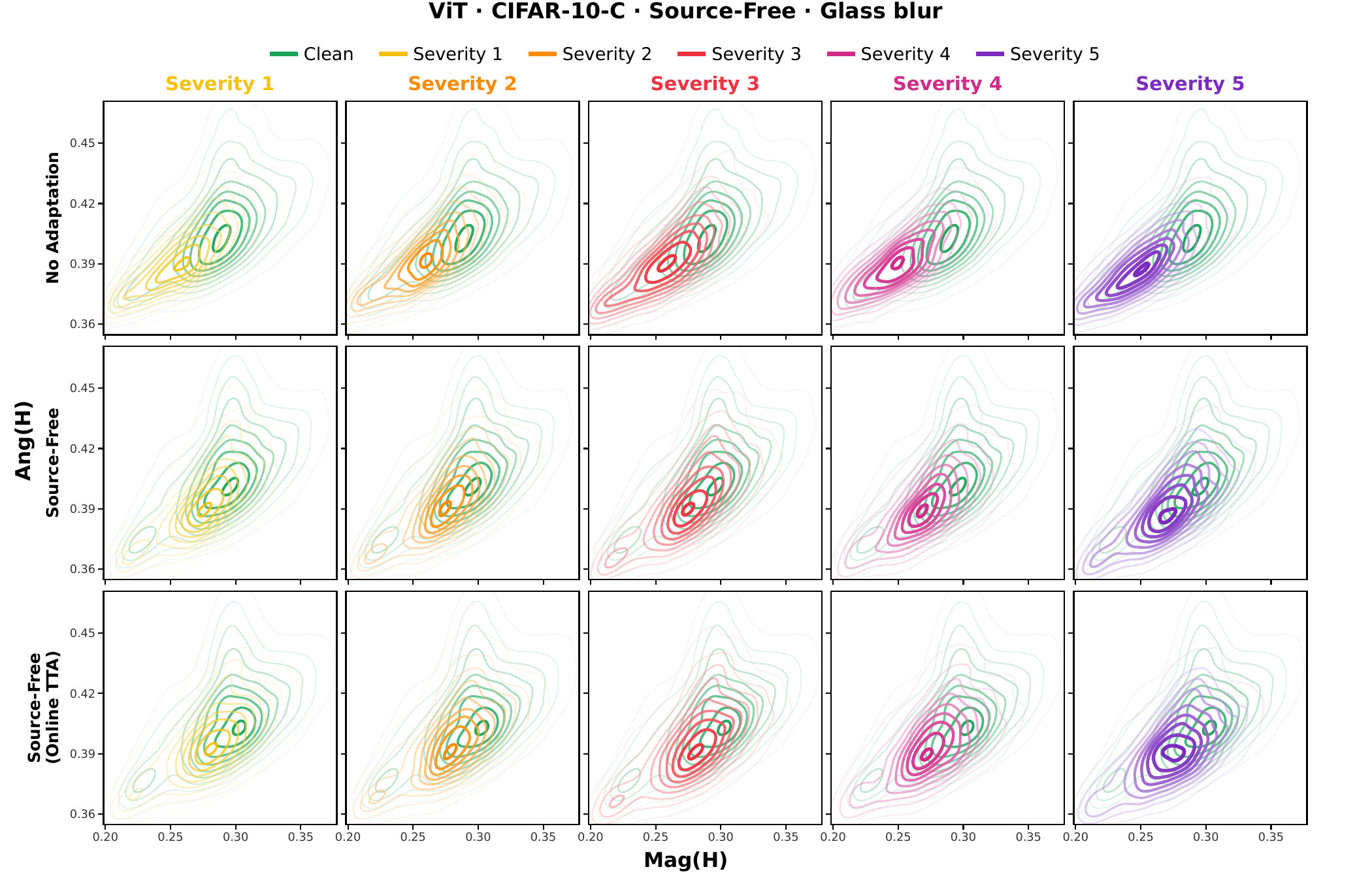}\\[0.45em]
% \includesvg[width=0.97\linewidth,inkscapearea=drawing]{imgs/fig2_vit_sf_c100_shot_noise_3row_panel}
\includegraphics[width=0.97\linewidth]{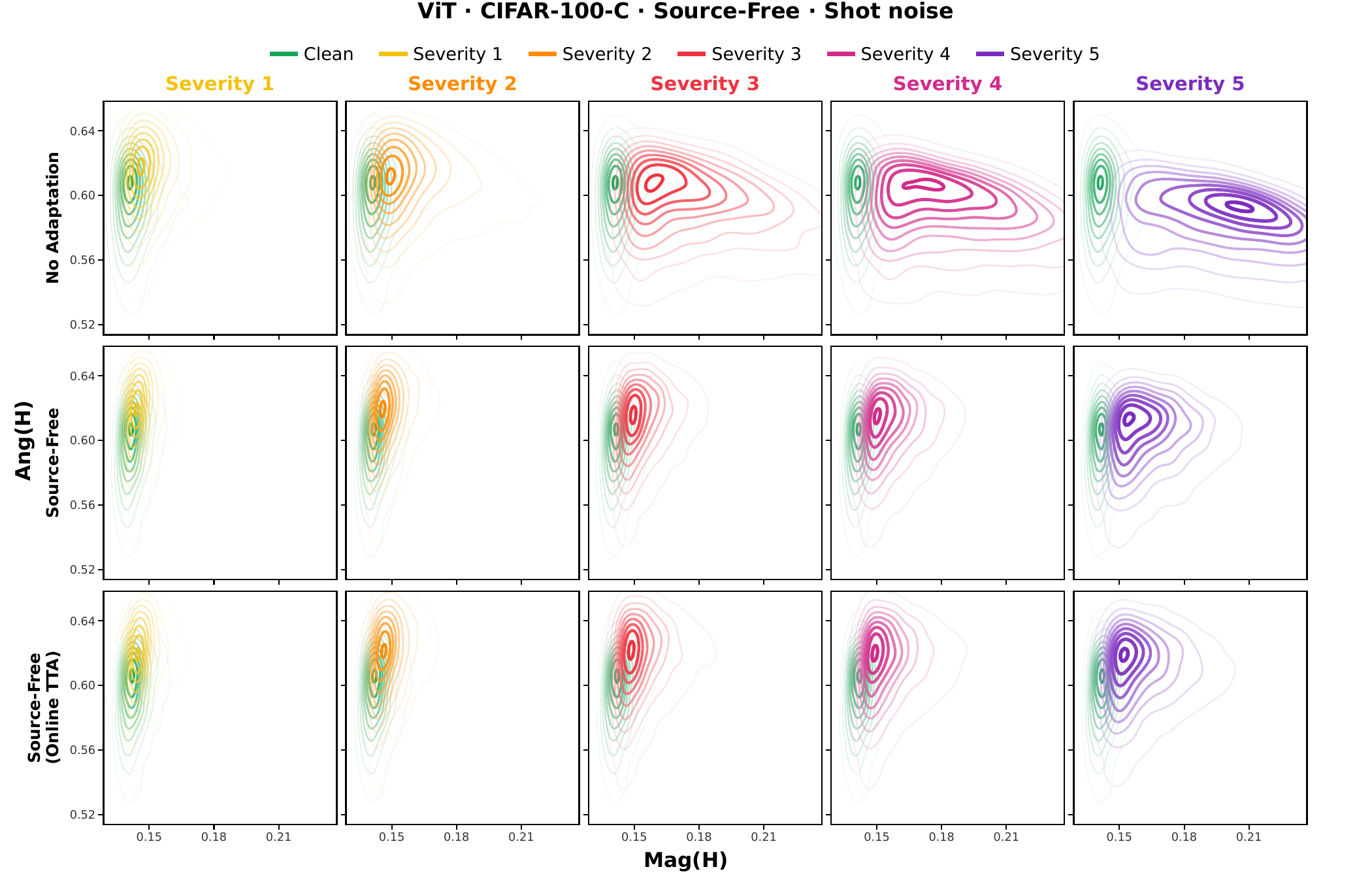}
\caption{\textbf{Trajectory-feature distributions, supervised ViT (source-free).} Each panel plots $\mathrm{Mag}(H)$ (horizontal) against $\mathrm{Ang}(H)$ (vertical); green is the clean reference and colored contours are severities~1--5 (columns), while the three rows are No Adaptation, offline Drift-TTA, and online Drift-TTA. Top: glass blur on CIFAR-10-C; bottom: shot noise on CIFAR-100-C. Leave-one-family-out.}
\label{fig:traj-dist-1}
\end{figure}

\begin{figure}[p]
\centering
% \includesvg[width=0.97\linewidth,inkscapearea=drawing]{imgs/fig2_vit_sf_in_snow_3row_panel}
\includegraphics[width=0.97\linewidth]{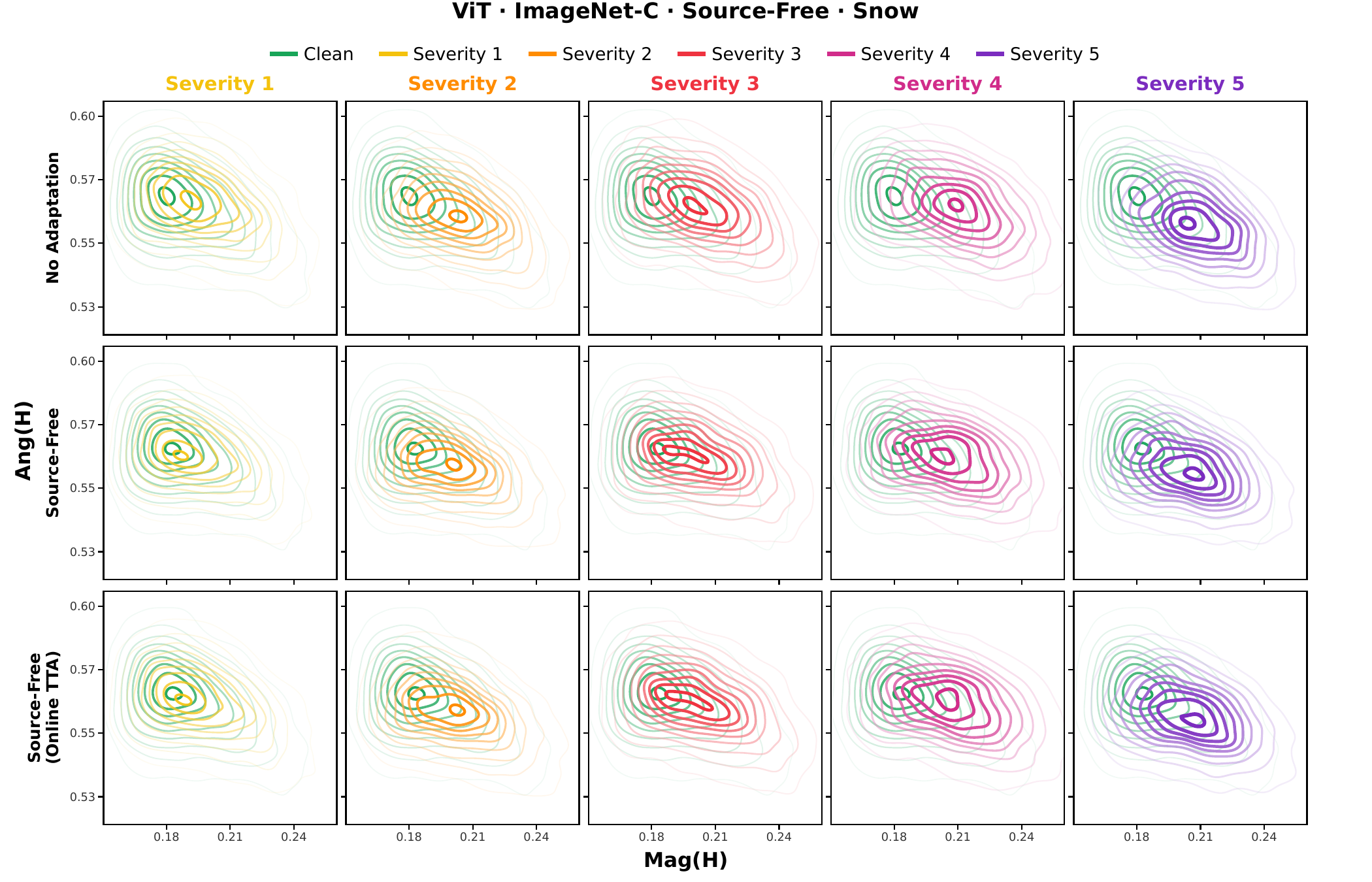}\\[0.45em]
% \includesvg[width=0.97\linewidth,inkscapearea=drawing]{imgs/fig2_vit_sd_c10_pixelate_3row_panel}
\includegraphics[width=0.97\linewidth]{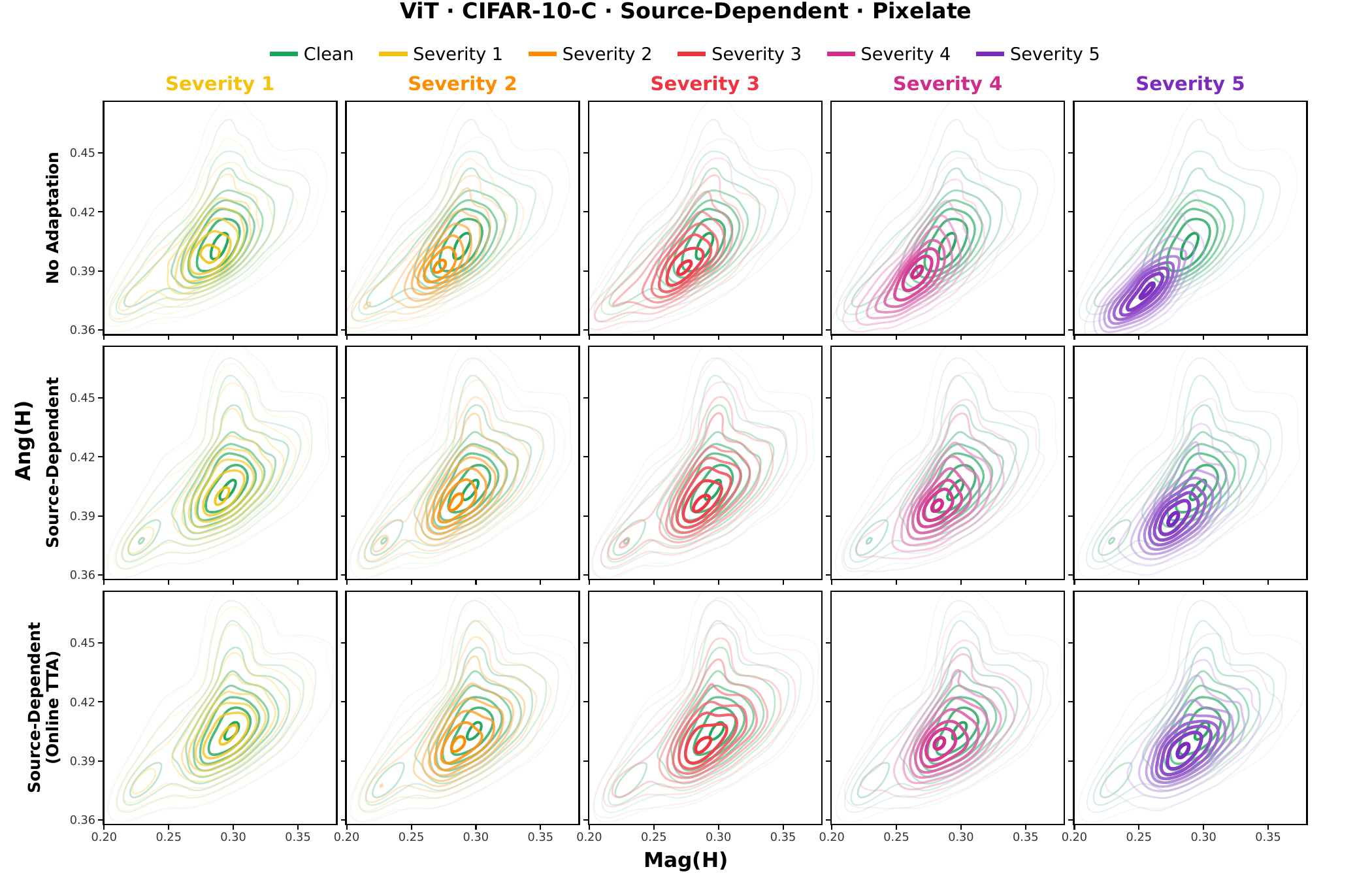}
\caption{\textbf{Trajectory-feature distributions, supervised ViT.} Green is the clean reference and colored contours are severities~1--5 (columns); the three rows are No Adaptation, offline Drift-TTA, and online Drift-TTA. Top: snow on ImageNet-C (source-free); bottom: pixelate on CIFAR-10-C (source-dependent). Leave-one-family-out.}
\label{fig:traj-dist-2}
\end{figure}

\begin{figure}[p]
\centering
% \includesvg[width=0.97\linewidth,inkscapearea=drawing]{imgs/fig2_vit_sd_c100_impulse_noise_3row_panel}
\includegraphics[width=0.97\linewidth]{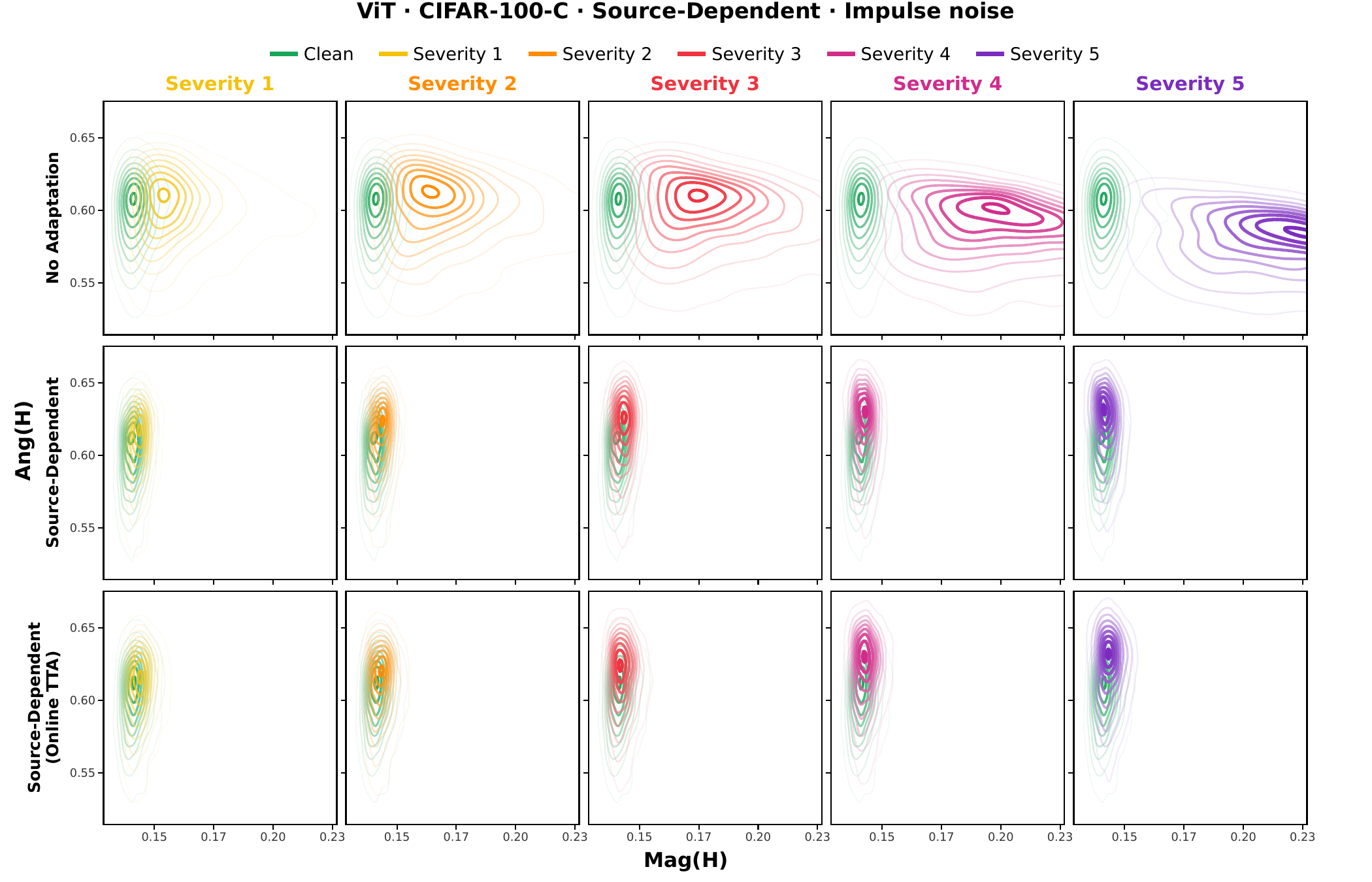}\\[0.45em]
% \includesvg[width=0.97\linewidth,inkscapearea=drawing]{imgs/fig2_vit_sd_in_fog_3row_panel}
\includegraphics[width=0.97\linewidth]{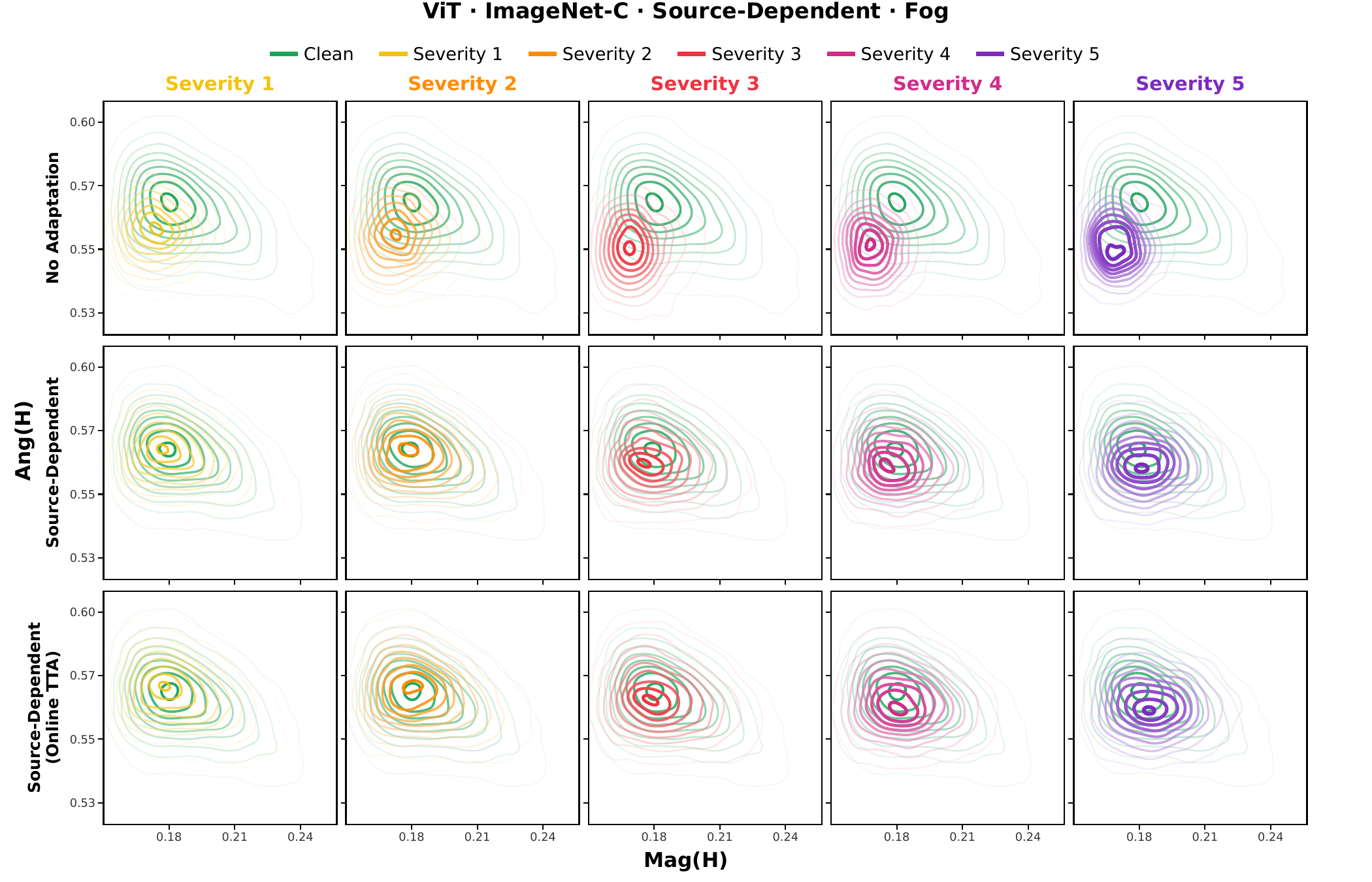}
\caption{\textbf{Trajectory-feature distributions, supervised ViT (source-dependent).} Green is the clean reference and colored contours are severities~1--5 (columns); the three rows are No Adaptation, offline Drift-TTA, and online Drift-TTA. Top: impulse noise on CIFAR-100-C; bottom: fog on ImageNet-C. Leave-one-family-out.}
\label{fig:traj-dist-3}
\end{figure}

\begin{figure}[p]
\centering
% \includesvg[width=0.97\linewidth,inkscapearea=drawing]{imgs/fig2_clip_sf_c10_motion_blur_3row_panel}
\includegraphics[width=0.97\linewidth]{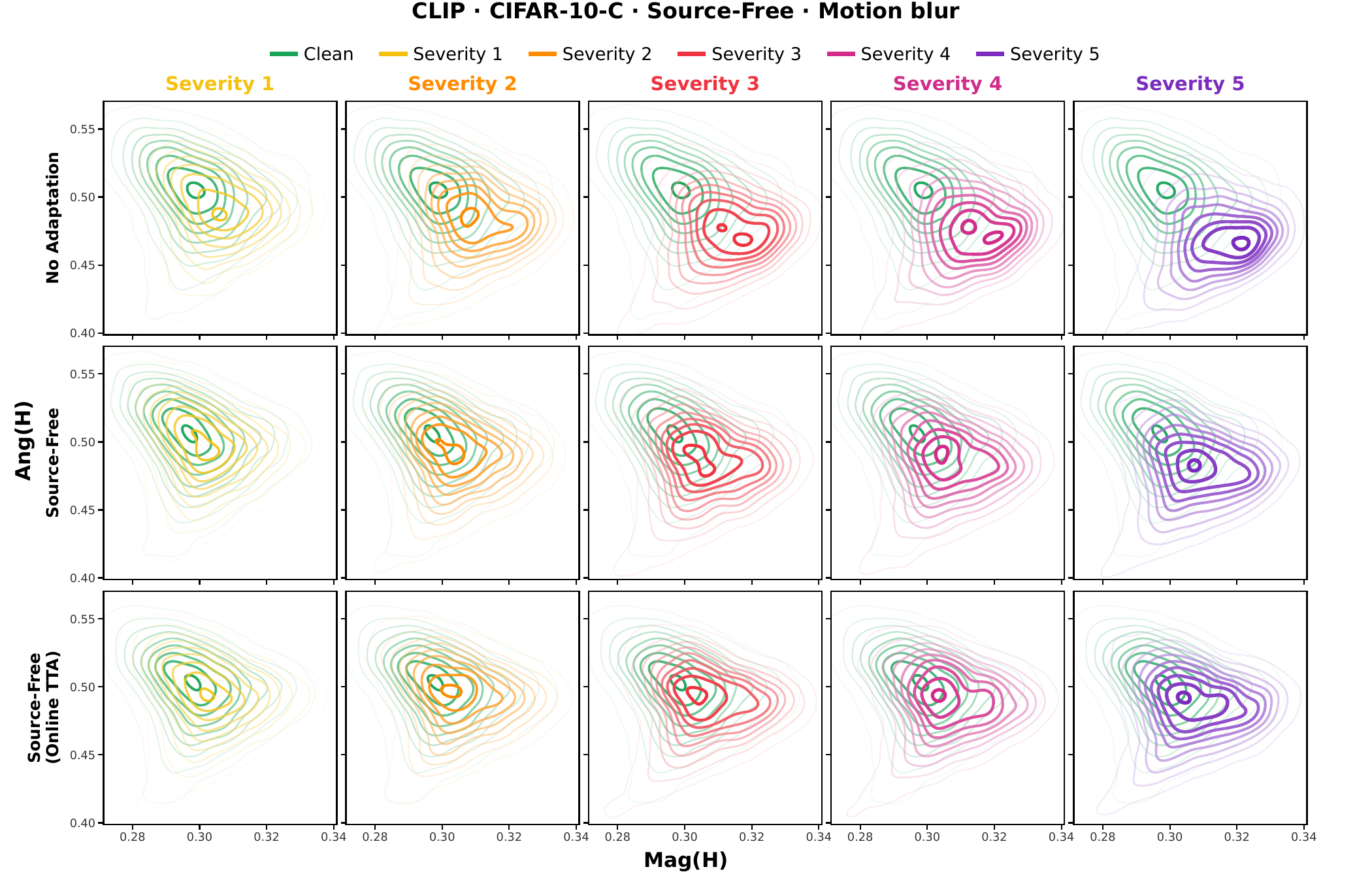}\\[0.45em]
% \includesvg[width=0.97\linewidth,inkscapearea=drawing]{imgs/fig2_clip_sf_c100_elastic_transform_3row_panel}
\includegraphics[width=0.97\linewidth]{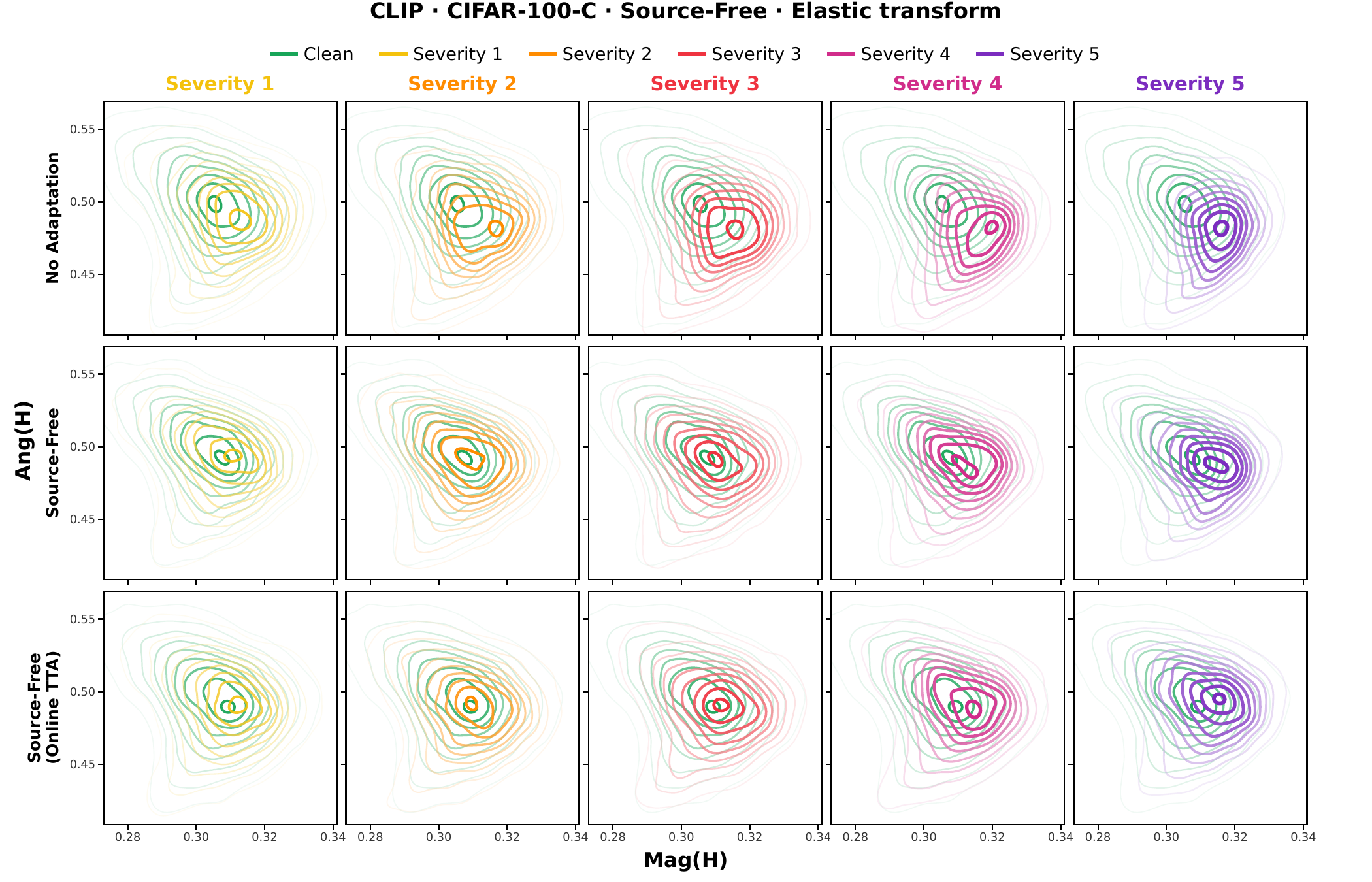}
\caption{\textbf{Trajectory-feature distributions, CLIP (source-free).} Green is the clean reference and colored contours are severities~1--5 (columns); the three rows are No Adaptation, offline Drift-TTA, and online Drift-TTA. Top: motion blur on CIFAR-10-C; bottom: elastic transform on CIFAR-100-C. Leave-one-family-out.}
\label{fig:traj-dist-4}
\end{figure}

\begin{figure}[p]
\centering
% \includesvg[width=0.97\linewidth,inkscapearea=drawing]{imgs/fig2_clip_sf_in_zoom_blur_3row_panel}
\includegraphics[width=0.97\linewidth]{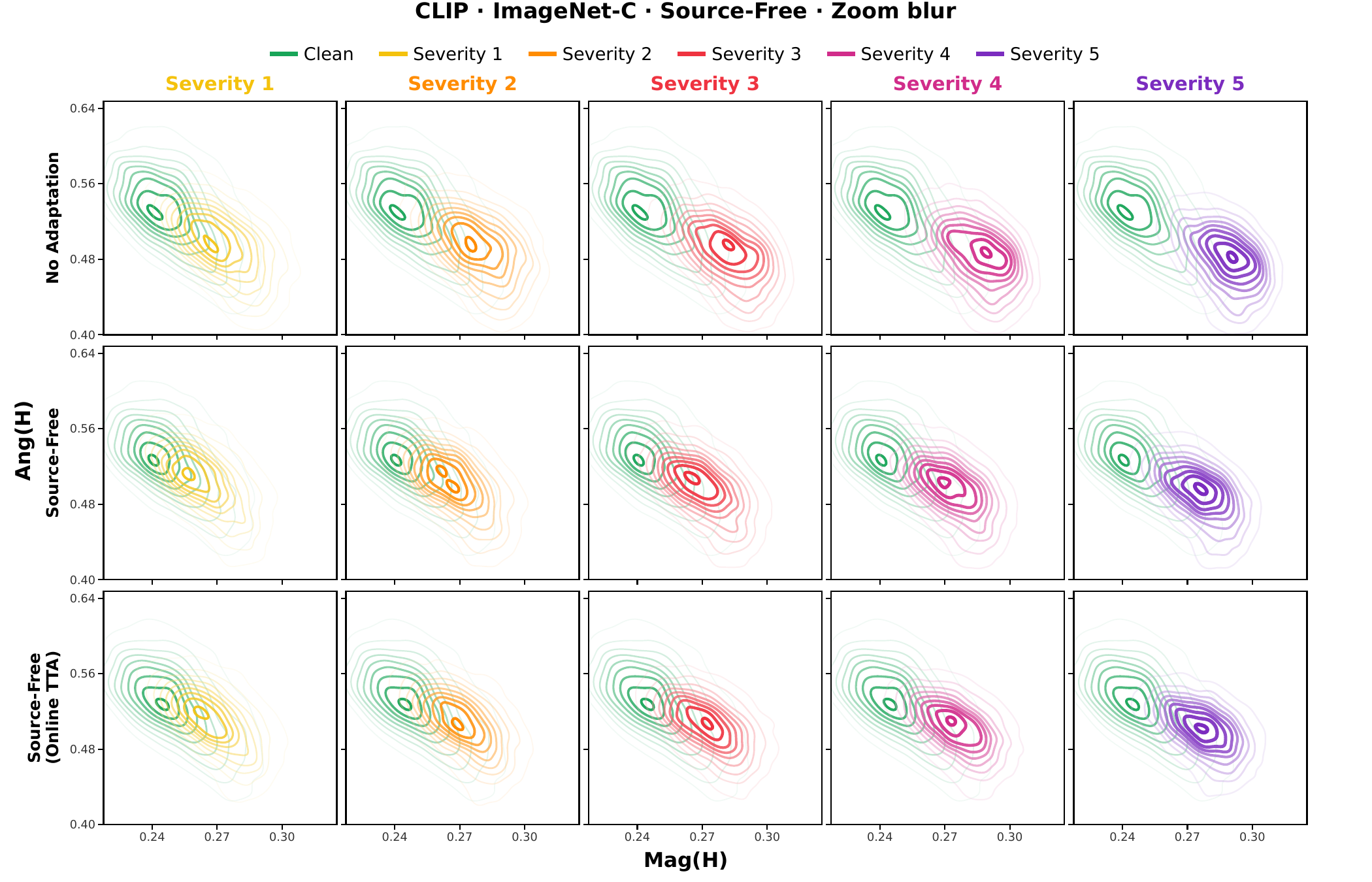}\\[0.45em]
% \includesvg[width=0.97\linewidth,inkscapearea=drawing]{imgs/fig2_clip_sd_c10_defocus_blur_3row_panel}
\includegraphics[width=0.97\linewidth]{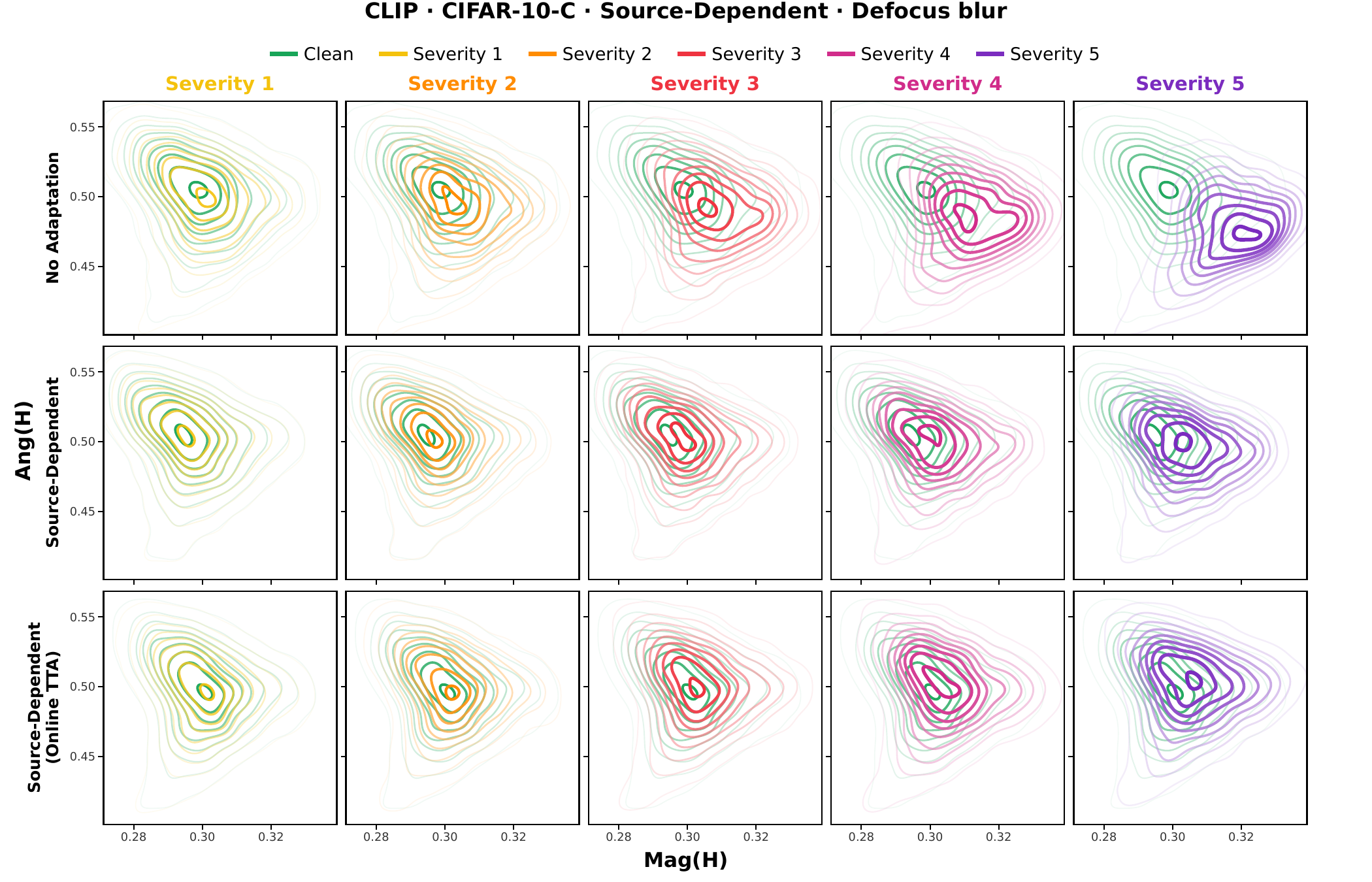}
\caption{\textbf{Trajectory-feature distributions, CLIP.} Green is the clean reference and colored contours are severities~1--5 (columns); the three rows are No Adaptation, offline Drift-TTA, and online Drift-TTA. Top: zoom blur on ImageNet-C (source-free); bottom: defocus blur on CIFAR-10-C (source-dependent). Leave-one-family-out.}
\label{fig:traj-dist-5}
\end{figure}

\begin{figure}[p]
\centering
% \includesvg[width=0.97\linewidth,inkscapearea=drawing]{imgs/fig2_clip_sd_c100_gaussian_noise_3row_panel}
\includegraphics[width=0.97\linewidth]{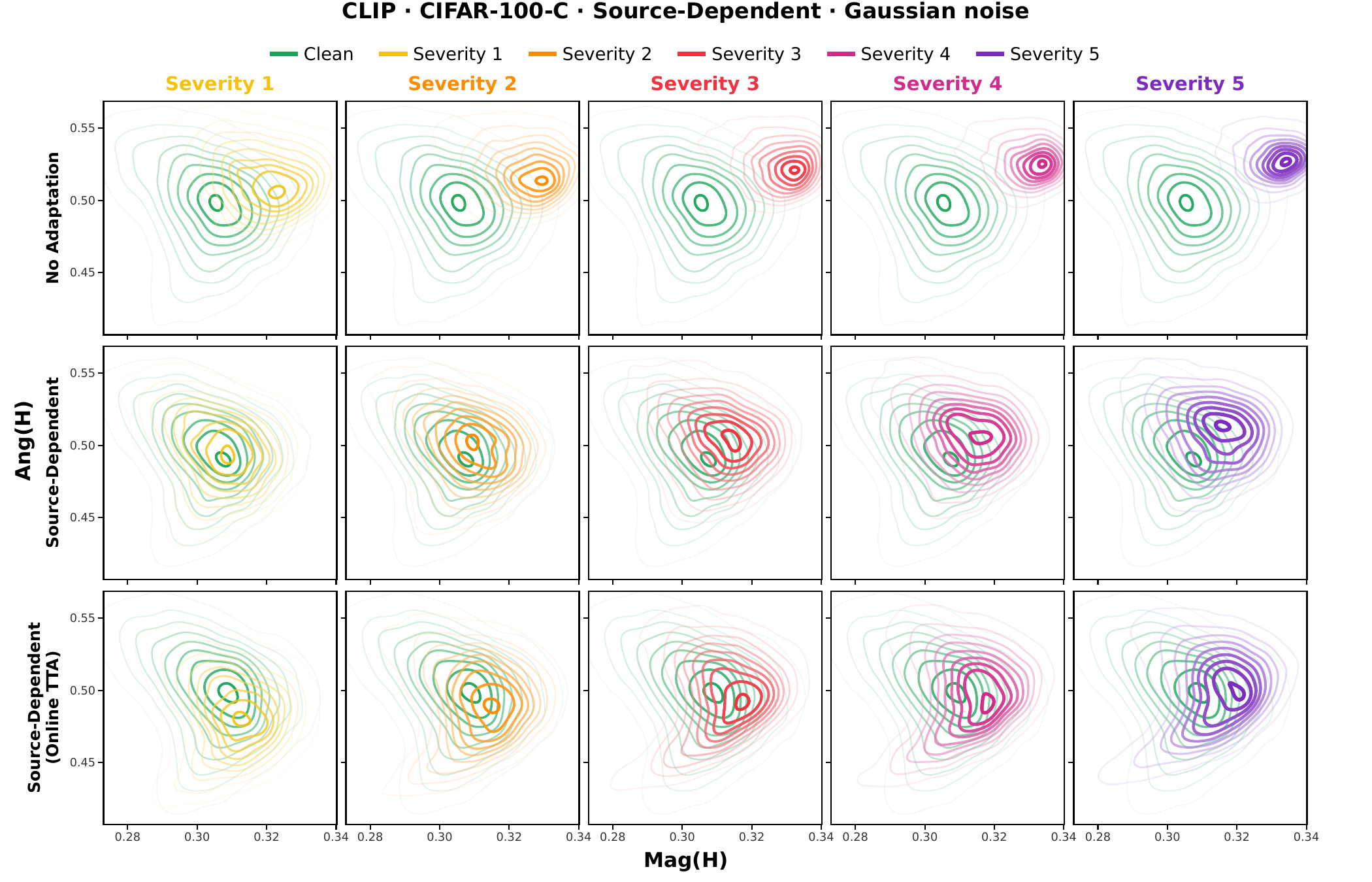}\\[0.45em]
% \includesvg[width=0.97\linewidth,inkscapearea=drawing]{imgs/fig2_clip_sd_in_contrast_3row_panel}
\includegraphics[width=0.97\linewidth]{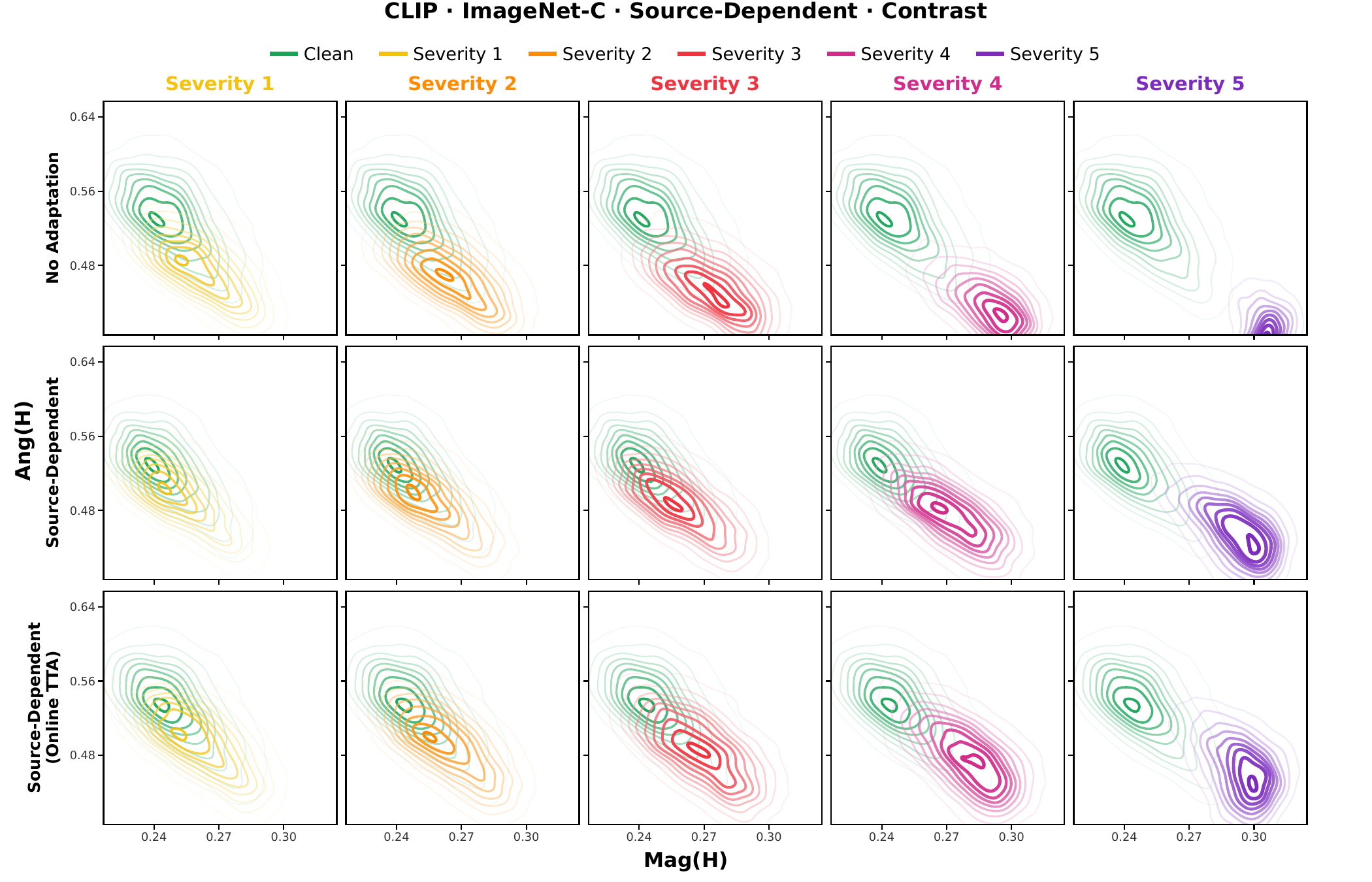}
\caption{\textbf{Trajectory-feature distributions, CLIP (source-dependent).} Green is the clean reference and colored contours are severities~1--5 (columns); the three rows are No Adaptation, offline Drift-TTA, and online Drift-TTA. Top: Gaussian noise on CIFAR-100-C; bottom: contrast on ImageNet-C. Leave-one-family-out.}
\label{fig:traj-dist-6}
\end{figure}

\end{document}